\documentclass{article}

\usepackage{xarxiv}

\usepackage[american]{babel}

\usepackage{natbib} 
\usepackage{mathtools} 
\usepackage{booktabs} 
\usepackage{tikz} 

\usepackage{graphicx}
\usepackage{multirow}
\usepackage{todonotes}
\usepackage{stmaryrd,nicefrac}
\usepackage{amsmath}
\usepackage{amssymb}
\usepackage{hyperref}
\usepackage{amsthm}
\usepackage{algorithm}
\usepackage{algorithmic}
\usepackage[capitalize,noabbrev]{cleveref}
\usepackage{subfigure}

\usepackage{pgfplots}
\pgfplotsset{compat=1.5}			
\usepgfplotslibrary{fillbetween}
\usepackage{tikz}
\usetikzlibrary{calc}

\newtheorem{theorem}{Theorem}
\newtheorem{lemma}{Lemma}
\newtheorem{corollary}{Corollary}

\title{\large Approximating the Shapley Value without Marginal Contributions}

\author{
    Patrick Kolpaczki \\
	Paderborn University \\
	\texttt{patrick.kolpaczki@upb.de}
    \AND
	Viktor Bengs$^{a,b}$, Maximilian Muschalik$^{a,b}$, Eyke H\"ullermeier${}^{a,b}$\\
	${}^{a}$Institute of Informatics, University of Munich (LMU)\\
	${}^{b}$Munich Center for Machine Learning\\
	\texttt{viktor.bengs@lmu.de, maximilian.muschalik@lmu.de, eyke@lmu.de} 
}

\begin{document}
	
\maketitle
	
\begin{abstract}
	The Shapley value, which is arguably the most popular approach for assigning a meaningful contribution value to players in a cooperative game, has recently been used intensively in explainable artificial intelligence.
Its meaningfulness is due to axiomatic properties that only the Shapley value satisfies, which, however, comes at the expense of an exact computation growing exponentially with the number of agents.
Accordingly, a number of works are devoted to the efficient approximation of the Shapley value, most of them revolve around the notion of an agent's marginal contribution.
In this paper, we propose with \emph{SVARM} and \emph{Stratified SVARM} two parameter-free and domain-independent approximation algorithms based on a representation of the Shapley value detached from the notion of marginal contribution.
We prove unmatched theoretical guarantees regarding their approximation quality and provide empirical results including synthetic games as well as common explainability use cases comparing ourselves with state-of-the-art methods.
\end{abstract}
	
\section{Introduction} \label{sec:Introduction}
Whenever agents can federalize in groups (form coalitions) to accomplish a task and get rewarded with a collective benefit that is to be shared among the group members, the notion of \emph{cooperative game} stemming from game theory is arguably the most favorable concept to model such situations.
This is due to its simplicity, which nevertheless allows for covering a whole range of practical applications.
The agents are called \emph{players} and are contained in a player set $\mathcal{N}$.
Each possible subset of players $S \subseteq \mathcal{N}$ is understood as a \emph{coalition} and the coalition $\mathcal{N}$ containing all players is called the \emph{grand coalition}.
The collective benefit $\nu(S)$ that a coalition $S$ receives upon formation is given by a \emph{value function} $\nu$ assigning each coalition a real-valued \emph{worth}.

The connection of cooperative games to (supervised) machine learning is already well-established.
The most prominent example is feature importance scores, both local and global, for a machine learning model:
features of a dataset can be seen as players, allowing one to interpret a feature subset as a coalition, while the model's generalization performance using exactly that feature subset is its worth \cite{Cohen.2007}.
Other applications include evaluating the \mbox{importance} of parameters in a machine learning model, e.g.\ single neurons in a deep neural network \cite{Ghorbani.2020} or base learners in an ensemble \cite{Rozemberczki.2021}, or assigning relevance scores to datapoints in a given dataset \cite{Ghorbani.2019}.
 See \citet{Rozemberczki.2022} for a wider overview of~its usage in the field of explainable artificial intelligence.
Outside the realm of machine learning cooperative games also found applications in operations research \cite{Luo.2022}, for finding fair compensation mechanisms in electricity grids \cite{OBrien.2015}, or even for the purpose of identifying the most influential individuals in terrorist networks \cite{vanCampen.2018}.

In all of these applications, the question naturally arises of how to appropriately determine the contribution of a single player (feature, parameter, etc.) with respect to the grand collective benefit.
In other words, how to allocate the worth $\nu(\mathcal{N})$ of the full player set $\mathcal{N}$ among the players in a fair manner.
The indisputably most popular solution to this problem is the \emph{Shapley value} \cite{Shapley.1953}, which can be intuitively expressed by \emph{marginal contributions}.
We call the increase in worth that comes with the inclusion of player $i$ to a coalition $S$, i.e., the difference $\nu(S \cup \{i\}) -\nu(S)$, player $i$'s marginal contribution to $S$.
The Shapley value of $i$ is a weighted average of all~its marginal contributions to coalitions that do not include $i$.
Its popularity stems from the fact that it is the only solution to satisfy axiomatic properties that arguably capture fairness \cite{Shapley.1953}. 

Despite the appealing theoretical properties of the Shapley value, there is one major drawback with respect to its practical application, as its computational complexity increases exponentially with the number of players $n$.  
As a consequence, the exact computation of the Shapley value becomes practically infeasible even for a moderate number of players.
This is especially the case where accesses to $\nu$ are costly, e.g., re-evaluating a (complex) machine learning model for a specific feature subset, or manipulating training data each time $\nu$ is accessed.
Recently, several approximation methods have been proposed in search of a remedy, enabling the utilization of the Shapley value in explainable AI (and beyond).
However, most works are stiffened towards the notion of marginal contribution, and, consequently, judge algorithms by their achieved approximation accuracy depending on the number of evaluated marginal contributions.
This measure does not do justice to the fact that approximations can completely dispense with the consideration of marginal contributions and elicit information from $\nu$ in a more efficient way\,---\,as we show in this paper.
We claim that the number of single accesses to $\nu$ should be considered instead, since especially in machine learning, as mentioned above, access to $\nu$ is a bottleneck in overall runtime.
In this paper, we make up for this deficit by considering the problem of approximating the Shapley values under a fixed \emph{budget} $T$ of evaluations (accesses) of $\nu$.

\paragraph{Contribution.}
We present a novel representation of the Shapley value that does not rely on the notion of marginal contribution.
Our first proposed approximation algorithm \emph{Shapley Value Approximation without Requesting Marginals} (SVARM) exploits this representation and directly samples values of coalitions, facilitating ``a swarm of updates'', i.e., multiple Shapley value estimates are updated at once.
This is in stark contrast to the usual way of sampling marginal contributions that only allows the update of a single 
estimate.
We prove theoretical guarantees \mbox{regarding} SVARM's precision including the bound of $\mathcal{O}(\frac{\log n}{T - n})$ on its variance.

Based on a partitioning of the set of all coalitions according to their size, we develop with \emph{Stratified SVARM} a refinement of SVARM.
The applied stratification materializes a twofold improvement: (i) the homogeneous strata (w.r.t.\ the coalition worth) significantly accelerate convergence of estimates, (ii) our stratified representation of the Shapley value with decomposed marginal contributions facilitates a mechanism that updates the estimates of \emph{all} players with \emph{each single} coalition sampled.
Among other results, we bound its variance by $\mathcal{O}\big(\frac{\log n}{T - n\log n}\big)$.

Besides our superior theoretical findings, both algorithms possess a number of properties in their favor.
More specifically, both are unbiased, parameter-free, incremental, i.e., the available budget has not to be fixed and can be enlarged or cut prematurely, facilitating on-the-fly approximations due to their anytime property, and do not require any knowledge about the latent value function. 
Moreover, both are domain-independent and not limited to some specific fields, but can be used to approximate the Shapley values of any possible cooperative game.

Finally, we compare our algorithms empirically against other popular competitors, demonstrating their practical \mbox{usefulness} and proving our empirical enhancement \emph{Stratified SVARM$^+$}, which samples without replacement to be the first sample-mean-based approach to achieve rivaling state-of-the-art approximation quality.
All code including documentation and the technical appendix can be found on GitHub\footnote{\url{https://github.com//kolpaczki//Approximating-the-Shapley-Value-without-Marginal-Contributions}}.
\section{Related Work} \label{sec:RelatedWork}
The recent rise of explainable AI has incentivized the research on approximation methods for the Shapley value leading to a variety of different algorithms for this purpose.
The first distinction to be made is between those that are domain-independent, i.e., able to deal with any cooperative game, and those that are tailored to a specific use case, e.g.\ assigning Shapley values to single neurons in neural networks, or which impose specific assumptions on the value function.
In this paper, we will consider only the former, as it is our goal to provide approximations algorithms independent of the context in which they are applied.
The first and so far simplest of this kind is \emph{ApproShapley} \cite{Castro.2009}, which samples marginal contributions from each player based on randomly drawn permutations of the player set.
The variance of each of its \mbox{Shapley} value estimates is bounded by $\mathcal{O}(\frac{n}{T})$.
\emph{Stratified Sampling}  \citep{Maleki.2013} and \emph{Structured Sampling} \citep{vanCampen.2018} both partition the marginal contributions of each player by coalition size in order to stratify the marginal contributions of the population from which to draw a sample, which leads to a variance reduction.
While \emph{Stratified Sampling} calculates a sophisticated allocation of samples for each coalition size, \emph{Structured Sampling} simply samples with equal frequencies.
Multiple follow-up works suggest specific techniques to improve the sampling allocation over the different coalition sizes \citep{OBrien.2015,Castro.2017,Burgess.2021}. 

In order to reduce the variance of the naive sampling approach underlying \emph{ApproShapley}, \citet{Illes.2019} suggest to use ergodic sampling, i.e., generating samples that are not independent but still satisfy the strong Law of Large numbers.
Quite recently, \citet{Mitchell.2022} investigated two techniques for improving \emph{ApproShapley}'s sampling approach. 
One is based on the theory of \mbox{reproducing} kernel Hilbert spaces, which focuses on minimizing the discrepancies for functions of permutations.  
The other exploits a geometrical connection between uniform sampling on the Euclidean sphere and uniform sampling over permutations.  

Adopting a Bayesian perspective, i.e.,
by viewing the Shapley values as random variables, \citet{Touati.2021} consider approximating the Shapley values by Bayesian estimates (posterior mean, mode, or median), where each posterior distribution of a player's Shapley value depends on the \mbox{remaining} ones. 
Utilizing a representation of the Shapley value as an integral \cite{Owen.1972}, \emph{Owen Sampling} \citep{Okhrati.2020} approximates this integral by sampling marginal contributions using antithetic sampling \citep{Rubinstein.2016,Lomeli.2019} for variance reduction.

A fairly new class of approaches that dissociates itself from the notion of marginal contribution are those that view the Shapley value as a solution of a quadratic program with equality constraints \cite{Lundberg.2017, Simon.2020, Covert.2021}.
Another unorthodox approach is to divide the player set into small enough groups for which the Shapley values within these groups can be computed exactly \cite{Soufiani.2014,Corder.2019}.
For an overview of approaches related to machine learning we refer to \cite{Chen.2022}. 
\section{Problem Statement} \label{sec:ProblemStatement}

The formal notion of a cooperative game is defined by a tuple $(\mathcal{N}, \nu)$ consisting of a set of players $\mathcal{N} = \{1,\ldots,n\}$ and a value function $\nu : \mathcal{P}(\mathcal{N}) \to \mathbb{R}$ that assigns to each subset of $\mathcal{N}$ a real-valued number.
The value function must satisfy $\nu(\emptyset) = 0$.
We call the subsets of $\mathcal{N}$ coalitions, $\mathcal{N}$ itself the grand coalition, and the assigned value $\nu(S)$ to a coalition $S \subseteq \mathcal{N}$ its worth.
Given a cooperative game $(\mathcal{N}, \nu)$, the Shapley value assigns each player a share of the grand coalition's worth.
In particular, the Shapley value \cite{Shapley.1953} of any player $i \in \mathcal{N}$ is defined as
\begin{equation}
    \phi_i = \sum\limits_{S \subseteq \mathcal{N}_i} \frac{1}{n \cdot \binom{n-1}{|S|}} \left[ \nu(S \cup \{i\}) - \nu(S) \right] ,
\end{equation}
where $\mathcal{N}_i := \mathcal{N} \setminus \{i\}$ for each player $i \in \mathcal{N}.$ 
The term
$\nu(S \cup \{i\}) - \nu(S)$
is also known as player $i$'s marginal contribution to $S \subseteq \mathcal{N}_i$ and captures the increase in collective benefit when player $i$ joins the coalition $S$.
Thus, the Shapley value can be seen as the weighted average of a player's marginal contributions. 

The exact computation of all Shapley values requires the knowledge of the values of all $2^n$ many coalitions\footnote{In fact, only $2^n-1$ many coalitions, as $\nu(\emptyset)=0$ is known.} and is shown to be NP-hard \cite{Deng.1994}.
In light of the exponential computational effort w.r.t.\ to $n$, we consider the goal of approximating the Shapley value of all players as precisely as possible for a given \emph{budget} of $T \in \mathbb{N}$ many evaluations (accesses) of $\nu$ in discrete time steps $1, \ldots, T$.
Since $\nu(\emptyset) = 0$ holds by definition, the evaluation of $\nu(\emptyset)$ comes for free without any budget cost.
We judge the quality of the estimates $\hat\phi_1, \ldots, \hat\phi_n$\,---\,which are possibly of stochastic nature\,---\,obtained by an approximation algorithm after $T$ many evaluations by two criteria that have to be minimized for all $i \in \mathcal{N}$. 
First, the mean squared error (MSE) of the estimate $\hat{\phi}_i$ is given by
\begin{equation}
    \mathbb{E} \big[ \big( \hat\phi_i - \phi_i \big)^2 \big] \, .
\end{equation}
Utilizing the bias-variance decomposition allows us to reduce the squared error to the variance $\mathbb{V} [\hat\phi_i]$ of the Shapley value estimate in case that it is unbiased, i.e. $\mathbb{E} [\hat\phi_i] = \phi_i$.
The second criterion is the probability of $\hat{\phi}_i$ deviating from $\phi_i$ by more than a fixed $\varepsilon > 0$:
\begin{equation}
    \mathbb{P} \big( | \hat\phi_i - \phi_i | > \varepsilon \big) \, .
\end{equation}
Both criteria are well-established for measuring the quality of an algorithm approximating the Shapley value. 
\section{SVARM} \label{sec:SVARM}

\begin{figure*}[t]
\centering
\includegraphics[width=0.85\textwidth]{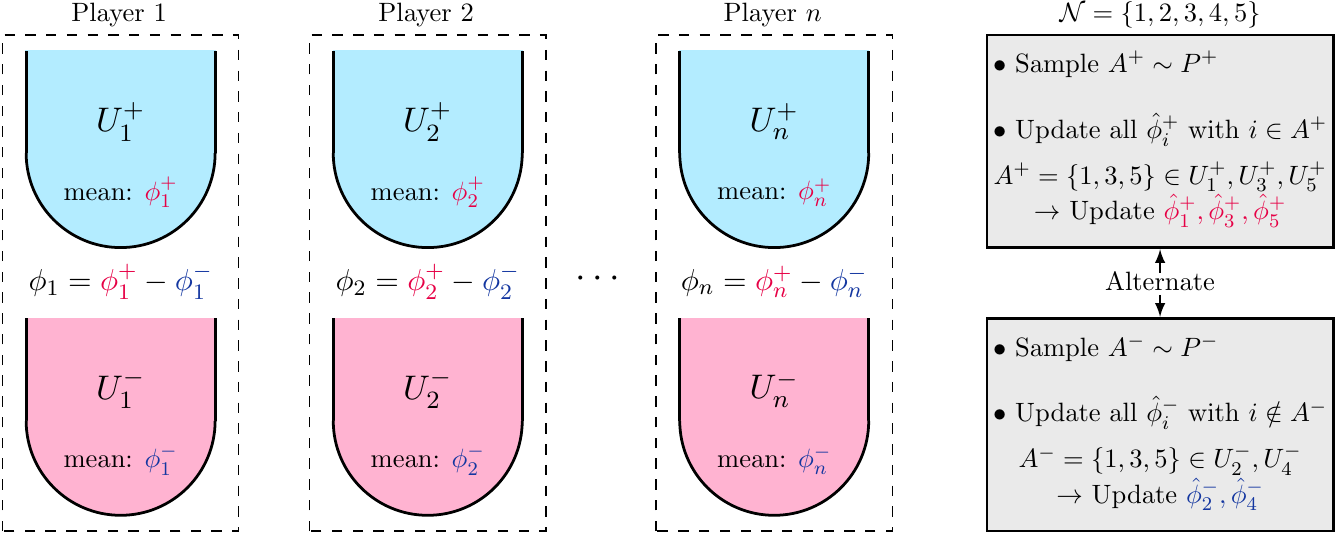}
\caption{Illustration of SVARM's sampling process and update rule: Each player $i$ has two urns $U_i^+ := \{S \cup \{i\} \mid S \subseteq \mathcal{N}_i\}$ and $U_i^- := \{S \mid S \subseteq \mathcal{N}_i\}$ containing marbles which represent coalitions, with mean coalition worth $\phi_i^+$ and $\phi_i^-$.
SVARM alternates between sampling coalitions $A^+ \sim P^+$ and $A^- \sim P^-$.
With each drawn coalition all estimates of those urns are updated which contain the corresponding marble.
Since each player's two urns form a partition of the powerset $\mathcal{P}(\mathcal{N})$, all players have exactly one urn updated with each sample.}
\label{fig:SVARM}
\end{figure*}

Thanks to the distributive law, the formula of the Shapley value for a player $i$ can be rearranged so that it is not its weighted average of marginal contributions, but the difference of the weighted average of coalition values by adding $i$ and the weighted average of coalition values without $i$:
\begin{equation} \label{eq:ShapleyTwoSums}
    \phi_i = \underbrace{\sum\limits_{S \subseteq \mathcal{N}_i} w_S \cdot \nu(S \cup \{i\})}_{=: \, \phi_i^+} - 
    \underbrace{\sum\limits_{S \subseteq \mathcal{N}_i} w_S \cdot \nu(S)}_{=: \, \phi_i^-} \, ,
\end{equation} 
with weights $w_S = \frac{1}{n \cdot \binom{n-1}{|S|}}$ for each $S \subseteq \mathcal{N}_i$.
We call $\phi_i^+$ the positive and $\phi_i^-$ the negative Shapley value, while we refer to the collective of both as the signed Shapley values.
The weighted averages $\phi_i^+$ and $\phi_i^-$ can also be viewed as expected values, i.e., $\phi_i^+ = \mathbb{E} [\nu(\mathcal{S} \cup \{i\})]$ and $\phi_i^- = \mathbb{E} [\nu(\mathcal{S})],$ where $\mathcal{S} \sim P^w$ and $P^w(S) = w_S$ for all $S \subseteq \mathcal{N}_i$.
Note that all weights add up to 1 and thus $P^w$ forms a well-defined probability distribution.
In this way, we can approximate each signed Shapley value separately using estimates $\hat\phi_i^+$ and $\hat\phi_i^-$ and combine them into a Shapley value estimate by means of $\hat\phi_i = \hat\phi_i^+ - \hat\phi_i^-$.

In light of this, a naive approach for approximating each signed Shapley value of a player is by sampling some number of $M$ many coalitions $S^{(1)}, \ldots, S^{(M)}$ with distribution $P^w$ and using the sample mean as the estimate, i.e., $\hat\phi_i^+ = \frac{1}{M} \sum_{m=1}^M \nu(S^{(m)} \cup \{i\})$.
However, this would require all $2n$ signed Shapley values (two per player) to be estimated separately by sampling coalitions in a dedicated manner, each of which would lead to an update of only one estimate.
This ultimately slows down the convergence of the estimates, especially for large $n$.

On the basis of the aforementioned representation of the Shapley value, we present the \textit{Shapley Value Approximation without Requesting Marginals} (SVARM) algorithm, a novel approach that updates multiple Shapley value estimates at once with a single evaluation of $\nu$.
Its novelty consists of sampling coalitions independently from two specifically chosen distributions $P^+$ and $P^-$ in an alternating fashion, which allows for a more powerful update rule:
each (independently) sampled coalition $A^+$ from $P^+$ allows one to update all positive Shapley value estimates $\hat\phi_i^+$ of all payers $i$ which are contained in $A^+$, i.e., $i \in A^+$.
Likewise, for a coalition $A^-$ drawn from $P^-,$ all negative Shapley value estimates $\hat\phi_i^-$ for $i \notin A^-$ can be updated.

It is worth noting that, for simplicity, we alternate evenly between the samples from the $P^+$ and $P^-$ distributions, although one could also use a ratio other than $\nicefrac{1}{2}$.
To avoid a bias, both distributions have to be tailored such that the following holds for all $i \in \mathcal{N}$ and $S \subseteq \mathcal{N}_i$:
\begin{equation}
    \mathbb{P}( A^+ = S \cup \{i\} \mid i \in A^+) = \mathbb{P}( A^- = S \mid i \notin A^-) = w_S \, .
\end{equation}
For this reason, we define the probability distributions over coalitions to sample from as
\begin{align} \label{eq:SVARMDistribution}
    & P^+(S) := \frac{1}{|S| \binom{n}{|S|}  H_n} & \forall S \in \mathcal{P}(\mathcal{N})\setminus \{\emptyset\}, \\
    & P^-(S) := \frac{1}{(n-|S|) \binom{n}{|S|} H_n} & \forall  S \in \mathcal{P}(\mathcal{N})\setminus \{\mathcal{N}\} ,
\end{align}
where $H_n = \sum_{k=1}^n 1/k$ denotes the $n$-th harmonic \mbox{number}.
Note that both $P^+$  and $P^-$ assign equal probabilities to coalitions of the same size, so that one can first sample the size and then draw a set uniformly of that size.
This pair of distributions is provably the only one to fulfill the required property (see \cref{subsec:unbiasednessSVARM}).

\noindent
The approach of dividing the Shapley value into two parts and approximating both has already been pursued (although not as formally rigorous) via importance sampling \cite{Covert.2019a}, allowing to update all $n$ estimates with each sample.
\citet{Wang.2023} adopt the same representation for the Banzhaf value, and coined the strategy of updating all players' estimates with each sampled coalition the \textit{maximum sample reuse} (MSR) principle.
Their approximation algorithm is specifically tailored to the Banzhaf value as it leverages its uniform weights $w_S = \frac{1}{2^{n-1}}$ and is thus, at least not directly, transferable to the Shapley value.

\begin{algorithm}[tb]
  \caption{SVARM}
  \label{alg:SVARM}
\textbf{Input}: $\mathcal{N}$, $T \in \mathbb{N}$
\begin{algorithmic}[1]
    \STATE $\hat\phi_i^+, \hat\phi_i^- \leftarrow 0$ for all $i \in \mathcal{N}$
    \STATE $c_i^+, c_i^- \leftarrow 1$ for all $i \in \mathcal{N}$
    \STATE \textsc{\texttt{WarmUp}}
    \STATE $t \leftarrow 2n$ 
    \WHILE{$t + 2 \leq T$}
        \STATE Draw $A^+ \sim P^+$
        \STATE Draw $A^- \sim P^-$
        \STATE $v^+ \leftarrow \nu(A^+)$
        \STATE $v^- \leftarrow \nu(A^-)$
        \FOR{$i \in A^+$}
             \STATE $\hat{\phi}_i^+ \leftarrow \frac{ c_i^+ \hat\phi_i^+ + v^+}{c_i^+ + 1}$
             \STATE $c_i^+ \leftarrow c_i^+ + 1$
        \ENDFOR
        \FOR{$i \in \mathcal{N} \setminus A^-$}
             \STATE $\hat\phi_i^- \leftarrow \frac{ c_i^- \hat\phi_i^- + v^-}{c_i^- + 1}$
             \STATE $c_i^- \leftarrow c_i^- + 1$
        \ENDFOR
        \STATE $t \leftarrow t+2$
    \ENDWHILE
    \STATE $\hat\phi_i \leftarrow \hat\phi_i^+ - \hat\phi_i^-$ for all $i \in \mathcal{N}$
\end{algorithmic}
\textbf{Output}: $\hat\phi_1, \ldots, \hat\phi_n$
\end{algorithm}

In the following we describe SVARM's procedure with the pseudocode of \cref{alg:SVARM}.
The overall idea of the sampling and update process is illustrated in \cref{fig:SVARM}.
It starts by initializing the positive and negative Shapley value estimates $\hat\phi_i^+$ and $\hat\phi_i^-$, and the number of samples $c_i^+$ and $c_i^-$ collected for each player $i$.
SVARM continues by launching a warm-up phase (see \cref{alg:SVARMWarmup} in \cref{app:Pseudocode}).
In the main loop, the update rule is applied for as many sampled pairs of coalitions $A^+$ and $A^-$ as possible until SVARM runs out of budget.
In each iteration $A^+$ is sampled from $P^+$ and $A^-$ from $P^-$.
The worth of $A^+$ and $A^-$ is evaluated and stored in $v^+$ and $v^-$, requiring two accesses to the value function.
The estimate $\hat\phi_i^+$ of each player $i \in A^+$ is updated with the worth $\nu(A^+)$ such that $\hat\phi_i^+$ is the mean of sampled coalition values.
Likewise, the estimate $\hat\phi_i^-$ of each player $i \notin A^-$ is updated with the worth $\nu(A^-)$. 
At the same time, the sample numbers of the respective signed Shapley value estimates are also updated.
Finally, SVARM computes its Shapley value estimate $\hat\phi_i$ of $\phi_i$ for each $i$ according to \cref{eq:ShapleyTwoSums}.
Note that since only the quantities $\hat\phi_i^+, \hat\phi_i^-, c_i^+$, and $c_i^+$ are stored for each player, its space complexity is in $\mathcal{O}(n)$.
Moreover, SVARM is incremental and can be stopped at any time to return its estimates after executing line 20, or it can be run further with increased budget.

\paragraph{Theoretical analysis.}
In the following we present theoretical results for SVARM.
All proofs are given in \cref{app:SVARMAnalysis} of the technical appendix.
For the remainder of this section we assume that a minimum budget of $T \geq 2n+2$ is given.
This assumption guarantees the completion of the warm-up phase such that each positive and negative Shapley value estimate has at least one sample and an additional pair sampled in the loop.
The lower bound on $T$ is essentially twice the number of players $n,$ which is a fairly weak assumption.
We denote by $\bar{T} := T - 2n$ the number of time steps (budget) left after the warm-up phase.
Moreover, we assume $\bar{T}$ to be even for sake of simplicity such that a lower bound on the number of sampled pairs in the main part can be expressed by $\frac{T}{2} - n$.
We begin with the unbiasedness of the estimates maintained by SVARM allowing us later to reduce the mean squared error (MSE) of each estimate to its variance.

\begin{theorem} \label{the:SVARMUnbiased}
    The Shapley value estimate $\hat\phi_i$ of any $i \in \mathcal{N}$ obtained by SVARM is unbiased, i.e.,
    \begin{equation*}
        \mathbb{E}[\hat\phi_i] = \phi_i \, .
    \end{equation*}
\end{theorem}

\noindent
Next, we give a bound on the variance of each Shapley value estimate.
For this purpose, we introduce notation for the variances of coalition values contained in $\phi_i^+$ and $\phi_i^-$.
For a random set $A_i \subseteq \mathcal{N}_i$ distributed  according to $P^w$ let
\begin{equation}
    {\sigma_i^+}^2 := \mathbb{V} \left[ \nu(A_i \cup \{i\}) \right]
    \text{ and } {\sigma_i^-}^2 := \mathbb{V} \left[ \nu(A_i) \right] \, .
\end{equation}

\begin{theorem} \label{the:SVARMVariance}
    The variance of any player's Shapley value estimate $\hat\phi_i$ obtained by SVARM is bounded by
    \begin{equation*}
        \mathbb{V} [ \hat\phi_i ] \leq \frac{2 H_n}{\bar{T}} ({\sigma_i^+}^2 + {\sigma_i^-}^2) \, .
    \end{equation*}
\end{theorem}

\noindent
Combining the unbiasedness in \cref{the:SVARMUnbiased} with the latter variance bound implies the following result on the MSE.

\begin{corollary} \label{cor:SVARMSE}
    The MSE of any player's Shapley value estimate $\hat\phi_i$ obtained by SVARM is bounded by
    \begin{equation*}
        \mathbb{E} \big[ \big( \hat\phi_i - \phi_i \big)^2 \big] \leq \frac{2 H_n}{\bar{T}} ({\sigma_i^+}^2 + {\sigma_i^-}^2) \, .
    \end{equation*}
\end{corollary}

\noindent
Assuming that each variance term ${\sigma_i^+}^2$ and ${\sigma_i^-}^2$ is bounded by some constant independent of $n$ (and $T$), the MSE bound in \cref{cor:SVARMSE} is in $\mathcal{O}(\frac{\log n}{T -n})$ and so is the variance bound in \cref{the:SVARMVariance}.
Note that this assumption is rather mild and satisfied if the underlying value function is bounded by constants independent of $n$, which again is the case for a wide range of games and in particular in explainable AI for global and local feature importance based on classification probabilities lying between 0 and 1.
Further, as $T$ is growing linearly with $n$ by assumption, the denominator is essentially driven by the asymptotics of $T.$
Thus, the dependency on $n$ is logarithmic, which is a significant improvement over existing theoretical results having a linear dependency on $n$ like $\mathcal{O}(\frac{n}{T})$ for \emph{ApproShapley} \citep{Castro.2009} or possibly worse \citep{Simon.2020}.
Finally, we present two probabilistic bounds on the approximated Shapley value.
The first utilizes the variance bound shown in \cref{the:SVARMVariance} by applying Chebyshev's inequality.

\begin{theorem} \label{the:SVARMCheby}
    The probability that the Shapley value estimate $\hat\phi_i$ of any fixed player $i \in \mathcal{N}$ deviates from $\phi_i$ by a margin of any fixed $\varepsilon > 0$ or greater is bounded by
    \begin{equation*}
        \mathbb{P} ( | \hat\phi_i - \phi_i | \geq \varepsilon ) \leq \frac{2 H_n}{\varepsilon^2 \bar{T}} ( {\sigma_i^-}^2 + {\sigma_i^+}^2 ) \, .
    \end{equation*}
\end{theorem}

\noindent
The presented bound is in $\mathcal{O}(\frac{\log n}{T - n})$ and improves upon the bound derived by Chebyshev's inequality of $\mathcal{O}(\frac{n}{T})$ for \emph{ApproShapley} \cite{Maleki.2013}.
Our second bound derived by Hoeffding's inequality is tighter, but requires the introduction of notation for the ranges of $\nu(A_i)$ and $\nu(A_i \cup \{i\})$:
\begin{align}
    & r_i^+ := \max\limits_{S \subseteq \mathcal{N}_i} \nu(S \cup \{i\}) - \min\limits_{S \subseteq \mathcal{N}_i} \nu(S \cup \{i\}) \, ,\\
    & r_i^- := \max\limits_{S \subseteq \mathcal{N}_i} \nu(S) - \min\limits_{S \subseteq \mathcal{N}_i} \nu(S) \, .
\end{align}
\begin{theorem} \label{the:SVARMHoeffding}
    The probability that the Shapley value estimate $\hat\phi_i$ of any fixed player $i \in \mathcal{N}$ deviates from $\phi_i$ by a margin of any fixed $\varepsilon > 0$ or greater is bounded by
    \begin{equation*}
        \mathbb{P} ( | \hat\phi_i - \phi_i | \geq \varepsilon ) \leq 2 e^{- \frac{\bar{T}}{4 {H_n}^2}} + 4 \frac{e^{ -\Psi \left\lfloor \frac{\bar{T}}{4 H_n} \right\rfloor}}{e^{\Psi} - 1},
    \end{equation*}
    where $\Psi = \nicefrac{2 \varepsilon^2}{(r_i^+ + r_i^-)^2}.$
\end{theorem}
\noindent
Note that this bound is exponentially decreasing with $T$ and can be expressed asymptotically as $\mathcal{O}(e^{-\frac{T-n}{(\log n)^2}})$.
In comparison, the bounds of $\mathcal{O}(e^{- \frac{T}{n}})$ for \emph{ApproShapley}, $\mathcal{O}(n e^{-\frac{T}{n^3}})$ for \emph{Stratified Sampling} \cite{Maleki.2013}, and the projected SGD variant \cite{Simon.2020} show worse asymptotic dependencies on $n$ in comparison.
\section{Stratified SVARM} \label{sec:StratifiedSVARM}

\begin{figure*}[t]
\centering
\includegraphics[width=1.0\textwidth]{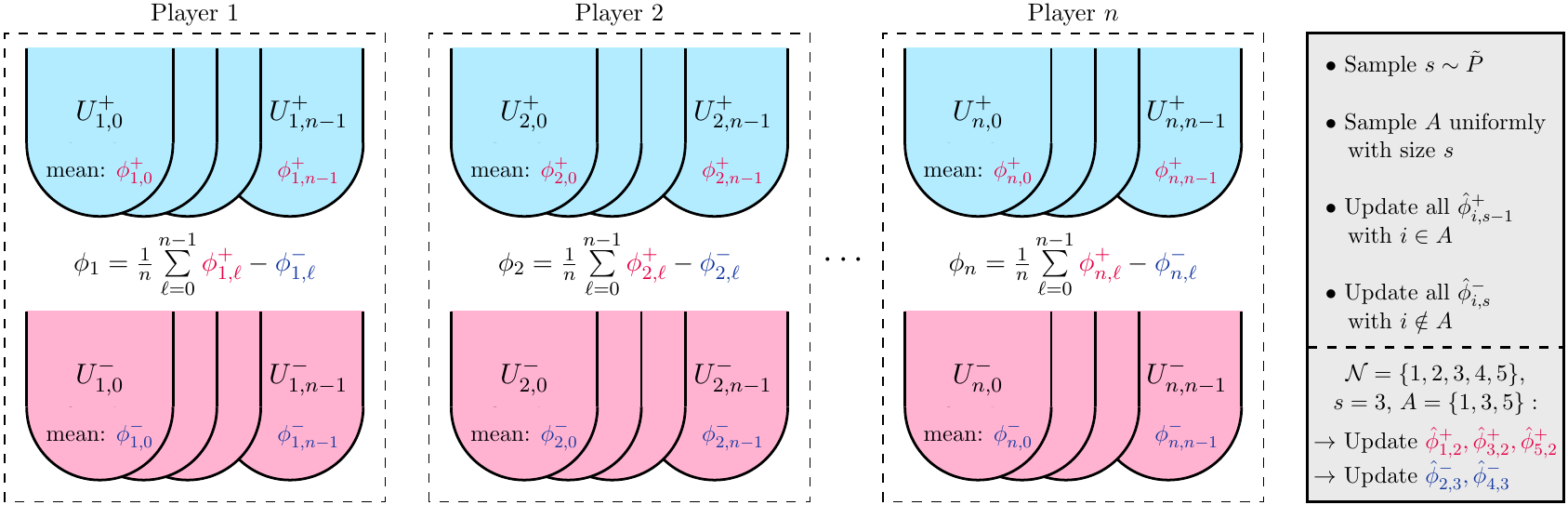}
\caption{Illustration of Stratified SVARM's sampling process and update rule: Each player i has urns $U_{i,\ell}^+ := \{S \cup \{i\} \mid S \subseteq \mathcal{N}_i, |S|=\ell\}$ and \mbox{$U_{i,\ell}^- := \{S \mid S \subseteq \mathcal{N}_i, |S| = \ell \}$} for all $\ell \in \{0,\ldots,n-1\}$, $2n$ in total, containing marbles which represent coalitions, with mean coalition worth $\phi_{i,\ell}^+$ and $\phi_{i,\ell}^-$.
Stratified SVARM samples in each time step $t$ a coalition $A_t \subseteq \mathcal{N}$ and updates the estimates of all players' urns that contain the corresponding marble.
Since each player's urns form a partition of the powerset $\mathcal{P}(\mathcal{N})$, all players have exactly one urn updated with each sample.
}
\label{fig:StratifiedSVARM}
\end{figure*}

On the basis of the representation of the Shapley value in \cref{eq:ShapleyTwoSums}, we develop another approximation algorithm named \emph{Stratified SVARM} to further pursue and reach the maximum sample reuse principle.
Its crux is a refinement~of \emph{SVARM} obtained by stratifying the positive and the \mbox{negative} Shapley value $\phi_i^+$ and $\phi_i^-$. We exploit the latter to develop an even more powerful update rule that allows for updating~all players simultaneously with each single coalition sampled.
Both, $\phi_i^+$ and $\phi_i^-$ can be rewritten using stratification such that each becomes an average of strata, whereas the strata themselves are averages of the coalitions' worth:
\begin{align}
   \phi_i^+ & = \frac{1}{n} \sum\limits_{\ell=0}^{n-1}  \frac{1}{\binom{n-1}{\ell}}\sum\limits_{\substack{S \subseteq \mathcal{N}_i \\ |S| = \ell}} \nu(S \cup \{i\}) =: \frac{1}{n} \sum\limits_{\ell=0}^{n-1} \phi_{i,\ell}^+ \, , \\
    \phi_i^- & = \frac{1}{n} \sum\limits_{\ell=0}^{n-1} \frac{1}{\binom{n-1}{\ell}} \sum\limits_{\substack{S \subseteq \mathcal{N}_i \\ |S| = \ell}} \nu(S) \hphantom{\cup \{i\}} =: \frac{1}{n} \sum\limits_{\ell=0}^{n-1} \phi_{i,\ell}^- \, .
\end{align}
We call $\phi_{i,\ell}^+$ the $\ell$-th positive Shapley subvalue and $\phi_{i,\ell}^-$ the $\ell$-th negative Shapley subvalue for all $\ell \in \mathcal{L} := \{0,\ldots,n-1\}$.
Now, we can write $\phi_i$ as
\begin{equation}
    \phi_i = \frac{1}{n} \sum\limits_{\ell = 0}^{n-1} \phi_{i,\ell}^+ - \phi_{i,\ell}^- \, .
\end{equation}
Note that this representation of $\phi_i$ coincides with Equation~6 in \cite{Ancona.2019}.
Intuitively speaking at the example of $\phi_i^+$ (and analogously for $\phi_i^-$), we partition the population of coalitions contained in $\phi_i^+$ into $n$ strata.
Each \emph{stratum} $\phi_{i,\ell}^+$ comprises all coalitions which include the player $i$ and have cardinality $\ell+1$.
Instead of sampling directly for $\phi_i^+$, the stratification allows one to sample coalitions from each stratum, obtain mean estimates $\hat\phi_{i,\ell}^+$, and aggregate them to
\begin{equation}
    \hat\phi_i^+ = \frac{1}{n} \sum\limits_{\ell=0}^{n-1} \hat\phi_{i,\ell}^+
\end{equation}
in order to obtain an estimate for $\phi_i^+$.
Due to the increase in homogeneity of the strata in comparison to their origin population, caused by the shared size and inclusion or exclusion of $i$ for coalitions in the same stratum, one would expect the strata to have significantly lower variances and ranges resulting in approximations of better quality compared to SVARM.
In combination with our bounds shown in \cref{the:SVARMVariance} and \cref{the:SVARMHoeffding}, this should result in approximations of better quality.
In the following we present further techniques for improvement which we apply for Stratified SVARM (\cref{alg:StartifiedSVARM}).

\paragraph{Exact calculation.}
First, we observe that some~strata contain very few coalitions.
Thus, we calculate $\phi_{i,0}^+, \phi_{i,n-2}^+, \phi_{i,n-1}^+, \phi_{i,1}^-$, and $\phi_{i,n-1}^-$ for all players exactly by evaluating $\nu$ for all coalitions of size $1,n-1$, and $n$.
This requires $2n+1$ many evaluations of $\nu$ (see \cref{alg:ExactCalculation} in \cref{app:Pseudocode}).
We already obtain $\phi_{i,0}^- = \nu(\emptyset) = 0$ by definition.
As a consequence, we can exclude the sizes $0,1,n-1$, and $n$ from further consideration.
We assume for the remainder that $n \geq 4$, otherwise we would have already calculated all Shapley values exactly.

\paragraph{Refined warm-up.}
Next, we split the warm-up into two parts, one for the positive, the other for the negative Shapley subvalues (see \cref{alg:StrataWarmupPos} and \ref{alg:StrataWarmupNeg} in \cref{app:Pseudocode}).
Each collects for each estimate $\hat\phi_{i,\ell}^+$ or $\hat\phi_{i,\ell}^-$, respectively, one sample and consumes a budget of $\sum\nolimits_{s=2}^{n-2} \left\lceil \frac{n}{s} \right\rceil$.

\begin{algorithm}[tb]
  \caption{Stratified SVARM}
  \label{alg:StartifiedSVARM}
\textbf{Input}: $\mathcal{N}$, $T \in \mathbb{N}$
\begin{algorithmic}[1]
    \STATE $\hat{\phi}_{i,\ell}^+, \hat{\phi}_{i,\ell}^- \leftarrow 0$ for all $i \in \mathcal{N}$ and $\ell \in \mathcal{L}$
    \STATE $c_{i,\ell}^+, c_{i,\ell}^- \leftarrow 0$ for all $i \in \mathcal{N}$ and $\ell \in \mathcal{L}$
    \STATE \textsc{\texttt{ExactCalculation}}$(\mathcal{N})$
    \STATE \textsc{\texttt{WarmUp}}$^+(\mathcal{N})$
    \STATE \textsc{\texttt{WarmUp}}$^-(\mathcal{N})$
    \STATE $t \leftarrow 2n + 1 + 2 \sum\limits_{s=2}^{n-2} \lceil \frac{n}{s} \rceil$ 
    \WHILE{$t < T$}
        \STATE Draw $s_t\sim \tilde P$
        \STATE Draw $A_t$ from $\{S \subseteq \mathcal{N} \mid |S| = s_t \}$ uniformly
        \STATE \textsc{\texttt{Update}}$(A_t)$
        \STATE $t \leftarrow t + 1$
    \ENDWHILE
    \STATE $\hat{\phi}_i \leftarrow \frac{1}{n} \sum\limits_{\ell=0}^{n-1} \hat{\phi}_{i,\ell}^+ - \hat{\phi}_{i,\ell}^-$ for all $i \in \mathcal{N}$
\end{algorithmic}
\textbf{Output}: $\hat\phi_1, \ldots, \hat\phi_n$
\end{algorithm}

\paragraph{Enhanced update rule.}
Thanks to the stratified representation of the Shapley value, we can enhance SVARM's update rule and update with each sampled coalition $A_t \subseteq \mathcal{N}$ the estimates $\hat{\phi}_{i,|A_t|-1}^+$ for all $i \in A_t$ and $\hat{\phi}_{i,|A_t|}^-$ for all $i \notin A_t$.
Thus, we can update all
estimates $\hat\phi_i$ at once with a single sample.
This enhanced update step is given in \cref{alg:Update} (see \cref{app:Pseudocode}) and illustrated in \cref{fig:StratifiedSVARM}.
In order to obtain unbiased estimates, it suffices to select an arbitrary size $s$ of the coalition $A$ to be sampled and draw $A$ \mbox{uniformly} at random from the set of coalitions with size $s$.
We go one step further and choose not only the coalition $A$, but also the size $s$ randomly according to a specifically tailored probability distribution $\tilde P$ over $\{2,\ldots,n-2\}$, which leads to simpler bounds in our theoretical analysis in which each stratum receives the same weight.
We define for $n$ even:
\begin{equation}
    \begin{split}
        \tilde P(s) := \begin{cases}
        \frac{n \log n -1}{2sn \log n \left( H_{\frac{n}{2}-1} - 1\right)} & \text{if } s \leq \frac{n-2}{2} \\ 
        \frac{1}{n \log n} & \text{if } s = \frac{n}{2} \\
        \frac{n \log n -1}{2(n-s)n \log n \left( H_{\frac{n}{2}-1} - 1\right)} & \text{otherwise} 
    \end{cases}
    \end{split}
    \text{ and for } n \text{ odd: }
    \tilde P(s) := \begin{cases}
    \frac{1}{2s \left(H_{\frac{n-1}{2}}-1\right)} & \text{if } s \leq \frac{n-1}{2} \\ 
    \frac{1}{2(n-s) \left(H_{\frac{n-1}{2}}-1\right)} & \text{otherwise}
    \end{cases} \, .
\end{equation}
Note that Stratified SVARM is incremental just as SVARM, but in contrast, requires quadratic space $\mathcal{O}(n^2)$ as it stores estimates and counters for each player \emph{and} stratum.

\paragraph{Theoretical analysis.}
Similar to SVARM, we present in the following our theoretical results for Stratified SVARM.
All proofs are given in \cref{app:StratSVARMAnalysis}.
Again, we assume a minimum budget of $T \geq 2n + 1 + 2 \sum_{s=2}^{n-2} \left\lceil \frac{n}{s} \right\rceil =: W \in \mathcal{O} (n \log n)$, guaranteeing the completion of the warm-up phase,  and denote by $\bar{T} = T - W$ the budget left after the warm-up phase.
We start by showing that Stratified SVARM is not afflicted with any bias.
\begin{theorem} \label{the:StratSVARMUnbiased}
    The Shapley value estimate $\hat\phi_i$ of any $i \in \mathcal{N}$ obtained by Stratified SVARM is unbiased, i.e.,
    \begin{equation*}
       \mathbb{E}[\hat\phi_i] = \phi_i \, .
    \end{equation*}
\end{theorem}
\noindent
Next, we consider the variance of the Shapley value estimates and quickly introduce some notation.
Let $A_{i,\ell} \subseteq \mathcal{N}_i$ be a random coalition of size $\ell$ distributed with \mbox{$\mathbb{P}(A_{i,\ell} = S) = \binom{n-1}{\ell}^{-1}$}.
Define the strata variances
\begin{equation}
    {\sigma_{i,\ell}^+}^2 := \mathbb{V} \left[ \nu(A_{i,\ell} \cup \{i\}) \right] \text{ and } {\sigma_{i,\ell}^-}^2 := \mathbb{V} \left[ \nu(A_{i,\ell}) \right] \, .
\end{equation}

\begin{theorem} \label{the:StratSVARMVariance}
    The variance of any player's Shapley value estimate $\hat\phi_i$ obtained by Stratified SVARM is bounded by
    \begin{equation*}
        \mathbb{V} [ \hat\phi_i ] \leq \frac{2 \log n}{n \bar{T}} \sum\limits_{\ell=1}^{n-3} {\sigma_{i,\ell}^+}^2 + {\sigma_{i,{\ell+1}}^-}^2 \, .
    \end{equation*}
\end{theorem}

\noindent
Together with the unbiasedness shown in \cref{the:StratSVARMUnbiased}, the variance bound implies the following MSE bound.
\begin{corollary} \label{cor:StratSVARMSE}
    The MSE of any player's Shapley value estimate $\hat\phi_i$ obtained by Stratified SVARM is bounded by
    \begin{equation*}
        \mathbb{E} [ ( \hat\phi_i - \phi_i )^2 ] \leq \frac{2 \log n}{n \bar{T}} \sum\limits_{\ell=1}^{n-3} {\sigma_{i,\ell}^+}^2 + {\sigma_{i,{\ell+1}}^-}^2 \, .
    \end{equation*}
\end{corollary}
\noindent
With our choice of the sampling distribution $\tilde{P}$ we achieved an easily interpretable bound on the MSE in which each stratum variance is equally weighted.
Assuming that each stratum variance is bounded by some constant independent of $n$, the MSE bound in \cref{cor:StratSVARMSE} is in $\mathcal{O}(\frac{\log n}{T - n\log n}).$
Note that, by assumption, $T$ is growing log-linearly with $n$ so that the denominator is essentially driven by the asymptotics of $T$.
Again, compared to existing theoretical results, with linear dependence on $n$, the logarithmic dependence on $n$ is a significant improvement. 
Still, it is worth emphasizing that the more homogeneous strata with lower variances constitute the core improvement of Stratified SVARM, which are not reflected within the O-notation.
Our first probabilistic bound is obtained by Chebyshev's inequality and the bound from \cref{the:StratSVARMVariance}.
\begin{theorem} \label{the:StratSVARMCheby}
    The probability that the Shapley value estimate $\hat\phi_i$ of any fixed player $i \in \mathcal{N}$ deviates from $\phi_i$ by a margin of any fixed $\varepsilon > 0$ or greater is bounded by
    \begin{equation*}
        \mathbb{P} ( | \hat\phi_i - \phi_i | \geq \varepsilon ) \leq \frac{2 \log n}{\varepsilon^2 n \bar{T}} \sum\limits_{\ell=1}^{n-3} {\sigma_{i,\ell}^+}^2 + {\sigma_{i,{\ell+1}}^-}^2 \, .
    \end{equation*}
\end{theorem}
\noindent
Lastly, our second probabilistic bound 
derived via Hoeffding's inequality is tighter, but less trivial.
It requires some further notation, namely the ranges of the strata values:
\begin{align}
    & r_{i,\ell}^+ := \max\limits_{S \subseteq \mathcal{N}_i : |S| = \ell} \nu(S \cup \{i\}) - \min\limits_{S \subseteq \mathcal{N}_i : |S| = \ell} \nu(S \cup \{i\}) \, , \\
    & r_{i,\ell}^- := \max\limits_{S \subseteq \mathcal{N}_i : |S| = \ell} \nu(S) - \min\limits_{S \subseteq \mathcal{N}_i : |S| = \ell} \nu(S) \, .
\end{align}
\begin{theorem} \label{the:StratSVARMHoeffding}
    The probability that the Shapley value estimate $\hat\phi_i$ of any fixed player $i \in \mathcal{N}$ deviates from $\phi_i$ by a margin of any fixed $\varepsilon > 0$ or greater is bounded by
    \begin{equation*}
        \mathbb{P} ( | \hat\phi_i - \phi_i | \geq \varepsilon ) \leq 2(n-3) \left( e ^{- \frac{\bar{T}}{8 n^2 (\log n)^2}} + 2 \frac{ e^{ -\Psi \left\lfloor \frac{\bar{T}}{4n \log n} \right\rfloor}}{e^{\Psi}-1} \right)\, ,
    \end{equation*} 
    where $\Psi = \nicefrac{2 \varepsilon^2 n^2}{ \left( \sum\limits_{\ell=1}^{n-3} r_{i,\ell}^+ + r_{i,\ell+1}^- \right)^2}$.
\end{theorem}

\noindent
This bound is of order $\mathcal{O}(n e^{-\frac{T-n \log n}{n^2 (\log n)^2}})$ showing a slightly worse dependency on $n$ compared to \cref{the:SVARMHoeffding} due to the introduction of strata.
\section{Empirical Results} \label{sec:EmpiricalResults}

\begin{figure*}[t]
\centering
\includegraphics[width=0.99\textwidth]{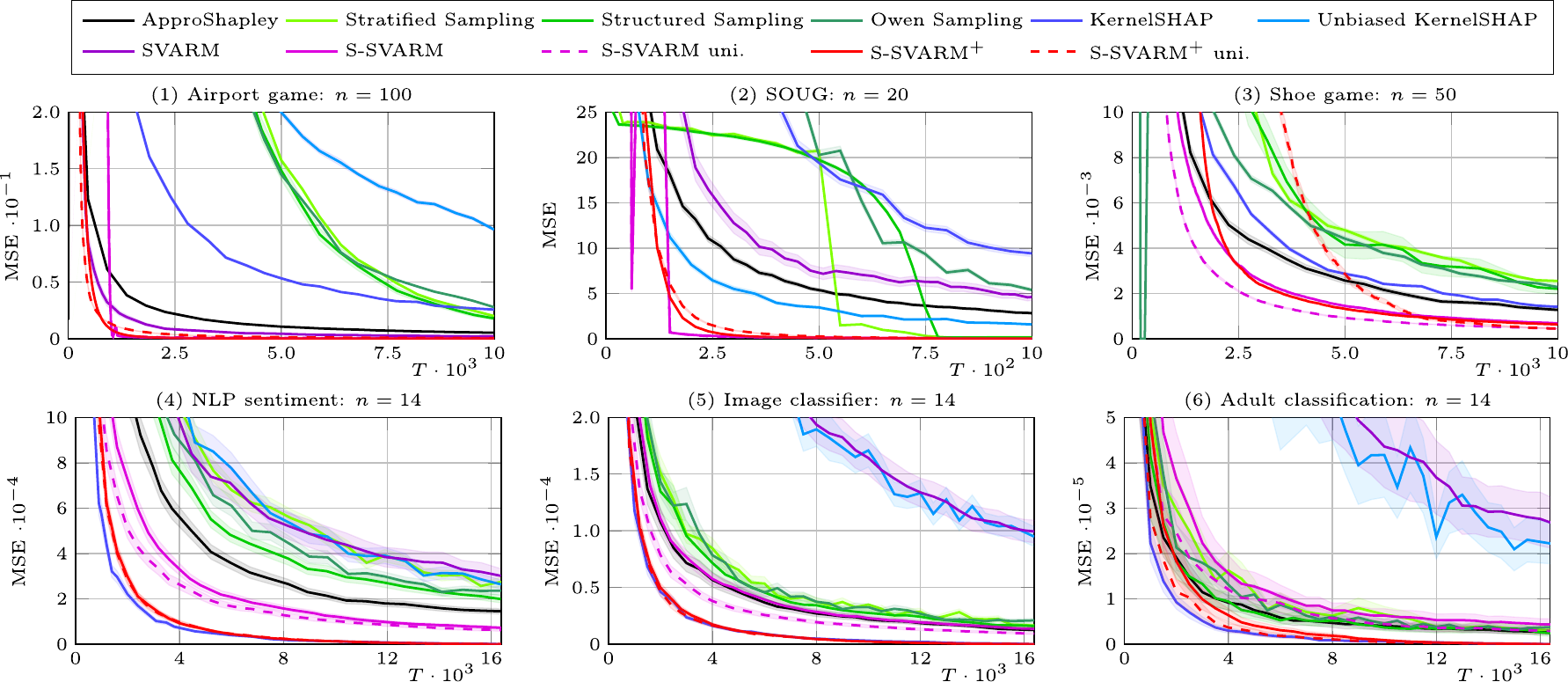}
\caption{Averaged MSE and standard errors over 100 repetitions in dependence of fixed budget $T$: (1) Airport game, (2) Shoe game, (3) SOUG game, (4) NLP sentiment analysis, (5) Image classifier, (6) Adult classification.}
\label{fig:mainResults}
\end{figure*}

To complement our theoretical findings, we evaluate our algorithms and its competitors on commonly considered synthetic cooperative games and explainable AI scenarios in which Shapley values need to be approximated. In particular, we select parameterless algorithms that do not rely on provided knowledge about the value function of the problem instance at hand, since ours do not either.
Besides the sampling distribution $\tilde{P}$ over coalition sizes proposed for Stratified SVARM (S-SVARM),  we also consider sampling with the simpler uniform distribution over all sizes from $2$ to $n-2$ (S-SVARM uniform).
In order to allow for a fair comparison with KernelSHAP, which samples coalitions without replacement, we include with S-SVARM$^+$ (uniform) an empirical version of S-SVARM without the warm-up that also samples without replacement to compensate for this underlying advantage (see \cref{alg:StartifiedSVARM+} in \cref{app:Pseudocode}), which obviously comes at the price of space complexity linear in~$T$.

We run the algorithms multiple times on the selected game types and measure their performances by the mean squared error (MSE) averaged over all players and runs depending on a range of fixed budget values $T$.
Measuring the approximation quality by the MSE requires the true Shapley values of the considered games to be available.
These are either given by a polynomial closed-form solution for the synthetic games (see \cref{sec:Synthetic})
or we compute them exhaustively for our explanation tasks (see \cref{sec:Explainability}).
The results of our evaluation are shown in \cref{fig:mainResults} and are presented in more detail in \cref{app:EmpiricalResults}.

As already said, we judge the algorithms' approximation qualities in dependence on the spent budget (model evaluations) $T$ instead of the consumed runtime.
In fact, the algorithms differ in actual runtime.
For example SVARM performs less arithmetic operations than Stratified SVARM since it does not update all players' estimates $\hat\phi_i^+$ or $\hat\phi_i^-$ with each sample.
Some algorithms, e.g.\ KernelSHAP, vary strongly in their time consumption per sample since a costly quadratic optimization problem needs to be solved after observing all samples.
We intentionally avoid the runtime comparison for three reasons:
(i) the observed runtimes may differ depending on the actual implementation,
(ii) the fixed-budget setting facilitates a coherent theoretical analysis where the observed information is restricted,
(iii) evaluating the worth of a coalition poses the bottleneck in explanation tasks, rendering the difference in performed arithmetic operations negligible.

\subsection{Synthetic games} \label{sec:Synthetic}
Cooperative games with polynomial closed-form solutions of their Shapley values are well suited for tracking the approximation error for large player numbers.
We exploit this fact and investigate a broad range of player numbers $n$ which are significantly higher than those for the explanation tasks.
We conduct experiments on the predefined Shoe and Airport game as done in \cite{Castro.2009, Castro.2017}.
Their degree of non-additivity poses a difficult challenge to all approximation algorithms.
Further, we consider randomly generated Sum of \mbox{Unanimity} Games (SOUG) games \cite{vanCampen.2018} which are capable of representing any cooperative game.
The value function and Shapley values of each game are given in \cref{app:CooperativeGames}.

We observe that S-SVARM itself already shows reliably good approximation performance across all considered games and budget ranges.
It is significantly superior to its competitors ApproShapley and KernelSHAP and as expected, S-SVARM$^+$ extends the lead in approximation quality even more.
In contrast, SVARM can rarely keep up with its refined counterpart S-SVARM.
However, in light of the bounds on the MSEs in \cref{cor:SVARMSE} and \ref{cor:StratSVARMSE} this is not surprising:  
SVARM's MSE bound scales linearly with the variances ${\sigma_i^+}^2$ and ${\sigma_i^-}^2$ of \emph{all} coalition values containing respectively not containing $i,$ while the relevant variance terms ${\sigma_{i,\ell}^+}^2$ and ${\sigma_{i,\ell}^-}^2$ for S-SVARM are restricted to coalitions of fixed size.
In most games, the latter terms are significantly lower since coalitions of the same size are plausibly closer in worth.
Finally, S-SVARM is quite robust regarding the magnitude of the standard errors.

\subsection{Explainabality games} \label{sec:Explainability}
We further conduct experiments on cooperative games stemming from real-world explainability scenarios, in particular, use cases in which local feature importance of machine learning models are to be quantified via Shapley values.
The NLP sentiment analysis game is based on the DistilBERT \cite{Sanh.2019} model architecture and consists of randomly selected movie reviews from the IMDB dataset \cite{Maas.2011} containing 14 words.
Missing features are masked in the tokenized representation and the value of a set is its sentiment score.
In the image classifier game, we explain the output of a ResNet18 \cite{resnet18} trained on ImageNet \cite{ImageNet}.
The images' pixels are summarized into $n=14$ super-pixels and absent features are masked with mean imputation.
The worth of a coalition is the returned class probability of the model (using only the present super-pixels) for the class of the original prediction which was made with all pixels being present.
For the adult classification game, we train a gradient-boosted tree model on the adult dataset \cite{Becker.1996}.
A coalition's worth is the predicted class probability of the true income class (income above or below $50\,000)$ of the given datapoint with the absent features being removed via mean imputation.
Since no polynomial closed-form solution exists for the Shapley values in these games, we compute them exhaustively, limiting us to a feasible number of players for which we can track the MSE.
While this restricts us to a player number (tokens, superpixels, features) of $n=14$ due to limited computational resources, this is arguably still an appropriate and commonly appearing number of entities involved in an explanation task.
We refer to \cref{app:CooperativeGames} for a more detailed explanation of the chosen~games.

A first observation is the close head-to-head race between S-SVARM$^+$ and KernelSHAP across the considered games leaving all other methods behind.
Thus, S-SVARM$^+$ is the first sample-mean-based approach achieving rivaling state-of-the-art approximation quality.
KernelSHAP's counterpart Unbiased KernelSHAP, designed to facilitate approximation guarantees similar to our theoretical results which KernelSHAP lacks, is clearly outperformed by S-SVARM.
Given the consistency demonstrated by S-SVARM and S-SVARM$^+$, we claim that both constitute a reliable choice under absence of domain knowledge.
We conjecture that~the reason for the slight performance decrease of S-SVARM from synthetic to explainability games lies not only within the latent structure of $\nu$, but is also caused by the lower player numbers.
As our theoretical results indicate, its sample efficiency grows with $n$ due to its enhanced update rule.
However, conducting experiments with larger $n$ becomes computationally prohibitive for explainability games,
since the Shapley values have to be calculated exhaustively in order to track the approximation error.
Further, our results indicate the robustness of S-SVARM$(^+)$ w.r.t.\ the utilized distribution $\tilde{P}$, which allows us to use the uniform distribution without performance loss, and secondly shows that our derived distribution is not just a theoretical artifact, but a valid contribution to express simpler bounds which are easier to grasp and interpret.
\section{Conclusion} \label{sec:Conclusion}

We considered the problem of precisely
approximating the Shapley value of all players in a cooperative game \mbox{under} the restriction that the value function can be evaluated only~a given number of times.
We presented a reformulation of the Shapley value, detached from the ubiquitous notion of marginal contribution, facilitating the approximation by estimates of which a multitude can be updated with each \mbox{access} to the value function.
On this basis, we proposed two approximation algorithms, SVARM and Stratified SVARM, which have a number of desirable properties.
Both are parameter-free, incremental, domain-independent, unbiased, and do not require any prior knowledge of the value function.
Further, Stratified SVARM shows a satisfying compromise between peak approximation quality and consistency across all considered games, paired with unmatched theoretical guarantees regarding its approximation quality.
While fulfilling more desirable properties and not having to solve a quadratic optimization problem of size $T$ in comparison to the state-of-the-art method KernelSHAP, effectively disabling on-the-fly approximations, our simpler sample-mean-based method Stratified SVARM$^+$ can fully keep up in common explainable AI scenarios, and even shows empirical superiority on synthetic games.

\paragraph{Limitations and future work.}
The quadratically growing number of strata w.r.t.\ $n$ might pose a challenge for higher player numbers, which future work could remedy by applying a coarser stratification that assigns multiple coalition sizes to a single stratum.
One could investigate the empirical behavior in further popular explanation domains such as data valuation, federated learning, or neuron importance and extend our evaluation to scenarios with higher player numbers.
Since the true Shapley values are not accessible for larger $n$, a different measure of approximation quality than the MSE needs to be taken for reference.
The convergence speed of the estimates is a naturally arising alternative.
Our empirical results give further evidence for the non-existence of a universally best approximation algorithm and encourage future research into the cause of the observed differences in performance w.r.t.\ the game type.
Further, it would be interesting to analyze whether structural properties of the value function, such as monotonicity or submodularity, have an impact on the approximation quality of both algorithms.

\section*{Acknowledgments}

This research was supported by the research training group Dataninja (Trustworthy AI for Seamless Problem Solving: Next Generation Intelligence Joins Robust Data Analysis) funded by the German federal state of North Rhine-Westphalia.
We gratefully acknowledge funding by the Deutsche Forschungsgemeinschaft (DFG, German Research Foundation): TRR 318/1 2021 – 438445824.
We would like to thank Fabian Fumagalli and especially Patrick Becker for their efforts in supporting our implementation.

\bibliography{references}
\clearpage

\begin{appendix}

\section{List of Symbols} \label{app:symbols}

\begin{table}[H]
\centering \caption{List of frequently symbols used throughout the paper.}
	\begin{tabular}{l|l}	
		\hline
		\multicolumn{2}{c}{\textbf{Problem setting}} \\
		\hline
        $\mathcal{N}$ & set of players \\
        $\mathcal{N}_i$ & set of players without $i$ \\
        $n$ & number of players \\
        $\nu$ & value function \\
        $T$ & budget, number of allowed evaluations of $\nu$ \\
        $\phi_i$ & Shapley value of player $i$ \\
        $\hat\phi_i$ & estimated Shapley value of player $i$ \\
		\hline
		\multicolumn{2}{c}{\textbf{SVARM}} \\
		\hline
		$\phi_i^+$ & positive Shapley value \\
        $\phi_i-$ & negative Shapley value \\
        $\hat\phi_i^+$ & estimated positive Shapley value \\
        $\hat\phi_i^-$ & estimated negative Shapley value \\
        $P^+$ & sampling probability distribution over coalitions to estimate $\phi_i^+$ \\
        $P^-$ & sampling probability distribution over coalitions to estimate $\phi_i^-$ \\
        $\bar{T}$ & remaining budget after completion of the warm-up phase \\
        ${\sigma_i^+}^2$ & variance of coalition values including player $i$ \\
        ${\sigma_i^-}^2$ & variance of coalition values excluding player $i$ \\
        $r_i^+$ & range of coalition values including player $i$ \\
        $r_i^-$ & range of coalition values excluding player $i$ \\
        $\bar{m}_i^+$ & number of sampled coalitions after the warm-up phase to update $\hat\phi_i^+$ \\
        $\bar{m}_i^-$ & number of sampled coalitions after the warm-up phase to update $\hat\phi_i^-$ \\
        $m_i^+$ & total number of sampled coalitions to update $\hat\phi_i^+$\\
        $m_i^-$ & total number of sampled coalitions to update $\hat\phi_i^-$ \\
        \hline
		\multicolumn{2}{c}{\textbf{Stratified SVARM}} \\
        \hline
        $\phi_{i,\ell}^+$ & $\ell$-th positive Shapley subvalue \\
        $\phi_{i,\ell}^-$ & $\ell$-th negative Shapley subvalue \\
        $\hat\phi_{i,\ell}^+$ & estimated $\ell$-th positive Shapley subvalue \\
        $\hat\phi_{i,\ell}^-$ & estimated $\ell$-th positive Shapley subvalue \\
        $\tilde{P}$ & sampling probability distribution over coalition sizes \\
        $\bar{T}$ & remaining budget after completion of the warm-up phase \\
        ${\sigma_{i,\ell}^+}^2$ & variance of coalition values in the $\ell$-th stratum including player $i$ \\
        ${\sigma_{i,\ell}^+}^2$ & variance of coalition values in the $\ell$-th stratum excluding player $i$ \\
        $r_{i,\ell}^+$ & range of coalition values in the $\ell$-th stratum including player $i$ \\
        $r_{i,\ell}^-$ & range of coalition values in the $\ell$-th stratum excluding player $i$ \\
        $\bar{m}_{i,\ell}^+$ & number of sampled coalitions after the warm-up phase to update $\hat\phi_{i,\ell}^+$ \\
        $\bar{m}_{i,\ell}^-$ & number of sampled coalitions after the warm-up phase to update $\hat\phi_{i,\ell}^-$ \\
        $m_{i,\ell}^+$ & total number of sampled coalitions to update $\hat\phi_{i,\ell}^+$ \\
        $m_{i,\ell}^-$ & total number of sampled coalitions to update $\hat\phi_{i,\ell}^-$ \\
        \hline
	\end{tabular}
 \label{tab:Symbols}
\end{table}
\section{Further Pseudocode} \label{app:Pseudocode}

\subsection{SVARM}

\begin{algorithm}[H]
  \caption{\textsc{\texttt{WarmUp}}}
  \label{alg:SVARMWarmup}
\begin{algorithmic}[1]
    \FOR{$i \in \mathcal{N}$}
        \STATE Draw $A^+$ and $A^-$ i.i.d.\ from $P^w$
        \STATE $\hat\phi_i^+ \leftarrow \nu(A^+ \cup \{i\})$
        \STATE $\hat\phi_i^- \leftarrow \nu(A^-)$
    \ENDFOR
\end{algorithmic}
\end{algorithm}

\noindent
The warm-up of SVARM samples for each player $i$ two coalitions $A^+$ and $A^-$, both drawn i.i.d.\ according to the weights $w_S$, i.e., the probability distribution $P^w$, and updates $\hat\phi_i^+$ and $\hat\phi_i^-,$ which needs a budget of $2n$ in total.
This ensures that each estimate is based on at least one sample.

\subsection{Stratified SVARM}

\begin{algorithm}[H]
    \caption{\textsc{\texttt{Update}}$(A)$}
    \label{alg:Update}
\begin{algorithmic}[1]
    \STATE $v \leftarrow \nu(A)$
    \FOR{$i \in A$}
       \STATE $\hat{\phi}_{i,|A|-1}^+ \leftarrow \frac{c_{i,|A|-1}^+ \cdot \hat{\phi}_{i,|A|-1}^+ + v}{c_{i,|A|-1}^+ + 1}$
       \STATE $c_{i,|A|-1}^+ \leftarrow c_{i,|A|-1}^+ + 1$
    \ENDFOR
    \FOR{$i \in \mathcal{N} \setminus A$}
        \STATE $\hat{\phi}_{i,|A|}^- \leftarrow \frac{c_{i,|A|}^- \cdot \hat{\phi}_{i,|A|}^- + v}{c_{i,|A|}^- + 1}$
        \STATE $c_{i,|A|}^- \leftarrow c_{i,|A|}^- + 1$
    \ENDFOR
\end{algorithmic}
\end{algorithm}

\noindent
Stratified SVARM's update procedure updates exactly one Shapley subvalue of each player given a coalition $A$.
It consumes only one budget token by storing the worth of $A$ in the variable $v$ (line 1).
The first loop increments for all players $i \in A$ their counter $c_{i,|A|-1}^+$ by 1 and updates the $|A|-1$-th positive Shapley subvalue estimate $\hat\phi_{i,|A|-1}^+$ to be the average over all values of coalitions which are contained in that stratum of player $i$.
Analogously, the second loop incerements for all players $i$ not contained in $A$ their counter $c_{i,|A|}^-$ by 1 and updates the $|A|$-th negative Shapley subvalue estimate $\hat\phi_{i,|A|}^+$ to be the average over all values of coalitions which are contained in that stratum of player $i$.

\begin{algorithm}[H]
  \caption{\textsc{\texttt{ExactCalculation}}$(\mathcal{N})$}
  \label{alg:ExactCalculation}
\begin{algorithmic}[1]
    \FOR{$s \in \{1,n-1,n\}$}
        \FOR{$A \in \{S \subseteq \mathcal{N} \mid |S| = s\}$}
            \STATE \textsc{\texttt{Update}}$(A)$
        \ENDFOR
    \ENDFOR
\end{algorithmic}
\end{algorithm}

\noindent
The exact calculation evaluates all coalitions of size $1,n-1$ and $n$, thus $2n+1$ in total.
For each coalition, the update procedure is called.
Effectively, this leads to exactly computed strata $\hat\phi_{i,0}^+ = \phi_{i,0}^+, \hat\phi_{i,n-2}^+ = \phi_{i,n-2}^+, \hat\phi_{i,n-1}^+ = \phi_{i,n-1}^+, \hat\phi_{i,1}^- = \phi_{i,1}^-, \hat\phi_{i,n-1}^- = \phi_{i,n-1}^-$ and counters $c_{i,0}^+ = 1, c_{i,n-2}^+ = n-1, c_{i,n-1}^+ = 1, c_{i,1}^- = n-1, c_{i,n-1}^- = 1$.

\begin{algorithm}[H]
	\caption{\textsc{\texttt{WarmUp}}$^+(\mathcal{N})$}
	\label{alg:StrataWarmupPos}
	\begin{algorithmic}[1]
		\FOR{$s = 2, \ldots, n-2$}
		      \STATE Draw a permutation $\pi$ of $\mathcal{N}$ u.a.r.
		      \FOR{$k = 0,\ldots, \lfloor \frac{n}{s} \rfloor - 1$}
		          \STATE $A \leftarrow \{\pi(1 + ks), \ldots, \pi(s+ks)\}$
                \STATE $v \leftarrow \nu(A)$
                \FOR{$i \in A$}
                    \STATE $\hat\phi_{i,s-1}^+ \leftarrow v$
                    \STATE $c_{i,s-1}^+ \leftarrow 1$
                \ENDFOR
            \ENDFOR           
		      \IF{$n \mod s \neq 0$}
		          \STATE $A \leftarrow \{\pi(n- (n \mod s) + 1),\ldots,\pi(n)\}$
		          \STATE Draw $B \in \{S \subseteq \mathcal{N} \setminus A \mid |S| = s - (n \mod s)\}$ 
                \STATE $v \leftarrow \nu(A \cup B)$
                \FOR{$i \in A$}
                    \STATE $\hat\phi_{i,s-1}^+ \leftarrow v$
                    \STATE $c_{i,s-1}^+ \leftarrow 1$
                \ENDFOR
		      \ENDIF
		\ENDFOR
	\end{algorithmic}
\end{algorithm}

\noindent
The warm-up for the positive Shapley subvalues iterates over all coalition sizes from $2$ to $n-2$ (line 1) and draws for each size $s$ a permutation $\pi$ of $\mathcal{N}$ uniformly at random (line 2).
The ordering $\pi$ is sliced into coalitions of size $s$ and each of them is used to update only the players contained in that particular coalition (lines 6--9).
In particular, since each coalition $A$ is the first to be observed for the corresponding players' stratum, the estimate $\hat\phi_{i,s-1}^+$ is set to the worth of $A$ and its counter $c_{i,s-1}^+$ is set to 1.
Note that for each coalition $A$ of size $s$ only one access to $\nu$ is made to update all affected $s$ many players.
In case that $n$ is not a multiple of $s$, some players less than $s$ are left over at the end of $\pi$ (line 11).
We group those with other random players to form a coalition of size $s$, but only update with the worth of that coalition the left out players (lines 15--18).
Note that the warm-up comes without any bias, since for each player $i$ and stratum estimate $\hat\phi_{i,s}$ each coalition $A \subseteq \mathcal{N}$ with $i \in A$ and $|A|=s$ has the same probability of being chosen.
Finally, $\lceil \frac{n}{s} \rceil$ many coalitions are evaluated for each size $s$, resulting in $\sum\nolimits_{s=2}^{n-2} \lceil \frac{n}{s} \rceil \in \mathcal{O}(n \log n)$ total evaluations.

\begin{algorithm}[H]
	\caption{\textsc{\texttt{WarmUp}}$^-(\mathcal{N})$}
	\label{alg:StrataWarmupNeg}
	\begin{algorithmic}[1]
		\FOR{$s = 2, \ldots, n-2$}
		      \STATE Draw a permutation $\pi$ of $\mathcal{N}$ u.a.r.
		      \FOR{$k = 0,\ldots, \lfloor \frac{n}{s} \rfloor - 1$}
		          \STATE $A \leftarrow \{\pi(1 + ks), \ldots, \pi(s+ks)\}$
                \STATE $v \leftarrow \nu(\mathcal{N} \setminus A)$
                \FOR{$i \in A$}
                    \STATE $\hat\phi_{i,n-s}^- \leftarrow v$
                     \STATE $c_{i,n-s}^- \leftarrow 1$
                \ENDFOR
            \ENDFOR           
		      \IF{$n \mod s \neq 0$}
		          \STATE $A \leftarrow \{\pi(n- (n \mod s) + 1),\ldots,\pi(n)\}$
		          \STATE Draw $B \in \{S \subseteq \mathcal{N} \setminus A \mid |S| = s - (n \mod s)\}$ u.a.r.
                \STATE $v \leftarrow \nu(\mathcal{N} \setminus (A \cup B))$
		        \FOR{$i \in A$}
                    \STATE $\hat\phi_{i,n-s}^- \leftarrow v$
                    \STATE $c_{i,n-s}^- \leftarrow 1$
                \ENDFOR
		      \ENDIF
		\ENDFOR
	\end{algorithmic}
\end{algorithm}

\noindent
The warm-up for the negative subvalues proceeds analogously to the previously presented positive warm-up.
Instead of $\hat\phi_{i,s-1}^+$ and $c_{i,s-1}^+$, $\hat\phi_{i,n-s}^-$ and $c_{i,n-s}^-$ are updated with the wort of $\mathcal{N} \setminus A$ for all players contained in the coalition $A$. 

\subsection{Stratified SVARM$^+$}

\begin{algorithm}[H]
  \caption{Stratified SVARM$^+$}
  \label{alg:StartifiedSVARM+}
\textbf{Input}: $\mathcal{N}$, $T \in \mathbb{N}$
\begin{algorithmic}[1]
    \STATE $\hat{\phi}_{i,\ell}^+, \hat{\phi}_{i,\ell}^- \leftarrow 0$ for all $i \in \mathcal{N}$ and $\ell \in \mathcal{L}$
    \STATE $c_{i,\ell}^+, c_{i,\ell}^- \leftarrow 0$ for all $i \in \mathcal{N}$ and $\ell \in \mathcal{L}$
    \STATE \textsc{\texttt{ExactCalculation}}$(\mathcal{N})$
    \STATE $t \leftarrow 2n + 1$ \textcolor{gray}{\COMMENT{consumed budget}}
    \STATE $w_s \leftarrow \tilde{P}(s)$ for all $s \in \{2,\ldots,n-2\}$ \textcolor{gray}{\COMMENT{sampling weight of size $s$}}
    \STATE $L_s \leftarrow \emptyset$ for all $s \in \{2,\ldots,n-2\}$ \textcolor{gray}{\COMMENT{sampled coalitions of size $s$}}
    \STATE $m_s \leftarrow \binom{n}{s}$ for all $s \in \{2,\ldots,n-2\}$ \textcolor{gray}{\COMMENT{reamining sets to sample of size $s$}}
    \WHILE{$t < T$ and $m_s > 0$ for some $s$}
        \STATE Draw $s_t \in \{2,\ldots,n-2\}$ with probability $\frac{w_s m_s}{\sum\limits_{s'=2}^{n-2} w_{s'} m_{s'}}$
        \STATE Draw $A_t$ from $\{S \subseteq \mathcal{N} \mid |S| = s_t, S \notin L_{s_t} \}$ u.a.r.
        \STATE \textsc{\texttt{Update}}$(A_t)$
        \STATE $m_{s_t} \leftarrow m_{s_t} - 1$
        \STATE $t \leftarrow t + 1$
    \ENDWHILE
    \STATE $\hat{\phi}_i \leftarrow \frac{1}{|\{c_{i,\ell}^+ \mid \ell \in \{0,\ldots,n-1\}\}|} \sum\limits_{\ell=0}^{n-1} \hat{\phi}_{i,\ell}^+ - \frac{1}{|\{c_{i,\ell}^- \mid \ell \in \{0,\ldots,n-1\}\}|} \sum\limits_{\ell=0}^{n-1} \hat{\phi}_{i,\ell}^-$ for all $i \in \mathcal{N}$
\end{algorithmic}
\textbf{Output}: $\hat\phi_1, \ldots, \hat\phi_n$
\end{algorithm}

\noindent
Stratified SVARM$^+$ is a modification of Stratified SVARM to deliver better empirical performance with only two slight changes that do not alter the method on a conceptual level.
First, we remove the warm-up since it is less efficient in the sense that not all players estimates are updated with each sampled coalition.
Hence, we only consume a budget of $2n+1$ due to the exact calculation of the border strata before entering the main loop.
Although it is extremely unlikely for a sufficiently large chosen budget $T$ and an appropriate distribution $\tilde{P}$ over the coalition sizes, it can happen that some $c_{i,\ell}^+$ or $c_{i,\ell}^-$ are zero.
In this case dividing by $n$ the total number of strata per sign per player in line 15 would cause an unnecessary bias.
Instead, we average only over all strata for which at least one sample has been observed, i.e.\ $c_{i,\ell}^+ > 0$ respectively $c_{i,\ell}^- > 0$.
Second and most important, we sample coalitions without replacement.
We are aware that different ways of implementing this exist (saving substantial amounts of runtime), but we choose to demonstrate it as simply as possible.
Effectively each coalition of size $s \in \{2,\ldots,n-2\}$ is assigned the weight $\frac{\tilde{P}(s)}{\binom{n}{s}}$, such that coalitions of the same size have the same weight and their weight sums up to $\tilde{P}(s)$.
In each time step $t$ a remaining coalition $A_t$ is drawn with probability proportional to its weight (its own weight divided by the sum of all remaining coalitions' weights).
We realize this by a two-step procedure: first the size $s_t$ is drawn in line 9, then a remaining coalition of size $s_t$ is drawn uniformly at random in line 10.
For this purpose we keep track of all so far sampled coalitions of a given size $s$ in $L_s$ (line 6) and the number of coalitions $m_s$ of size $s$ left to sample (line 7).
Finally, we added the condition that at least one coalition must be left to sample to the loop in line 8, in case that $T$ is chosen larger than $2^n - 1$.
\newpage

\section{SVARM Analysis} \label{app:SVARMAnalysis}

\textbf{Notation:}

\begin{itemize}
    \item Let $\bar{T} = T - 2n$.
    \item Let $A_i$ be a random set with $\mathbb{P}(A_i = S) = w_S$ for all $S \subseteq \mathcal{N}_i$ .
    \item Let ${\sigma_i^+}^2 = \mathbb{V} [\nu(A_i \cup \{i\})]$.
    \item Let ${\sigma_i^-}^2 = \mathbb{V} [\nu(A_i)]$.
    \item Let $\bar{m}_i^+$ be number of sampled coalitions $A^+$ after the warm-up phase that contain $i$.
    \item Let $\bar{m}_i^-$ be number of sampled coalitions $A^-$ after the warm-up phase that do not contain $i$.
    \item Let $m_i^+ = \bar{m}_i^+ + 1$ be total number of samples for $\hat\phi_i^+$.
    \item Let $m_i^- = \bar{m}_i^+ + 1$ be total number of samples for $\hat\phi_i^-$.
    \item Let $r_i^+ = \max\limits_{S \subseteq \mathcal{N}_i} \nu(S \cup \{i\}) - \min\limits_{S \subseteq \mathcal{N}_i} \nu(S \cup \{i\})$ be the range of $\nu(A_i \cup \{i\})$.
    \item Let $r_i^- = \max\limits_{S \subseteq \mathcal{N}_i} \nu(S) - \min\limits_{S \subseteq \mathcal{N}_i} \nu(S)$ be the range of $\nu(A_i)$.
\end{itemize}

\noindent
\textbf{Assumptions:}
\begin{itemize}
    \item $\bar{T}$ is even
    \item $\bar{T} > 0$
\end{itemize}

\subsection{Unbiasedness of Shapley Value Estimates} \label{subsec:unbiasednessSVARM}

To start with, we prove that the distributions $P^+$ and $P^-$ are well-defined.
\begin{lemma} \label{lem:SVARMDistributions}
    The distributions $P^+$ and $P^-$ over $\mathcal{P}(\mathcal{N})$ are well-defined, i.e.,
    \begin{equation*}
        \sum\limits_{S \subseteq \mathcal{N}} P^+(S) = \sum\limits_{S \subseteq \mathcal{N}} P^-(S) = 1 .
    \end{equation*}
\end{lemma}

\begin{proof}
    The statement is easily shown for $P^+$ by grouping the coalitions by size.
    We derive:
    \begin{align*}
        & \ \sum\limits_{S \subseteq \mathcal{N}} P^+(S) \\
        = & \ \sum\limits_{\ell=1}^n \sum\limits_{\substack{ S \subseteq \mathcal{N} \\ |S| = \ell }} P^+(S) \\
        = & \ \sum\limits_{\ell=1}^n \sum\limits_{\substack{ S \subseteq \mathcal{N} \\ |S| = \ell }} \frac{1}{\ell \binom{n}{\ell} H_n} \\
        = & \ \frac{1}{H_n} \sum\limits_{\ell=1}^n \frac{1}{\ell} \\
        = & \ 1 .
    \end{align*}
    One can prove the desired property analogously for $P^-$.
\end{proof}

\noindent
For the remainder of this section we assume that $T \geq 2n + 1$ such that the warm-up phase can be completed by SVARM.
\begin{lemma} \label{lem:SVARMUnbiasedness}
    For each player $i \in \mathcal{N}$ the positive and negative Shapley Value estimates $\hat\phi_i^+$ and $\hat\phi_i^-$ are unbiased, i.e.,
    \begin{equation*}
         \mathbb{E} \left[ \hat\phi_i^+ \right] = \phi_i^+ \hspace{0.5cm} \text{and} \hspace{0.5cm} \mathbb{E} \left[ \hat\phi_i^- \right] = \phi_i^- .
    \end{equation*}
\end{lemma}

\begin{proof}
Let $\bar{m}_i^+$ be the number of coalitions sampled after the warm-up phase that contain $i$ and $m_i^+$ be the total number of samples used to update $\hat\phi_i^+$, thus $m_i^+ = \bar{m}_i^+ + 1$.
Further, let $A_{m}^+$ for $m \in \{1,3,5,\ldots,T-1\}$ be the sampled coalitions for updating the positive Shapley values $(\hat\phi_i^+)_{i\in \mathcal{N}},$ then we can write the positive Shapley value of player $i\in \mathcal{N}$ as
\begin{align}  \label{def:PositiveShapleySampleVersion}
    \begin{split}
        \hat\phi_i^+ 
        & = \frac{1}{m_i^+} \sum_{\tilde m=1}^{T/2} \nu(A_{2\tilde m-1}^+) \mathbb{I}_{ \{i \in A_{2\tilde m-1}^+ \} } \\
        & = \frac{1}{m_i^+} \left(  \nu(A_{2i-1}^+) +  \sum_{\tilde  m=n}^{T/2} \nu(A_{2\tilde m-1}^+) \mathbb{I}_{ \{i \in A_{2\tilde m-1}^+ \}  } \right),
    \end{split} 
\end{align}
where $\mathbb{I}_{\cdot}$ denotes the indicator function, and we used that during the warm-up phase ($m\leq 2n$) there is for each player $i$ only one $A_{m}^+$ to update the corresponding positive Shapley value, namely at time step $2i-1.$ 
First, we show for each odd $m \geq 2n$ and $S \subseteq \mathcal{N}_i$ that $\mathbb{P} (A_{m}^+ = S \cup \{i\} \mid i \in A_{m}^+) = w_S.$
Note that since $A_{m}^+ \sim P^+$ (see \eqref{eq:SVARMDistribution}) it holds that
\begin{align} \label{eq:AuxResult1}
	\begin{split}
		\mathbb{P}(i \in A_{m}^+) = & \ \sum\limits_{\ell=1}^n \mathbb{P}(i \in A_{m}^+, |A_{m}^+| = \ell) \\
		= & \ \sum\limits_{\ell=1}^n \mathbb{P}(i \in A_{m}^+ \mid |A_{m}^+| = \ell) \cdot \mathbb{P}(|A_{m}^+| = \ell) \\
		= & \ \sum\limits_{\ell=1}^n \frac{\ell}{n} \cdot \frac{1}{\ell \cdot H_{n-1}} \\
		= & \ \frac{1}{H_n} .
	\end{split}
\end{align}
With this, we derive 
\begin{align*}
    \mathbb{P} (A_{m}^+ = S \cup \{i\} \mid i \in A_{m}^+) 
    = & \ \frac{\mathbb{P}(A_{m}^+ = S \cup \{i\})}{\mathbb{P}(i \in A_{i}^+)} \\
    = & \  H_n \cdot P^+(S  \cup \{i\}) \\
    = & \ \frac{1}{(|S+1|) \cdot \binom{n}{|S+1|}} \\
    = & \ w_S.
\end{align*}
Since $A_{2i-1}^+\setminus\{i\} \sim P^w$ it holds that $  \mathbb{P}(A_{2i-1}^+\setminus\{i\} = S) = w_S $ for any $i \in \mathcal{N}$ and $S \subset \mathcal{N}_i.$ \\
Taking all of this into account, we derive for $\hat\phi_i^+$ using \eqref{def:PositiveShapleySampleVersion}:
\begin{align*}
    \mathbb{E} \left[ \hat\phi_i^+ \mid m_i^+ \right] 
    = & \ \frac{1}{m_i^+} \left( \mathbb{E} \left[ \nu(A_{2i-1}^+) \mid m_i^+ \right] 
    +    \mathbb{E} \left[ \sum_{\tilde m =n}^{T/2} \nu(A_{2\tilde m -1 }^+) \mathbb{I}_{ \{i \in A_{2\tilde m-1}^+ \}  } \mid m_i^+ \right] \right) \\
    = & \ \frac{1}{m_i^+} \Big( \sum\limits_{S \subseteq \mathcal{N}_i} \mathbb{P}(A_{2i-1}^+\setminus\{i\} = S) \cdot \nu(S \cup \{i\}) \\
    & \quad + \mathbb{E} \Big[  \sum_{\tilde m =n}^{T/2}  \mathbb{I}_{ \{i \in A_{2\tilde m-1}^+ \}}\sum\limits_{S \subset \mathcal{N}_i} \mathbb{P}( A_{2\tilde m-1}^+ = S \cup \{i\} \mid i \in A_{2\tilde m-1}^+) \cdot \nu(S \cup \{i\}) \mid m_i^+  \Big] \Big)  \\
    = & \ \frac{1}{m_i^+} \left( \sum\limits_{S \subseteq \mathcal{N}_i} w_S \cdot \nu(S \cup \{i\}) + \bar{m}_i^+ \sum\limits_{S \subseteq \mathcal{N}_i} w_S \cdot \nu(S \cup \{i\}) \right) \\
     = & \ \phi_i^+ .
\end{align*}
Finally, we conclude:
\begin{align*}
    \mathbb{E} \left[ \hat\phi_i^+ \right] = & \ \sum\limits_{m=1}^{\frac{\bar{T}}{2}} \mathbb{E} \left[ \hat\phi_i^+ \mid m_i^+ = m \right] \cdot \mathbb{P}(m_i^+ = m) \\
    = & \ \sum\limits_{m=1}^{\frac{\bar{T}}{2}} \phi_i^+ \cdot \mathbb{P}(m_i^+ = m) \\
    = & \ \phi_i^+ .
\end{align*}
Analogously we derive $ \mathbb{E} \left[ \hat\phi_i^- \right] = \phi_i^-$ by defining $\bar{m}_i^-$, $m_i^-$, and $A_{m}^-$ for $m \in \{2,4,6,\ldots,T\}$ similarly as for their positive counterparts.
\end{proof}

\noindent
\textbf{Theorem} \ref{the:SVARMUnbiased}
\textit{
    For each player $i \in \mathcal{N}$ the estimate $\hat\phi_i$ obtained by SVARM is unbiased, i.e.,
    \begin{equation*}
        \mathbb{E} \left[ \hat\phi_i \right] = \phi_i .
    \end{equation*}
}

\begin{proof}
We apply \cref{lem:SVARMUnbiasedness} and obtain in combination with \cref{eq:ShapleyTwoSums}:
\begin{align*}
   \mathbb{E} \left[ \hat\phi_i \right] = & \  \mathbb{E} \left[ \hat\phi_i^+ \right] - \mathbb{E} \left[ \hat\phi_i^- \right] \\
    = & \ \phi_i^+ - \phi_i^- \\
    = & \ \phi_i .
\end{align*}
\end{proof}

\subsection{Sample Numbers}

\begin{lemma} \label{lem:SVARMSampleNumbers}
    For any $i \in \mathcal{N}$ the expected number of updates of $\hat\phi_i^+$ and $\hat\phi_i^-$ after the warm-up phase is
    \begin{equation*}
        \mathbb{E} \left[ \bar{m}_i^+ \right] = \mathbb{E} \left[ \bar{m}_i^- \right] = \frac{\bar{T}}{2 H_n} .
    \end{equation*}
\end{lemma}

\begin{proof}
First, we observe that $\bar{m}_i^+$ is binomially distributed with $\bar{m}_i^+ \sim Bin \left( \frac{\bar{T}}{2}, \frac{1}{H_n} \right)$ because $\frac{\bar{T}}{2}$ many pairs are sampled and each independently sampled coalition $A^+$ contains the player $i$ with probability $H_n^{-1},$ see \eqref{eq:AuxResult1}.

\noindent
Consequently, we obtain
\begin{equation*}
    \mathbb{E} \left[ \bar{m}_i^+ \right] = \frac{\bar{T}}{2 H_n} . 
\end{equation*}

\noindent
Similarly, we observe that $m_i^-$ is also binomially distributed with $m_i^- \sim Bin \left( \frac{\bar{T}}{2}, \frac{1}{H_n} \right)$, leading to the same expected number of updates.
\end{proof}

\subsection{Variance and Squared Error}

\begin{lemma} \label{lem:SVARMVariance}
    The variance of any player's Shapley value estimate $\hat\phi_i$ given the number of samples $m_i^+$ and $m_i^-$ is exactly
    \begin{equation*}
        \mathbb{V} \left[ \hat\phi_i \mid m_i^+, m_i^- \right] = \frac{{\sigma_i^+}^2}{m_i^+} + \frac{{\sigma_i^-}^2}{m_i^-} .
    \end{equation*}
\end{lemma}

\begin{proof}
We first decompose the variance of $\hat\phi_i$ into the variances of $\hat\phi_i^+$ and $\hat\phi_i^-$ and their covariance:
\begin{equation*}
    \mathbb{V} \left[ \hat\phi_i \mid m_i^+, m_i^- \right] = \left( \mathbb{V} \left[ \hat\phi_i^+ \mid m_i^+ \right] + \mathbb{V} \left[ \hat\phi_i^- \mid m_i^- \right] - 2 \text{Cov} \left( \hat\phi_i^+, \hat\phi_i^- \mid m_i^+, m_i^- \right) \right)
\end{equation*}
We derive for the variance of $\hat\phi_i^+$:
\begin{align*}
    \mathbb{V} \left[ \hat\phi_i^+ \mid m_i^+ \right] = & \ \frac{1}{{m_i^+}^2} \sum\limits_{m=0}^{\bar{m}_i^+} \mathbb{V} \left[ \nu(A_{i,m}^+) \right] \\
    = & \ \frac{{\sigma_i^+}^2}{m_i^+} .
\end{align*}
Similarly we obtain for $\hat\phi_i^-$:
\begin{equation*}
    \mathbb{V} \left[ \hat\phi_i^- \mid m_i^- \right] = \frac{{\sigma_i^-}^2}{m_i^-} .
\end{equation*}
The covariance of $\hat\phi_i^+$ and $\hat\phi_i^-$ is zero because both are updated with sampled coalitions drawn independently of each other.
Thus, we conclude:
\begin{equation*}
    \mathbb{V} \left[ \hat\phi_i \mid m_i^+, m_i^- \right] = \frac{{\sigma_i^+}^2}{m_i^+} + \frac{{\sigma_i^-}^2}{m_i^-} .
\end{equation*}
\end{proof}

\begin{lemma} \label{lem:SVARMInvertedExpectation}
    For the sample numbers of any player $i \in \mathcal{N}$ holds
    \begin{equation*}
        \mathbb{E} \left[ \frac{1}{m_i^+} \right] = \mathbb{E} \left[ \frac{1}{m_i^-} \right] \leq \frac{2 H_n}{\bar{T}} .
    \end{equation*}
\end{lemma}

\begin{proof}
By combining Equation~(3.4) in \cite{Chao.1972}:
\begin{equation*}
    \mathbb{E} \left[ \frac{1}{1+X} \right] = \frac{1 - (1-p)^{m+1}}{(m+1)p} \leq \frac{1}{mp} = \frac{1}{\mathbb{E}[X]},
\end{equation*}
for any binomially distributed random variable $X \sim Bin(m,p)$ with \cref{lem:SVARMSampleNumbers}, we obtain:
\begin{equation*}
    \mathbb{E} \left[ \frac{1}{m_i^+} \right] = \mathbb{E} \left[ \frac{1}{1 + \bar{m}_i^+} \right] \leq \frac{1}{\mathbb{E} \left[ \bar{m}_i^+ \right]} = \frac{2 H_n}{\bar{T}}.
\end{equation*}
Notice that $m_i^+$ and $m_i^-$ are identically distributed.
\end{proof}

\noindent
\textbf{Theorem} \ref{the:SVARMVariance}
\textit{
     The variance of any player's Shapley value estimate $\hat\phi_i$ is bounded by
    \begin{equation*}
        \mathbb{V} \left[ \hat\phi_i \right] \leq \frac{2 H_n}{\bar{T}} \left({\sigma_i^+}^2 + {\sigma_i^-}^2 \right) .
    \end{equation*}
}

\begin{proof}
The combination of \cref{lem:SVARMVariance} and \cref{lem:SVARMInvertedExpectation} yields:
\begin{align*}
    \mathbb{V} \left[ \hat\phi_i \right] = & \ \mathbb{E} \left[ \mathbb{V} \left[ \hat\phi_i \mid m_i^+, m_i^- \right] \right] \\
    = & \ \mathbb{E} \left[ \frac{{\sigma_i^+}^2}{m_i^+} + \frac{{\sigma_i^-}^2}{m_i^-} \right] \\
    = & \ {\sigma_i^+}^2 \mathbb{E} \left[ \frac{1}{m_i^+} \right] + {\sigma_i^-}^2 \left[ \frac{1}{m_i^-} \right] \\
    \leq & \ \frac{2 H_n}{\bar{T}} \left({\sigma_i^+}^2 + {\sigma_i^-}^2 \right).
\end{align*}
\end{proof}

\noindent
\textbf{Corollary} \ref{cor:SVARMSE}
\textit{
    The expected squared error of any player's Shapley value estimate is bounded by
    \begin{equation*}
        \mathbb{E} \left[ \left( \hat\phi_i - \phi_i \right)^2 \right] \leq \frac{2 H_n}{\bar{T}} \left({\sigma_i^+}^2 + {\sigma_i^-}^2 \right) .
    \end{equation*}
}

\begin{proof}
The bias-variance decomposotion allows us to plug in the unbiasedness of $\hat\phi_i$ shown in \cref{the:SVARMUnbiased} and the bound on the variance from
\begin{align*}
    \mathbb{E} \left[ \left( \hat\phi_i - \phi_i \right)^2 \right] = & \ \left( \mathbb{E} \left[ \hat\phi_i \right] - \phi_i \right)^2 + \mathbb{V} \left[ \hat\phi_i \right] \\
    \leq & \ \frac{2 H_n}{\bar{T}} \left({\sigma_i^+}^2 + {\sigma_i^-}^2 \right) .
\end{align*}
\end{proof}

\subsection{Probabilistic Bounds}

\textbf{Theorem} \ref{the:SVARMCheby}
\textit{
     Fix any player $i \in \mathcal{N}$ and $\varepsilon > 0$. The probability that the Shapley value estimate $\hat\phi_i$ deviates from $\phi_i$ by a margin of $\varepsilon$ or greater is bounded by
    \begin{equation*}
        \mathbb{P} \left( | \hat\phi_i - \phi_i | \geq \varepsilon \right) \leq \frac{2 H_n}{\varepsilon^2 \bar{T}} \left( {\sigma_i^-}^2 + {\sigma_i^+}^2 \right) . 
    \end{equation*}
}

\begin{proof}
The bound on the variance of $\hat\phi_i$ in \cref{the:SVARMVariance} allows us to apply Chebyshev's inequality:
\begin{equation*}
    \mathbb{P} \left( | \hat\phi_i - \phi_i | \geq \varepsilon \right) \leq \frac{\mathbb{V} \left[ \hat\phi_i \right]}{\varepsilon^2} \leq \frac{2 H_n}{\varepsilon^2 \bar{T}} \left( {\sigma_i^-}^2 + {\sigma_i^+}^2 \right) .
\end{equation*}
\end{proof}

\begin{corollary} \label{cor:SVARMCheby}
    Fix any player $i \in \mathcal{N}$ and $\delta \in (0,1]$. The Shapley value estimate $\hat\phi_i$ deviates from $\phi_i$ by a margin of $\varepsilon$ or greater with probability not greater than $\delta$, i.e.,
    \begin{equation*}
        \mathbb{P} \left( | \hat\phi_i - \phi_i | \geq \varepsilon \right) \leq \delta \hspace{0.5cm} \text{for} \hspace{0.5cm} \varepsilon = \sqrt{\frac{2 H_n}{\delta \bar{T}} \left( {\sigma_i^+}^2 + {\sigma_i^-}^2 \right)} .
    \end{equation*}
\end{corollary}

\begin{lemma} \label{lem:SVARMHoeffdingGivenSamples}
    For any fixed player $i \in \mathcal{N}$ and $\varepsilon > 0$ holds
    \begin{equation*}
        \mathbb{P} \left( | \hat\phi_i^+ - \phi_i^+ | \geq \varepsilon \mid m_i^+ \right) \leq 2 \exp \left( - \frac{2 m_i^+  \varepsilon^2}{{r_i^+}^2} \right) \hspace{0.5cm} \text{and} \hspace{0.5cm} \mathbb{P} \left( | \hat\phi_i^- - \phi_i^- | \geq \varepsilon \mid m_i^- \right) \leq 2 \exp \left( - \frac{2 m_i^+  \varepsilon^2}{{r_i^+}^2} \right).
    \end{equation*}
\end{lemma}

\begin{proof}
We prove the statement for $\hat\phi_i$ by making use of Hoeffding's inequality in combination with the unbiasedness of the positive and negative Shapley value estimates shown in \cref{lem:SVARMUnbiasedness}.
The proof for $\hat\phi_i^-$ is analogous.
\begin{align*}
    \mathbb{P} \left( | \hat\phi_i^+ - \phi_i^+ | \geq \varepsilon \mid m_i^+ \right) = & \ \mathbb{P} \left( | \hat\phi_i^+ - \mathbb{E} \left[ \hat\phi_i+ \right] | \mid m_i^+ \right) \\
    = & \ \mathbb{P} \left( \Bigg| \sum\limits_{m=0}^{\bar{m}_i^+} \nu(A_{i,m}^+) - \mathbb{E} \left[ \sum\limits_{m=0}^{\bar{m}_i^+} \nu(A_{i,m}^+) \right] \Bigg| \geq m_i^+ \varepsilon \mid m_i^+ \right) \\
    \leq & \ 2 \exp \left( - \frac{2 m_i^+  \varepsilon^2}{{r_i^+}^2} \right) .
\end{align*}
\end{proof}

\begin{lemma} \label{lem:SVARMHoeffding}
    For any fixed player $i \in \mathcal{N}$ and $\varepsilon > 0$ holds:
    \begin{itemize}
        \item $\mathbb{P} \left( | \hat\phi_i^+ - \phi_i^+  | \geq \varepsilon \right) \leq \exp \left( - \frac{\bar{T}}{4 {H_n}^2} \right) + 2 \frac{\exp \left( - \frac{2 \varepsilon^2}{{r_i^+}^2} \right)^{\left\lfloor \frac{\bar{T}}{4 H_n} \right\rfloor}}{\exp \left( \frac{2 \varepsilon^2}{{r_i^+}^2} \right) - 1}$,
        \item $\mathbb{P} \left( | \hat\phi_i^- - \phi_i^-  | \geq \varepsilon \right) \leq \exp \left( - \frac{\bar{T}}{4 {H_n}^2} \right) + 2 \frac{\exp \left( - \frac{2 \varepsilon^2}{{r_i^-}^2} \right)^{\left\lfloor \frac{\bar{T}}{4 H_n} \right\rfloor}}{\exp \left( \frac{2 \varepsilon^2}{{r_i^-}^2} \right) - 1}$ .
    \end{itemize}
\end{lemma}

\begin{proof}
We prove the statement for $\hat\phi_i^+$.
The proof for $\hat\phi_i^-$ is analogous.
To begin with, we derive with the help of Hoeffding's inequality for binomial distributions a bound for the probability of $\bar{m}_i^+$ not exceeding $\frac{\bar{T}}{4 H_n}$:
\begin{align*}
    \mathbb{P} \left( \bar{m}_i^+ \leq \frac{\bar{T}}{4 H_n} \right) \leq & \ \mathbb{P} \left( \mathbb{E} \left[ \bar{m}_i^+ \right] - \bar{m}_i^+ \geq \mathbb{E} \left[ \bar{m}_i^+ \right] - \frac{\bar{T}}{4 H_n} \right) \\
    \leq & \ \exp \left( - \frac{4 \left( \mathbb{E} \left[ \bar{m}_i^+ \right] - \frac{\bar{T}}{4 H_n} \right)^2}{\bar{T}} \right) \\
    \leq & \ \exp \left( - \frac{\bar{T}}{4 {H_n}^2} \right),
\end{align*}
where we used the lower bound on $\mathbb{E} \left[ \bar{m}_i^+ \right]$ shown in \cref{lem:SVARMSampleNumbers}.
Next, we derive with the help of \cref{lem:SVARMHoeffdingGivenSamples} a statement of technical nature to be used later:
\begin{align*}
    & \ \sum\limits_{m= \left\lfloor \frac{\bar{T}}{4 H_n} \right\rfloor + 1}^{\frac{\bar{T}}{2}} \mathbb{P} \left( | \hat\phi_i^+ - \phi_i^+ | \geq \varepsilon \mid m_i^+ = m \right) \\
    \leq & \ 2 \sum\limits_{m= \left\lfloor \frac{\bar{T}}{4 H_n} \right\rfloor + 1}^{\frac{\bar{T}}{2}} \exp \left( - \frac{2 m \varepsilon^2}{{r_i^+}^2} \right) \\
    = & \ 2 \sum\limits_{m = 0}^{\frac{\bar{T}}{2}} \exp \left( - \frac{2 \varepsilon^2}{{r_i^+}^2} \right)^m - 2 \sum\limits_{m = 0}^{ \left\lfloor \frac{\bar{T}}{4 H_n} \right\rfloor} \exp \left( - \frac{2 \varepsilon^2}{{r_i^+}^2} \right)^m \\
    = & \ 2 \frac{\exp \left( - \frac{2 \varepsilon^2}{{r_i^+}^2} \right)^{\left\lfloor \frac{\bar{T}}{4 H_n} \right\rfloor} - \exp \left( - \frac{\varepsilon^2}{{r_i^+}^2} \right)^{\bar{T}} }{\exp \left( \frac{2 \varepsilon^2}{{r_i^+}^2} \right) - 1} \\
    \leq & \ 2 \frac{\exp \left( - \frac{2 \varepsilon^2}{{r_i^+}^2} \right)^{\left\lfloor \frac{\bar{T}}{4 H_n} \right\rfloor}}{\exp \left( \frac{2 \varepsilon^2}{{r_i^+}^2} \right) - 1} .
\end{align*}
At last, putting both findings together, we derive our claim:
\begin{align*}
    & \ \mathbb{P} \left( | \hat\phi_i^+ - \phi_i^+ | \geq \varepsilon \right) \\
    = & \ \sum\limits_{m=1}^{\frac{\bar{T}}{2}} \mathbb{P} \left( | \hat\phi_i^+ - \phi_i^+ | \geq \varepsilon \mid m_i^+ = m \right) \cdot \mathbb{P} \left( m_i^+ = m \right) \\
    = & \ \sum\limits_{m=1}^{\left\lfloor \frac{\bar{T}}{4 H_n} \right\rfloor} \mathbb{P} \left( | \hat\phi_i^+ - \phi_i^+ | \geq \varepsilon \mid m_i^+ = m \right) \cdot \mathbb{P} \left( m_i^+ = m \right) + \sum\limits_{m = \left\lfloor \frac{\bar{T}}{4 H_n} \right\rfloor + 1}^{\frac{\bar{T}}{2}} \mathbb{P} \left( | \hat\phi_i^+ - \phi_i^+ | \geq \varepsilon \mid m_i^+ = m \right) \cdot \mathbb{P} \left( m_i^+ = m \right) \\
    \leq & \ \mathbb{P} \left( \bar{m}_i^+ \leq \left\lfloor \frac{\bar{T}}{4 H_n} \right\rfloor \right) + \sum\limits_{m = \left\lfloor \frac{\bar{T}}{4 H_n} \right\rfloor + 1}^{\frac{\bar{T}}{2}} \mathbb{P} \left( | \hat\phi_i^+ - \phi_i^+ | \geq \varepsilon \mid m_i^+ = m \right) \\
    \leq & \ \exp \left( - \frac{\bar{T}}{4 {H_n}^2} \right) + 2 \frac{\exp \left( - \frac{2 \varepsilon^2}{{r_i^+}^2} \right)^{\left\lfloor \frac{\bar{T}}{4 H_n} \right\rfloor}}{\exp \left( \frac{2 \varepsilon^2}{{r_i^+}^2} \right) - 1} .
\end{align*}
\end{proof}

\noindent
\textbf{Theorem} \ref{the:SVARMHoeffding}
\textit{
    For any fixed player $i \in \mathcal{N}$ and $\varepsilon > 0$ the probability that the Shapley value estimate $\hat\phi_i$ deviates from $\phi_i$ by a margin of $\varepsilon$ or greater is bounded by
    \begin{equation*}
        \mathbb{P} \left( | \hat\phi_i - \phi_i  | \geq \varepsilon \right) \leq 2 \exp \left( - \frac{\bar{T}}{4 {H_n}^2} \right) + 4 \frac{\exp \left( - \frac{2 \varepsilon^2}{(r_i^+ + r_i^-)^2} \right)^{\left\lfloor \frac{\bar{T}}{4 H_n} \right\rfloor}}{\exp \left( \frac{2 \varepsilon^2}{(r_i^+ + r_i^-)^2} \right) - 1} .
    \end{equation*}
}

\begin{proof}
\begin{align*}
    & \ \mathbb{P} \left( | \hat\phi_i - \phi_i | \geq \varepsilon \right) \\
    = & \ \mathbb{P} \left( | ( \hat\phi_i^+ - \phi_i^+  ) + ( \phi_i^- - \hat\phi_i^- ) | \geq \varepsilon \right) \\
    \leq & \ \mathbb{P} \left( | \hat\phi_i^+ - \phi_i^+ | + | \hat\phi_i^- - \phi_i^- | \geq \varepsilon \right) \\
    \leq & \ \mathbb{P} \left( | \hat\phi_i^+ - \phi_i^+ | \geq \frac{\varepsilon r_i^+}{r_i^+ + r_i^-} \right) + \mathbb{P} \left( | \hat\phi_i^- - \phi_i^- | \geq \frac{\varepsilon r_i^-}{r_i^+ + r_i^-} \right) \\
    \leq & \ 2 \exp \left( - \frac{\bar{T}}{4 {H_n}^2} \right) + 4 \frac{\exp \left( - \frac{2 \varepsilon^2}{(r_i^+ + r_i^-)^2} \right)^{\left\lfloor \frac{\bar{T}}{4 H_n} \right\rfloor}}{\exp \left( \frac{2 \varepsilon^2}{(r_i^+ + r_i^-)^2} \right) - 1} .
\end{align*}
\end{proof}
\newpage

\section{Stratified SVARM Analysis} \label{app:StratSVARMAnalysis}

\textbf{Notation:}
\begin{itemize}
    \item Let $\mathcal{L} = \{0,\ldots,n-1\}$, $\mathcal{L}^+ = \{1,\ldots,n-3\}$, and $\mathcal{L}^- = \{2,\ldots,n-2\}$.
    \item Let $W = 2n + 1 + 2 \sum\limits_{s=2}^{n-2} \lceil \frac{n}{s} \rceil$ denote the length of the warm-up phase.
    \item Let $\bar{T} = T - W$ be the available steps after the warm-up phase.
    \item Let $m_{i,\ell}^+ = \# \{t \mid i \in A_t, |A_t| = \ell+1 \}$ be the total number of samples used to update $\hat{\phi}_{i,\ell}^+$.
    \item Let $m_{i,\ell}^- = \# \{t \mid i \notin A_t, |A_t| = \ell \}$ be the total number of samples used to update $\hat{\phi}_{i,\ell}^-$.
    \item Let $\bar{m}_{i,\ell}^+ = \# \{t \mid i \in A_t, |A_t| = \ell+1, t > W \}$ be the number of samples used to update $\hat{\phi}_{i,\ell}^+$ after the warm-up phase.
    \item Let $\bar{m}_{i,\ell}^- = \# \{t \mid i \notin A_t, |A_t| = \ell, t > W \}$ be the number of samples used to update $\hat{\phi}_{i,\ell}^-$ after the warm-up phase.
    \item Let $A_{i,\ell,k}^+$ be the $k$-th set used to update $\phi_{i,\ell}^+$ and $A_{i,\ell,k}^-$ the $k$-th set used to update $\phi_{i,\ell}^-$.
    \item Let $A_{i,\ell}$ be a random set with $\mathbb{P}(A_{i,\ell} = S) = \frac{1}{\binom{n-1}{\ell}}$ for all $S \subseteq \mathcal{N} \setminus \{i\}$ with $|S| = \ell$.
    \item Let $ \hat{\phi}_{i,\ell}^+ = \frac{1}{m_{i,\ell}^+} \sum\limits_{k=1}^{m_{i,\ell}^+} \nu(A_{i,\ell,k}^+) $ and  $\hat{\phi}_{i,\ell}^- = \frac{1}{m_{i,\ell}^-} \sum\limits_{k=1}^{m_{i,\ell}^-} \nu(A_{i,\ell,k}^-).$
    \item Let $\hat{\phi}_i  = \frac{1}{n} \sum\limits_{\ell = 0}^{n-1} \hat{\phi}_{i,\ell}^+ - \hat{\phi}_{i,\ell}^-.$
    \item Let ${\sigma_{i,\ell}^+}^2 = \mathbb{V} \left[ \nu(A_{i,\ell} \cup \{i\}) \right]$ and ${\sigma_{i,\ell}^-}^2 = \mathbb{V} \left[ \nu(A_{i,\ell}) \right]$.
    \item Let $r_{i,\ell}^+ = \max\limits_{S \subseteq \mathcal{N} : i \notin S, |S| = \ell} \nu(S \cup \{i\}) - \min\limits_{S \subseteq \mathcal{N} : i \notin S, |S| = \ell} \nu(S \cup \{i\})$ be the range of $\nu(A_{i,\ell,k}^+)$.
    \item Let $r_{i,\ell}^- = \max\limits_{S \subseteq \mathcal{N} : i \notin S, |S| = \ell} \nu(S) - \min\limits_{S \subseteq \mathcal{N} : i \notin S, |S| = \ell} \nu(S)$ be  the range of $\nu(A_{i,\ell,k}^-)$.
    \item Let $R_i^+ = \sum\limits_{\ell=1}^{n-3} r_{i,\ell}^+$ and $R_i^- = \sum\limits_{\ell=2}^{n-2} r_{i,\ell}^-$.
\end{itemize}

\noindent \\
\textbf{Assumptions:}
\begin{itemize}
    \item $n \geq 4$, for $n \leq 3$ the algorithm computes all Shapley values exactly.
\end{itemize}

\subsection{Unbiasedness of Shapley Value Estimates}

\begin{lemma} \label{lem:StratSVARMExactEstimates}
    Due to the exact calculation, the following estimates are exact for all $i \in \mathcal{N}$:
    \begin{itemize}
        \item $\hat{\phi}_{i,0}^+ = \phi_{i,0}^+ = \nu(\{i\})$
        \item $\hat{\phi}_{i,n-2}^+ = \phi_{i,n-2}^+ = \frac{1}{n-1} \sum\limits_{j \in \mathcal{N} : j \neq i} \nu(\mathcal{N} \setminus \{j\})$
        \item $\hat{\phi}_{i,n-1}^+ = \phi_{i,n-1}^+ = \nu(\mathcal{N})$
        \item $\hat{\phi}_{i,0}^- = \phi_{i,0}^- = \nu(\emptyset) = 0$
        \item $\hat{\phi}_{i,1}^- = \phi_{i,1}^- = \frac{1}{n-1} \sum\limits_{j \in \mathcal{N} : j \neq i} \nu(\{j\})$
        \item $\hat{\phi}_{i,n-1}^- = \phi_{i,n-1}^- = \nu(\mathcal{N} \setminus \{i\})$
    \end{itemize}
\end{lemma}

\noindent
\begin{lemma} \label{lem:StratSVARMUnbiasedStrata}
    All remaining estimates that are not calculated exactly are unbiased, i.e., for all $i \in \mathcal{N}$:
    \begin{itemize}
        \item $\mathbb{E} \left[ \hat{\phi}_{i,\ell}^+ \right] = \phi_{i,\ell}^+$ for all $\ell \in \mathcal{L}^+$
        \item $\mathbb{E} \left[ \hat{\phi}_{i,\ell}^- \right] = \phi_{i,\ell}^-$ for all $\ell \in \mathcal{L}^-$
    \end{itemize}
\end{lemma}

\begin{proof}
We prove the statement only for $\hat{\phi}_{i,\ell}^+$ as the proof for $\hat{\phi}_{i,\ell}^-$ is analogous.
Fix any $i \in \mathcal{N}$ and $\ell \in \mathcal{L}^+$.
As soon as the size $s_t$ of the to be sampled coalition $|A_t|$ is fixed, $A_t$ is sampled uniformly from $\{S \subseteq \mathcal{N} \mid |S| = s_t\}$.
This allows us to state for every $A_t$ and any $S \subseteq \mathcal{N}$ with $|S| = \ell+1$ and $i \notin S$:
\begin{equation*}
     \mathbb{P} \left( A_t = S \cup \{i\} \mid i \in A_t, |A_t| = \ell+1 \right) = \frac{1}{\binom{n-1}{\ell}} .
\end{equation*}
Continuing, we derive for the expectation of $\hat{\phi}_{i,\ell}^+$ given the number of samples $m_{i,\ell}^+$:
\begin{align*}
    & \ \mathbb{E} \left[ \hat{\phi}_{i,\ell}^+ \mid m_{i,\ell}^+ \right] \\
    = & \ \mathbb{E} \left[ \frac{1}{m_{i,\ell}^+} \sum\limits_{k=1}^{m_{i,\ell}^+} \nu(A_{i,\ell,k}^+) \Bigg| m_{i,\ell}^+ \right] \\
    = & \ \frac{1}{m_{i,\ell}^+} \sum\limits_{k=1}^{m_{i,\ell}^+} \mathbb{E} \left[ \nu(A_{i,\ell,k}^+) \mid m_{i,\ell}^+ \right] \\
    = & \ \frac{1}{m_{i,\ell}^+} \sum\limits_{k=1}^{m_{i,\ell}^+} \sum\limits_{S \subseteq \mathcal{N} \setminus \{i\} : |S| = \ell} \mathbb{P} \left( A_{i,\ell,k}^+ = S \cup \{i\} \mid i \in A_{i,\ell,k}^+, |A_{i,\ell,k}^+| = \ell+1 \right) \cdot \nu(S \cup\{i\}) \\
    = & \ \frac{1}{m_{i,\ell}^+} \sum\limits_{k=1}^{m_{i,\ell}^+} \sum\limits_{S \subseteq \mathcal{N} \setminus \{i\} : |S| = \ell} \frac{1}{\binom{n-1}{\ell}} \cdot \nu(S \cup\{i\}) \\
    = & \ \frac{1}{m_{i,\ell}^+} \sum\limits_{k=1}^{m_{i,\ell}^+} \phi_{i,\ell}^+ \\
    = & \ \phi_{i,\ell}^+.
\end{align*}
Note that the term is well defined, since $m_{i,\ell}^+ \in \{1,\ldots,T\}$ due to the warm-up phase.
We conclude:
\begin{align*}
\mathbb{E} \left[ \hat{\phi}_{i,\ell}^+ \right] = & \ \sum\limits_{m=1}^T \mathbb{E} \left[ \hat{\phi}_{i,\ell}^+ \mid m_{i,\ell}^+ = m \right] \cdot \mathbb{P} \left( m_{i,\ell}^+ = m \right) \\
    = & \ \sum\limits_{m=1}^T \phi_{i,\ell}^+ \cdot \mathbb{P} \left( m_{i,\ell}^+ = m \right) \\
    = & \ \phi_{i,\ell}^+.
\end{align*}
\end{proof}

\noindent
\textbf{Theorem} \ref{the:StratSVARMUnbiased}
\textit{
    The Shapley value estimates for all $i \in \mathcal{N}$ are unbiased, i.e.,
    \begin{equation*}
        \mathbb{E} \left[ \hat{\phi}_i \right] = \phi_i .
    \end{equation*}
}

\begin{proof}
By applying \cref{lem:StratSVARMExactEstimates} and \cref{lem:StratSVARMUnbiasedStrata} we obtain:
\begin{align*}
    \mathbb{E} \left[ \hat{\phi}_i  \right] = & \ \frac{1}{n} \sum\limits_{\ell = 0}^{n-1} \mathbb{E} \left[ \hat{\phi}_{i,\ell}^+ \right] - \mathbb{E} \left[ \hat{\phi}_{i,\ell}^- \right] \\
    = & \ \frac{1}{n} \sum\limits_{\ell = 0}^{n-1} \phi_{i,\ell}^+ - \phi_{i,\ell}^- \\
    = & \ \phi_i .
\end{align*}
\end{proof}

\subsection{Sample numbers}

\begin{lemma} \label{lem:StratSVARMSampleDistribution}
    For any $i \in \mathcal{N}$ the numer of updates $\bar{m}_{i,\ell}^+$ and $m_{i,\ell}^-$ are binomially distributed with
    \begin{align*}
        & \bar{m}_{i,\ell}^+ \sim Bin \left(\bar{T}, \frac{\ell+1}{n} \cdot \tilde P(\ell+1) \right) \text{ for all } \ell \in \mathcal{L}^+ \\
        \text{and} \hspace{0.3cm} & \bar{m}_{i,\ell}^- \sim Bin \left(\bar{T}, \frac{n-\ell}{n} \cdot \tilde P(\ell) \right) \text{ for all } \ell \in \mathcal{L}^- .
    \end{align*}
\end{lemma}

\begin{proof}
We argue that there are $\bar{T}$ many independent time steps in which $\hat{\phi}_{i,\ell}^+$ can be updated.
If $|A_t| = \ell+1$ then $i$ is included in $A_t$ with a probability of $\frac{\ell+1}{n}$ due to the uniform sampling of $A_t$ given that its size $s_t$ is fixed, leading to an update.
Since the choice of size and members of the set $A_t$ are independent, the probability of $\hat{\phi}_{i,\ell}^+$ being updated in time step $t$ is $\frac{\ell+1}{n} \cdot \tilde P(\ell+1)$.
The same argument holds true for $\hat\phi_{i,\ell}^-$ with an update probability of $\frac{n-\ell}{n} \cdot \tilde P(\ell)$ in each time step.
\end{proof}

\begin{lemma} \label{lem:StratSVARMExpectedSampleNumbers}
    For any $i \in \mathcal{N}$ the expected number of updates of $\hat\phi_{i,\ell}^+$ and $\hat\phi_{i,\ell}^-$ after the warm-up phase is at least
    \begin{align*}
        & \mathbb{E} \left[ \bar{m}_{i,\ell}^+ \right] \geq \frac{\bar{T}}{2n \log n} \text{ for all } \ell \in \mathcal{L}^+ \\
        \text{and} \hspace{0.3cm} & \mathbb{E} \left[ \bar{m}_{i,\ell}^- \right] \geq \frac{\bar{T}}{2n \log n} \text{ for all } \ell \in \mathcal{L}^-.
    \end{align*}
\end{lemma}

\begin{proof}
In the following we distinguish between different cases, depending on the parity of $n$ and size of $\ell$.
We will use the bound $H_{n} \leq \log n + 1$ and the following inequalities multiple times which hold true for $n \geq 4$:

\begin{equation*}
    \frac{1}{H_{\frac{n-1}{2}}-1} \geq \frac{1}{\log n}
    \hspace{0.5cm} \text{and} \hspace{0.5cm}
    \frac{n \log n - 1}{n \left( H_{\frac{n}{2}-1} - 1 \right)} \geq 1 . \\
\end{equation*}
We begin with the case of $n \nmid 2$ and $\ell \leq \frac{n-1}{2}-1$:
\begin{align*}
    \mathbb{E} \left[ \bar{m}_{i,\ell}^+ \right] = & \ \bar{T} \cdot \frac{\ell+1}{n} \cdot \tilde P(\ell+1) &
    \mathbb{E} \left[ \bar{m}_{i,\ell}^- \right] = & \ \bar{T} \cdot \frac{n-\ell}{n} \cdot \tilde P(\ell) \\
    = & \ \frac{\bar{T}}{2n} \cdot \frac{1}{H_{\frac{n-1}{2}}-1} &
    = & \ \frac{\bar{T}}{2n} \cdot \frac{n-\ell}{\ell} \cdot \frac{1}{H_{\frac{n-1}{2}}-1} \\
    \geq & \ \frac{\bar{T}}{2n \log n} &
    \geq & \ \frac{\bar{T}}{2n \log n} \\
\end{align*}
For $n \nmid 2$ and $\ell = \frac{n-1}{2}$ we obtain:
\begin{align*}
    \mathbb{E} \left[ \bar{m}_{i,\ell}^+ \right] = & \ \bar{T} \cdot \frac{\ell+1}{n} \cdot \tilde P(\ell+1) &
    \mathbb{E} \left[ \bar{m}_{i,\ell}^- \right] = & \ \bar{T} \cdot \frac{n-\ell}{n} \cdot \tilde P(\ell) \\
    = & \ \frac{\bar{T}}{2n} \cdot \frac{\ell+1}{n-\ell-1} \cdot \frac{1}{H_{\frac{n-1}{2}}-1} &
    = & \ \frac{\bar{T}}{2n} \cdot \frac{n-\ell}{\ell} \cdot \frac{1}{H_{\frac{n-1}{2}}-1} \\
    \geq & \ \frac{\bar{T}}{2n \log n} &
    \geq & \ \frac{\bar{T}}{2n \log n}. \\
\end{align*}
For $n \nmid 2$ and $\ell \geq \frac{n+1}{2}$ we obtain:
\begin{align*}
    \mathbb{E} \left[ \bar{m}_{i,\ell}^+ \right] = & \ \bar{T} \cdot \frac{\ell+1}{n} \cdot \tilde P(\ell+1) &
    \mathbb{E} \left[ \bar{m}_{i,\ell}^- \right] = & \ \bar{T} \cdot \frac{n-\ell}{n} \cdot \tilde P(\ell) \\
    = & \ \frac{\bar{T}}{2n} \cdot \frac{\ell+1}{n-\ell-1} \cdot \frac{1}{H_{\frac{n-1}{2}}-1} &
    = & \ \frac{\bar{T}}{2n} \cdot \frac{1}{H_{\frac{n-1}{2}}-1} \\
    \geq & \ \frac{\bar{T}}{2n \log n} &
    \geq & \ \frac{\bar{T}}{2n \log n} \\
\end{align*}
Switching to $n \mid 2$, we start with $\ell = \frac{n}{2}-1$:
\begin{align*}
    \mathbb{E} \left[ \bar{m}_{i,\ell}^+ \right] = & \ \bar{T} \cdot \frac{\ell+1}{n} \cdot \tilde P(\ell+1) &
    \mathbb{E} \left[ \bar{m}_{i,\ell}^- \right] = & \ \bar{T} \cdot \frac{n-\ell}{n} \cdot \tilde P(\ell) \\
    = & \ \frac{\bar{T}}{2n \log n} &
    = & \ \frac{\bar{T}}{2n \log n} \cdot \frac{n-\ell}{\ell} \cdot \frac{n \log n - 1}{n \left( H_{\frac{n}{2}-1}-1\right)} \\
    & &
    \geq & \ \frac{\bar{T}}{2n \log n} \\
\end{align*}
For $n \mid 2$ and $\ell = \frac{n}{2}$ we derive:
\begin{align*}
    \mathbb{E} \left[ \bar{m}_{i,\ell}^+ \right] = & \ \bar{T} \cdot \frac{\ell+1}{n} \cdot \tilde P(\ell+1) &
    \mathbb{E} \left[ \bar{m}_{i,\ell}^- \right] = & \ \bar{T} \cdot \frac{n-\ell}{n} \cdot \tilde P(\ell) \\
    = & \ \frac{\bar{T}}{2n \log n} \cdot \frac{\ell+1}{n-\ell-1} \cdot \frac{n \log n - 1}{n \left( H_{\frac{n}{2}-1} -1 \right)} &
     = & \ \frac{\bar{T}}{2n \log n} \\
    \geq & \ \frac{\bar{T}}{2n \log n} \\
\end{align*}
For $n \mid 2$ and $\ell \leq \frac{n}{2}-2$ we derive:
\begin{align*}
    \mathbb{E} \left[ \bar{m}_{i,\ell}^+ \right] = & \ \bar{T} \cdot \frac{\ell+1}{n} \cdot \tilde P(\ell+1) &
    \mathbb{E} \left[ \bar{m}_{i,\ell}^- \right] = & \ \bar{T} \cdot \frac{n-\ell}{n} \cdot \tilde P(\ell) \\
    = & \ \frac{\bar{T}}{2n \log n} \cdot \frac{n \log n - 1}{n \left( H_{\frac{n}{2}-1} -1 \right)} &
    = & \ \frac{\bar{T}}{2n \log n} \cdot \frac{n-\ell}{\ell} \cdot \frac{n \log n - 1}{n \left( H_{\frac{n}{2}-1}-1\right)} \\
    \geq & \ \frac{\bar{T}}{2n \log n} &
    \geq & \ \frac{\bar{T}}{2n \log n} \\
\end{align*}
Finally, $n \mid 2$ and $\ell \geq \frac{n}{2}+1$ yields:
\begin{align*}
    \mathbb{E} \left[ \bar{m}_{i,\ell}^+ \right] = & \ \bar{T} \cdot \frac{\ell+1}{n} \cdot  \tilde P(\ell+1) &
    \mathbb{E} \left[ \bar{m}_{i,\ell}^- \right] = & \ \bar{T} \cdot \frac{n-\ell}{n} \cdot  \tilde P(\ell) \\
    = & \ \frac{\bar{T}}{2n \log n} \cdot \frac{\ell+1}{n-\ell-1} \cdot \frac{n \log n - 1}{n \left( H_{\frac{n}{2}-1} -1 \right)} &
    = & \ \frac{\bar{T}}{2n \log n} \cdot \frac{n \log n - 1}{n \left( H_{\frac{n}{2}-1} - 1\right)} \\
    \geq & \ \frac{\bar{T}}{2n \log n} &
    \geq & \ \frac{\bar{T}}{2n \log n}
\end{align*}
\end{proof}

\subsection{Variance and Expected Squared Error}

\begin{lemma} \label{lem:StratSVARMVarianceGivenSamples}
    The variance of any player's Shapley value estimate $\hat{\phi}_i$ given the number of samples $m_{i,\ell}^+$ and $m_{i,\ell}^-$ for all $\ell \in \mathcal{L}$ is given by
    \begin{equation*}
        \mathbb{V} \left[ \hat{\phi}_i \Big| \left(m_{i,\ell}^+\right)_{\ell \in \mathcal{L}^+}, \left( m_{i,\ell}^-\right)_{\ell \in \mathcal{L}^-} \right] = \frac{1}{n^2} \sum\limits_{\ell=1}^{n-3} \frac{{\sigma_{i,\ell}^+}^2}{m_{i,\ell}^+} + \frac{{\sigma_{i,\ell+1}^-}^2}{m_{i,\ell+1}^-}.
    \end{equation*}
\end{lemma}

\begin{proof}
We first decompose the variance of $\hat{\phi}_i$ into the variances of $\hat{\phi}_i^+$ and $\hat{\phi}_i^-$ and their covariance:
\begin{align*}
    & \ \mathbb{V} \left[ \hat{\phi}_i \mid \left(m_{i,\ell}^+\right)_{\ell \in \mathcal{L}^+}, \left( m_{i,\ell}^-\right)_{\ell \in \mathcal{L}^-} \right] \\
    = & \ \mathbb{V} \left[ \hat{\phi}_i^+ \mid \left( m_{i,\ell}^+ \right)_{\ell \in \mathcal{L}^+} \right] + \mathbb{V} \left[ \hat{\phi}_i^- \mid \left( m_{i,\ell}^- \right)_{\ell \in \mathcal{L}^-} \right] - 2 \text{Cov} \left( \hat{\phi}_i^+, \hat{\phi}_i^- \mid \left(m_{i,\ell}^+\right)_{\ell \in \mathcal{L}^+}, \left( m_{i,\ell}^-\right)_{\ell \in \mathcal{L}^-} \right) \\
     = & \ \mathbb{V} \left[ \hat{\phi}_i^+ \mid \left( m_{i,\ell}^+ \right)_{\ell \in \mathcal{L}^+} \right] + \mathbb{V} \left[ \hat{\phi}_i^- \mid \left( m_{i,\ell}^- \right)_{\ell \in \mathcal{L}^-} \right] .
\end{align*}
where we used the observation that $\hat\phi_i^+$ and $\hat\phi_i^-$ are independent.
We derive for $\hat{\phi}_i^+$:
\begin{align*}
    \mathbb{V} \left[ \hat{\phi}_i^+ \mid \left( m_{i,\ell}^+ \right)_{\ell \in \mathcal{L}^+} \right] = & \ \frac{1}{n^2} \sum\limits_{\ell=1}^{n-3} \mathbb{V} \left[ \hat{\phi}_{i,\ell}^+ \mid m_{i,\ell}^+ \right] + \sum\limits_{\ell \neq \ell'} \text{Cov} \left( \hat{\phi}_{i,\ell}^+, \hat{\phi}_{i,\ell'}^+ \mid m_{i,\ell}^+, m_{i,\ell'}^+ \right) \\
    = & \ \frac{1}{n^2} \sum\limits_{\ell=1}^{n-3} \mathbb{V} \left[ \hat{\phi}_{i,\ell}^+ \mid m_{i,\ell}^+ \right] \\
    = & \ \frac{1}{n^2} \sum\limits_{\ell=1}^{n-3} \frac{{\sigma_{i,\ell}^+}^2}{m_{i,\ell}^+},
\end{align*}
where we used the observation that $\hat{\phi}_{i,\ell}^+$ and $\hat{\phi}_{i,\ell'}^+$ are independent for $\ell \neq \ell'$.
Note that $\hat{\phi}_{i,0}^+$, $\hat{\phi}_{i,n-2}^+$, $\hat{\phi}_{i,n-1}^+$, $\hat{\phi}_{i,0}^-$, $\hat{\phi}_{i,1}^-$, and $\hat{\phi}_{i,n-1}^-$ are constants without variance.
A similar result can be obtained for $\hat{\phi}_i^-$.
Putting our intermediate results together yields:
\begin{align*}
    & \ \mathbb{V} \left[ \hat{\phi}_i \Big| \left( m_{i,\ell}^+ \right)_{\ell \in \mathcal{L}^+}, \left( m_{i,\ell}^- \right)_{\ell \in \mathcal{L}^-} \right] \\
    = & \ \mathbb{V} \left[ \hat{\phi}_i^+ \mid \left( m_{i,\ell}^+ \right)_{\ell \in \mathcal{L}^+} \right] + \mathbb{V} \left[ \hat{\phi}_i^- \mid \left( m_{i,\ell}^- \right)_{\ell \in \mathcal{L}^-} \right] \\
    = & \ \frac{1}{n^2} \sum\limits_{\ell=1}^{n-3} \frac{{\sigma_{i,\ell}^+}^2}{m_{i,\ell}^+} + \frac{1}{n^2} \sum\limits_{\ell=2}^{n-2} \frac{{\sigma_{i,\ell}^-}^2}{m_{i,\ell}^-} \\
    = & \ \frac{1}{n^2} \sum\limits_{\ell=1}^{n-3} \frac{{\sigma_{i,\ell}^+}^2}{m_{i,\ell}^+} + \frac{{\sigma_{i,\ell+1}^-}^2}{m_{i,\ell+1}^-}.
\end{align*}
\end{proof}

\begin{lemma} \label{lem:StratSVARMInvertedExpectation}
    For any $i \in \mathcal{N}$ holds
    \begin{equation*}
        \mathbb{E} \left[ \frac{1}{m_{i,\ell}^+} \right] \leq \frac{2n \log n}{\bar{T}} \text{ for all } \ell \in \mathcal{L}^+ \hspace{0.3cm} \text{and} \hspace{0.3cm} \mathbb{E} \left[ \frac{1}{m_{i,\ell}^-} \right] \leq \frac{2n \log n}{\bar{T}} \text{ for all } \ell \in \mathcal{L}^-.
    \end{equation*}
\end{lemma}

\begin{proof}
We prove the result only for $m_{i,\ell}^+$ since the proof for $m_{i,\ell}^-$ is analogous.
By combining Equation~(3.4) in \cite{Chao.1972}:
\begin{equation*}
    \mathbb{E} \left[ \frac{1}{1 + X} \right] = \frac{1 - (1-p)^{m+1}}{(m+1)p} \leq \frac{1}{mp} = \frac{1}{\mathbb{E}[X]},
\end{equation*}
for any binomially distributed random variable $X \sim Bin(m,p)$ with 
 \cref{lem:StratSVARMSampleDistribution} and \cref{lem:StratSVARMExpectedSampleNumbers}, we obtain:
\begin{equation*}
    \mathbb{E} \left[ \frac{1}{m_{i,\ell}^+} \right] =  \mathbb{E} \left[ \frac{1}{1 + \bar{m}_{i,\ell}^+} \right] \leq \frac{1}{\mathbb{E} \left[ \bar{m}_{i,\ell}^+ \right]} \leq \frac{2 n \log n}{\bar{T}} .
\end{equation*}
\end{proof}

\noindent
\textbf{Theorem} \ref{the:StratSVARMVariance}
\textit{
    For $\tilde P$ as chosen above the variance of any player's Shapley value estimate $\hat{\phi}_i$ is bounded by
    \begin{equation*}
        \mathbb{V} \left[ \hat{\phi}_i \right] \leq \frac{2\log n}{n \bar{T}} \sum\limits_{\ell=1}^{n-3} {\sigma_{i,\ell}^+}^2 + {\sigma_{i,{\ell+1}}^-}^2.
    \end{equation*}
}

\begin{proof}
The combination of \cref{lem:StratSVARMVarianceGivenSamples} and \cref{lem:StratSVARMInvertedExpectation} yields:
\begin{align*}
    \mathbb{V} \left[ \hat{\phi}_i \right] = & \ \mathbb{E} \left[ \mathbb{V} \left[ \hat{\phi}_i \Big| \left( m_{i,\ell}^+ \right)_{\ell \in \mathcal{L}^+}, \left( m_{i,\ell}^- \right)_{\ell \in \mathcal{L}^-} \right] \right] \\
    \leq & \ \mathbb{E} \left[ \frac{1}{n^2} \sum\limits_{\ell=1}^{n-3} \frac{{\sigma_{i,\ell}^+}^2}{m_{i,\ell}^+} + \frac{ {\sigma_{i,\ell+1}^-}^2}{m_{i,\ell+1}^-} \right] \\
    = & \ \frac{1}{n^2} \sum\limits_{\ell=1}^{n-3} {\sigma_{i,\ell}^+}^2 \cdot \mathbb{E} \left[ \frac{1}{m_{i,\ell}^+} \right] + {\sigma_{i,\ell+1}^-}^2 \cdot \mathbb{E} \left[ \frac{1}{m_{i,\ell+1}^-} \right] \\
    \leq & \ \frac{2\log n}{n \bar{T}} \sum\limits_{\ell=1}^{n-3} {\sigma_{i,\ell}^+}^2 + {\sigma_{i,{\ell+1}}^-}^2 .
\end{align*}
\end{proof}

\textbf{Corollary} \ref{cor:StratSVARMSE}
\textit{
    For $\tilde P$ as chosen above the MSE of any player's Shapley value estimate $\hat{\phi}_i$ is bounded by
    \begin{equation*}
        \mathbb{E} \left[ \left( \hat{\phi}_i - \phi_i \right)^2 \right] \leq \frac{2\log n}{n \bar{T}} \sum\limits_{\ell=1}^{n-3} {\sigma_{i,\ell}^+}^2 + {\sigma_{i,{\ell+1}}^-}^2.
    \end{equation*}
}

\begin{proof}
Using the bias-variance decomposition, the unbiasedness of $\hat{\phi}_i$ shown in \cref{the:StratSVARMUnbiased}, and the bound on the variance from \cref{the:StratSVARMVariance} we obtain that:
\begin{align*}
     \mathbb{E} \left[ \left( \hat{\phi}_i - \phi_i \right)^2 \right] = & \ \left( \mathbb{E} \left[ \hat{\phi_i} \right] - \phi_i \right)^2 + \mathbb{V} \left[ \hat{\phi}_i \right] \\
     \leq & \ \frac{2\log n}{n \bar{T}} \sum\limits_{\ell=1}^{n-3} {\sigma_{i,\ell}^+}^2 + {\sigma_{i,{\ell+1}}^-}^2 .
\end{align*}
\end{proof}

\subsection{Probabilistic Bounds}

\textbf{Theorem} \ref{the:StratSVARMCheby}
\textit{
    Fix any player $i \in \mathcal{N}$ and $\varepsilon > 0$.
    For $\tilde P$ as above the probability that the estimate $\hat{\phi}_i$ deviates from $\phi_i$ by a margin of $\varepsilon$ or greater is bounded by
    \begin{equation*}
        \mathbb{P} \left( |\hat{\phi}_i - \phi| \geq \varepsilon \right) \leq \frac{2\log n}{\varepsilon^2 n \bar{T}} \sum\limits_{\ell=1}^{n-3} {\sigma_{i,\ell}^+}^2 + {\sigma_{i,{\ell+1}}^-}^2 .
    \end{equation*}
}

\begin{proof}
The bound on the variance of $\hat{\phi}_i$ in \cref{the:StratSVARMVariance} allows us to apply Chebyshev's inequality:
\begin{equation*}
    \mathbb{P} \left( |\hat{\phi}_i - \phi| \geq \varepsilon \right) \leq \frac{\mathbb{V} \left[ \hat{\phi}_i \right]}{\varepsilon^2} \leq \frac{2\log n}{\varepsilon^2 n \bar{T}} \sum\limits_{\ell=1}^{n-3} {\sigma_{i,\ell}^+}^2 + {\sigma_{i,{\ell+1}}^-}^2 .
\end{equation*}
\end{proof}

\begin{corollary} \label{cor:StratSVARMCheby}
    Fix any player $i \in \mathcal{N}$ and $\delta \in (0,1]$.
    The estimate $\hat{\phi}_i$ deviates from $\phi_i$ by a margin of $\varepsilon$ or greater with probability not greater than $\delta$, i.e.,
    \begin{equation*}
        \mathbb{P} \left( | \hat{\phi}_i - \phi_i | \geq \varepsilon \right) \leq \delta \hspace{0.5cm} \text{for} \hspace{0.5cm} \varepsilon = \sqrt{\frac{2\log n}{\delta n \bar{T}} \sum\limits_{\ell=1}^{n-3} {\sigma_{i,\ell}^+}^2 + {\sigma_{i,{\ell+1}}^-}^2} .
    \end{equation*}
\end{corollary}

\begin{lemma} \label{lem:StratSVARMProbBoundStratumConditional}
    For any $i \in \mathcal{N}$ and fixed $\varepsilon > 0$ holds:
    \begin{itemize}
        \item $\mathbb{P}(|\hat{\phi}_{i,\ell}^+ - \phi_{i,\ell}^+ | \geq \varepsilon \mid m_{i,\ell}^+) \leq 2 \exp \left( - \frac{2 m_{i,\ell}^+ \varepsilon^2}{{r_{i,\ell}^+}^2} \right)$ for all $\ell \in \mathcal{L}^+$
        \item $\mathbb{P}(|\hat{\phi}_{i,\ell}^- - \phi_{i,\ell}^- | \geq \varepsilon \mid m_{i,\ell}^-) \leq 2 \exp \left( - \frac{2 m_{i,\ell}^- \varepsilon^2}{{r_{i,\ell}^-}^2} \right)$ for all $\ell \in \mathcal{L}^-$
    \end{itemize}
\end{lemma}

\begin{proof}
We prove the statement for $\hat{\phi}_{i,\ell}^+$ by making use of Hoeffding's inequality in combination with the unbiasedness of the strata estimates shown in \cref{lem:StratSVARMUnbiasedStrata}.
The proof for $\hat{\phi}_{i,\ell}^-$ is analogous.
\begin{align*}
    & \ \mathbb{P}(|\hat{\phi}_{i,\ell}^+ - \phi_{i,\ell}^+| \geq \varepsilon \mid m_{i,\ell}^+) \\
    = & \ \mathbb{P} \left( | \hat{\phi}_{i,\ell}^+ - \mathbb{E}[\hat{\phi}_{i,\ell}^+] | \geq \varepsilon \mid m_{i,\ell}^+ \right) \\
    = & \ \mathbb{P} \left( \Bigg| \sum\limits_{k=1}^{m_{i,\ell}^+} \nu(A_{i,\ell,k}^+) - \mathbb{E} \left[\sum\limits_{k=1}^{m_{i,\ell}^+} \nu(A_{i,\ell,k}^+) \right] \Bigg| \geq m_{i,\ell}^+ \varepsilon \mid m_{i,\ell}^+ \right) \\
    \leq & \ 2 \exp \left( - \frac{2 m_{i,\ell}^+ \varepsilon^2}{{r_{i,\ell}^+}^2} \right).
\end{align*}
\end{proof}

\begin{lemma} \label{lem:StratSVARMProbBoundStratum}
    For any $i \in \mathcal{N}$ and fixed $\varepsilon > 0$ holds:
    \begin{itemize}
        \item $\mathbb{P} \left( | \hat{\phi}_{i,\ell}^+ - \phi_{i,\ell}^+ | \geq \varepsilon \right) \leq \exp \left( - \frac{\bar{T}}{8n^2 (\log n)^2} \right) + 2 \frac{\exp \left( - \frac{2\varepsilon^2}{{r_{i,\ell}^+}^2} \right)^{\left\lfloor \frac{\bar{T}}{4n \log n} \right\rfloor}}{\exp \left( \frac{2\varepsilon^2}{{r_{i,\ell}^+}^2} \right) -1}$ for all $\ell \in \mathcal{L}^+$
        \item $\mathbb{P} \left( | \hat{\phi}_{i,\ell}^- - \phi_{i,\ell}^- | \geq \varepsilon \right) \leq \exp \left( - \frac{\bar{T}}{8n^2 (\log n)^2} \right) + 2 \frac{\exp \left( - \frac{2\varepsilon^2}{{r_{i,\ell}^-}^2} \right)^{\left\lfloor \frac{\bar{T}}{4n \log n} \right\rfloor}}{\exp \left( \frac{2\varepsilon^2}{{r_{i,\ell}^-}^2} \right) -1}$ for all $\ell \in \mathcal{L}^-$
    \end{itemize}
\end{lemma}

\begin{proof}
We prove the statement for $\hat{\phi}_{i,\ell}^+$.
The proof for $\hat{\phi}_{i,\ell}^-$ is analogous.
To begin with, we derive with the help of Hoeffding's inequality for binomial distributions a bound for the probability of $\bar{m}_{i,\ell}^+$ not exceeding $\frac{\bar{T}}{4 n\log n}$:
\begin{align*}
    & \ \mathbb{P} \left( \bar{m}_{i,\ell}^+ \leq \frac{\bar{T}}{4 n\log n} \right) \\
    \leq & \ \mathbb{P} \left( \mathbb{E} \left[ \bar{m}_{i,\ell}^+ \right] - \bar{m}_{i,\ell}^+ \geq \mathbb{E} \left[ \bar{m}_{i,\ell}^+ \right] - \frac{\bar{T}}{4 n\log n} \right) \\
    \leq & \ \exp \left( - \frac{2 \left(\mathbb{E} \left[ \bar{m}_{i,\ell}^+ \right] - \frac{\bar{T}}{4 n\log n} \right)^2}{\bar{T}} \right) \\
    \leq & \ \exp \left( - \frac{\bar{T}}{8n^2 (\log n)^2} \right),
\end{align*}
where we used the lower bound on $\mathbb{E} \left[ \bar{m}_{i,\ell}^+ \right]$ shown in \cref{lem:StratSVARMExpectedSampleNumbers}.
Next, we derive with the help of \cref{lem:StratSVARMProbBoundStratumConditional} a statement of technical nature to be used later:
\begin{align*}
    & \ \sum\limits_{m = \left\lfloor \frac{\bar{T}}{4 n\log n} \right\rfloor +1}^{\bar{T}} \mathbb{P} \left(|\hat{\phi}_{i,\ell}^+ - \phi_{i,\ell}^+ | \geq \varepsilon \mid m_{i,\ell}^+ = m \right) \\
    \leq & \ 2 \sum\limits_{m = \left\lfloor \frac{\bar{T}}{4 n\log n} \right\rfloor +1}^{\bar{T}} \exp \left( - \frac{2 m \varepsilon^2}{{r_{i,\ell}^+}^2} \right) \\
    = & \ 2 \sum\limits_{m = 0}^{\bar{T}} \exp \left( - \frac{2\varepsilon^2}{{r_{i,\ell}^+}^2} \right)^m - 2\sum\limits_{m = 0}^{\left\lfloor \frac{\bar{T}}{4 n\log n} \right\rfloor} \exp \left( - \frac{2 \varepsilon^2}{{r_{i,\ell}^+}^2} \right)^m \\
    = & \ 2 \frac{\exp \left( - \frac{2\varepsilon^2}{{r_{i,\ell}^+}^2} \right)^{\left\lfloor \frac{\bar{T}}{4 n\log n} \right\rfloor} - \exp \left( - \frac{2\varepsilon^2}{{r_{i,\ell}^+}^2} \right)^{\bar{T}}}{\exp \left( \frac{2\varepsilon^2}{{r_{i,\ell}^+}^2} \right) -1} \\
    \leq & \ 2 \frac{\exp \left( - \frac{2\varepsilon^2}{{r_{i,\ell}^+}^2} \right)^{\left\lfloor \frac{\bar{T}}{4 n\log n} \right\rfloor}}{\exp \left( \frac{2\varepsilon^2}{{r_{i,\ell}^+}^2} \right) -1} .
\end{align*}
At last, putting both findings together, we derive our claim:
\begin{align*}
    & \ \mathbb{P} \left( | \hat{\phi}_{i,\ell}^+ - \phi_{i,\ell}^+ | \geq \varepsilon \right) \\
    = & \ \sum\limits_{m = 1}^{\bar{T}} \mathbb{P} \left(|\hat{\phi}_{i,\ell}^+ - \phi_{i,\ell}^+ | \geq \varepsilon \mid m_{i,\ell}^+ = m \right) \cdot \mathbb{P} \left( m_{i,\ell}^+ = m \right) \\
    = & \ \sum\limits_{m = 1}^{ \left\lfloor \frac{\bar{T}}{4 n\log n} \right\rfloor} \mathbb{P} \left(|\hat{\phi}_{i,\ell}^+ - \phi_{i,\ell}^+ | \geq \varepsilon \mid m_{i,\ell}^+ = m \right) \cdot \mathbb{P} \left( m_{i,\ell}^+ = m \right) \\
    & \quad + \sum\limits_{m = \left\lfloor \frac{\bar{T}}{4 n\log n} \right\rfloor +1}^{\bar{T}} \mathbb{P} \left(|\hat{\phi}_{i,\ell}^+ - \phi_{i,\ell}^+ | \geq \varepsilon \mid m_{i,\ell}^+ = m \right) \cdot \mathbb{P} \left( m_{i,\ell}^+ = m \right) \\
    \leq & \ \mathbb{P} \left( \bar{m}_{i,\ell}^+ \leq \left\lfloor \frac{\bar{T}}{4 n \log n} \right\rfloor \right) + \sum\limits_{m = \left\lfloor \frac{\bar{T}}{4 n\log n} \right\rfloor+1}^{\bar{T}} \mathbb{P} \left(|\hat{\phi}_{i,\ell}^+ - \phi_{i,\ell}^+ | \geq \varepsilon \mid m_{i,\ell}^+ = m \right) \\
    \leq & \ \exp \left( - \frac{\bar{T}}{8n^2 (\log n)^2} \right) + 2 \frac{\exp \left( - \frac{2\varepsilon^2}{{r_{i,\ell}^+}^2} \right)^{\left\lfloor \frac{\bar{T}}{4n \log n} \right\rfloor}}{\exp \left( \frac{2\varepsilon^2}{{r_{i,\ell}^+}^2} \right) -1} .
\end{align*}
\end{proof}

\begin{lemma} \label{lem:StratSVARMProbBoundAverages}
    For any $i \in \mathcal{N}$ and fixed $\varepsilon > 0$ the probabilities that the estimates $\hat{\phi}_i^+$ and $\hat{\phi}_i^-$ deviate from $\phi_i^+$, respectively $\phi_i^-$ are bounded by:
    \begin{itemize}
        \item $\mathbb{P} \left( | \hat{\phi}_i^+ - \phi_i^+ | \geq \varepsilon \right) \leq (n-3) \left( \exp \left( - \frac{\bar{T}}{8n^2 (\log n)^2} \right) + 2 \frac{\exp \left( - \frac{2\varepsilon^2 n^2}{{R_i^+}^2} \right)^{\left\lfloor \frac{\bar{T}}{4n \log n} \right\rfloor}}{\exp \left( \frac{2 \varepsilon^2 n^2}{{R_i^+}^2} \right) -1} \right)$,
        \item $\mathbb{P} \left( | \hat{\phi}_i^- - \phi_i^- | \geq \varepsilon \right) \leq (n-3) \left( \exp \left( - \frac{\bar{T}}{8n^2 (\log n)^2} \right) + 2 \frac{\exp \left( - \frac{2\varepsilon^2 n^2}{{R_i^-}^2} \right)^{\left\lfloor \frac{\bar{T}}{4n \log n} \right\rfloor}}{\exp \left( \frac{2 \varepsilon^2 n^2}{{R_i^-}^2} \right) -1} \right)$ .
    \end{itemize}
\end{lemma}

\begin{proof}
We prove the statement for $\hat{\phi}_i^+$ using \cref{lem:StratSVARMProbBoundStratum}.
The proof for $\hat{\phi}_i^-$ is analogous.
\begin{align*}
    & \ \mathbb{P} \left( | \hat{\phi}_i^+ - \phi_i^+ | \geq \varepsilon \right) \\
    = & \ \mathbb{P} \left( \Big|  \frac{1}{n} \sum\limits_{\ell=0}^{n-1} \hat{\phi}_{i,\ell}^+ - \phi_{i,\ell}^+ \Big| \geq \varepsilon \right) \\
    \leq & \ \mathbb{P} \left( \frac{1}{n} \sum\limits_{\ell=0}^{n-1} | \hat{\phi}_{i,\ell}^+ - \phi_{i,\ell}^+ | \geq \varepsilon \right) \\
    = & \ \mathbb{P} \left( \sum\limits_{\ell=1}^{n-3} | \hat{\phi}_{i,\ell}^+ - \phi_{i,\ell}^+ | \geq \varepsilon n \right) \\
    \leq & \ \sum\limits_{\ell = 1}^{n-3} \mathbb{P} \left( | \hat{\phi}_{i,\ell}^+ - \phi_{i,\ell}^+ | \geq \frac{\varepsilon n r_{i,\ell}^+}{R^+} \right) \\
    \leq & \ (n-3) \left( \exp \left( - \frac{\bar{T}}{8n^2 (\log n)^2} \right) + 2 \frac{\exp \left( - \frac{2\varepsilon^2 n^2}{{R^+}^2} \right)^{\left\lfloor \frac{\bar{T}}{4n \log n} \right\rfloor}}{\exp \left( \frac{2 \varepsilon^2 n^2}{{R^+}^2} \right) -1} \right)
\end{align*}
\end{proof}

\noindent
\textbf{Theorem} \ref{the:StratSVARMHoeffding}
\textit{
    For any $i \in \mathcal{N}$ and fixed $\varepsilon > 0$ the probability that the estimate $\hat{\phi}_i$ deviates from $\phi_i$ by a margin of $\varepsilon$ or greater is bounded by 
    \begin{equation*}
        \mathbb{P} \left( | \hat{\phi}_i - \phi_i | \geq \varepsilon \right) \leq 2(n-3) \left( \exp \left( - \frac{\bar{T}}{8n^2 (\log n)^2} \right) + 2 \frac{\exp \left( - \frac{2\varepsilon^2 n^2}{{(R_i^+ + R_i^-)}^2} \right)^{\left\lfloor \frac{\bar{T}}{4n \log n} \right\rfloor}}{\exp \left( \frac{2 \varepsilon^2 n^2}{{(R_i^+ + R_i^-)}^2} \right) -1} \right) .
    \end{equation*}
}

\begin{proof}
We apply \cref{lem:StratSVARMProbBoundAverages} and obtain:
\begin{align*}
    & \ \mathbb{P} \left( | \hat{\phi}_i - \phi_i | \geq \varepsilon \right) \\
    = & \ \mathbb{P} \left( | (\hat{\phi}_i^+ - \phi_i^+) + (\phi_i^- - \hat{\phi}_i^-) | \geq \varepsilon \right) \\
    \leq & \ \mathbb{P} \left( | \hat{\phi}_i^+ - \phi_i^+ | + | \hat{\phi}_i^- - \phi_i^- | \geq \varepsilon \right) \\
    \leq & \ \mathbb{P} \left( | \hat{\phi}_i^+ - \phi_i^+ | \geq \frac{\varepsilon R_i^+}{R_i^+ + R_i^-} \right) + \mathbb{P} \left( | \hat{\phi}_i^- - \phi_i^- | \geq \frac{\varepsilon R_i^-}{R_i^+ + R_i^-} \right) \\
    \leq & \ 2(n-3) \left( \exp \left( - \frac{\bar{T}}{8n^2 (\log n)^2} \right) + 2 \frac{\exp \left( - \frac{2\varepsilon^2 n^2}{(R_i^+ + R_i^-)^2} \right)^{\left\lfloor \frac{\bar{T}}{4n \log n} \right\rfloor}}{\exp \left( \frac{2 \varepsilon^2 n^2}{(R_i^+ + R_i^-)^2} \right) -1} \right).
\end{align*}
\end{proof}
\newpage

\section{Cooperative Games} \label{app:CooperativeGames}

\subsection{Synthetic games}

We provide formal definitions of the synthetic games and their Shapley values used in our empirical evaluation (see Section~\ref{sec:EmpiricalResults}), and describe the process of how we randomly generated some of these.

\subsubsection{Shoe Game}

The number of players $n$ in the Shoe game has to be even.
The player set consist of two halves $A$ and $B$ of equal size, i.e., $\mathcal{N} = A \cup B$ with $A \cap B = \emptyset$ and $|A| = |B| = \frac{n}{2}$.
The value function is given by $\nu(S) = \min \{|S \cap A|,|S \cap B|\}$.
All players share the same Shapley value of $\phi_i = \frac{1}{2}$.

\subsubsection{Airport Game}

The Airport game entails $n=100$ players.
Each player $i$ has an assigned weight $c_i$.
The value function is the maximum of all weights contained in the coalition, i.e., $\nu(S) = \max_{i \in S} c_i $.
The weights and resulting Shapley values are:
\begin{equation*}
    c_i = \begin{cases}
    1 & \text{if } i \in \{1,\ldots,8\} \\
    2 & \text{if } i \in \{9,\ldots,20\} \\
    3 & \text{if } i \in \{21,\ldots,26\} \\
    4 & \text{if } i \in \{27,\ldots,40\} \\
    5 & \text{if } i \in \{41,\ldots,48\} \\
    6 & \text{if } i \in \{49,\ldots,57\} \\
    7 & \text{if } i \in \{58,\ldots,70\} \\
    8 & \text{if } i \in \{71,\ldots,80\} \\
    9 & \text{if } i \in \{81,\ldots,90\} \\
    10 & \text{if } i \in \{91,\ldots,100\}
\end{cases}, \hspace{2cm}
\phi_i = \begin{cases}
    0.01 & \text{if } i \in \{1,\ldots,8\} \\
    0.020869565 & \text{if } i \in \{9,\ldots,20\} \\
    0.033369565 & \text{if } i \in \{21,\ldots,26\} \\
    0.046883079 & \text{if } i \in \{27,\ldots,40\} \\
    0.063549745 & \text{if } i \in \{41,\ldots,48\} \\
    0.082780515 & \text{if } i \in \{49,\ldots,57\} \\
    0.106036329 & \text{if } i \in \{58,\ldots,70\} \\
    0.139369662 & \text{if } i \in \{71,\ldots,80\} \\
    0.189369662 & \text{if } i \in \{81,\ldots,90\} \\
    0.289369662 & \text{if } i \in \{91,\ldots,100\}
\end{cases}.
\end{equation*}

\subsubsection{SOUG Game}

A Sum of unanimity games (SOUG) is specified by $M$ many sets $S_1,\ldots,S_M \subseteq \mathcal{N}$ and weights $c_1,\ldots,c_M \in \mathbb{R}$.
The value functions is defined as $\nu(S) = \sum\limits_{m=1}^M c_m \cdot \mathbb{I}_{S_m \subseteq S}$ leading to Shapley values  $\phi_i = \sum\limits_{m=1}^M \frac{c_m}{|S_m|} \cdot \mathbb{I}_{i \in S_m}$, which can be computed in polynomial time if knowledge of sets and coefficients is provided.
We generate SOUG games with $M=50$ randomly by selecting for each $S_m$ to be drawn a size uniformly at random between $1$ and $n$, and then draw the set $S_m$ with that size uniformly.
We draw the coefficients uniformly at random from $[0,1]$.

\subsection{Explainability games}
In the following, we describe the three explainability games introduced in Section~\ref{sec:EmpiricalResults}; namely, the NLP sentiment analysis game (see Section~\ref{app:nlp_game}), the image classifier game (see Section~\ref{app:image_classificaion_game}), and the adult classification (see Section~\ref{app:adult_classification}),  and explain the value function of each resulting cooperative game.
Since there exists no efficient closed-form solution, we compute the Shapley values exhaustively (via brute force) in order to allow the tracking of the approximation error of the different algorithms.
Due to constraints in computational power this limits us to $n=14$ players per game, which necessitates $2^{14} = 16\,384$ model evaluation to exhaustively traverse the powerset of all coalitions.
Note that after a budget of $T=2^{14}$ both KernelSHAP and Stratified SVARM$^+$ have an approximation error of zero because both have observed all coalition values and thus have seen the cooperative game in its entirety.

\subsubsection{NLP sentiment analysis} \label{app:nlp_game}
The NLP sentiment game describes an explainability scenario for local feature importance of a sentiment classification model.
The sentiment classifier\footnote{\url{https://huggingface.co/dhlee347/distilbert-imdb}.} is a fine-tuned version of the DistilBERT transformer architecture \cite{Sanh.2019}.
The model was fine-tuned on the \emph{IMDB} dataset \cite{Maas.2011}.
The model expects a natural language sentence as input, transforms the sentence into a tokenized form and predicts a sentiment score ranging from $[-1,1]$.
We randomly select sentences from the \emph{IMDB} dataset, that contain no more than 14 tokens.
For a sentence the local explainability game consists of presenting the model a coalition of players (tokens) and observing the predicted sentiment as the value of a coalition.
Absent players are removed in the tokenized representation (i.e.\ tokens are removed).

\subsubsection{Image classifier} \label{app:image_classificaion_game}
The image classifier game is similar to the NLP sentiment analysis game (Section~\ref{app:nlp_game}) a local explanation scenario.
For this we explain the output of an image classifier given random images from ImageNet \cite{ImageNet}.
The model to be explained is a ResNet18\footnote{\url{https://pytorch.org/vision/main/models/generated/torchvision.models.resnet18.html}.} \cite{resnet18} trained on ImageNet \cite{ImageNet}. 
To restrict the number of players, we apply SLIC \cite{SLIC} to summarize individual pixels with 14 super-pixels.
The super-pixels then make up the players in the image classification game.
A coalition of players, thus, consists of the corresponding super-pixels.
The super-pixels of absent players are removed via mean-imputation by setting all their pixels to grey. 
The worth of a coalition is determined by the output of the model (using only the present super-pixels given by the coalition) for the class of the original prediction which was made with all pixels being present.

\subsubsection{Adult classification} \label{app:adult_classification}
Similar to the preceding two games, the adult classification game is also a local explanation scenario.
We train a gradient-boosted tree classifier (sklearn) on the adult dataset \cite{Becker.1996} to classify whether an adult has an income below or above $50\,000$.
Each game is based on a randomly chosen datapoint for which the players correspond to features.
A coalition is formed by removing the absent feature values of the selected datapoint via mean imputation.
The worth of a coalition is the predicted class probability of the true income class given the manipulated datapoint after mean imputation of absent features.
\newpage

\section{Further Empirical Results} \label{app:EmpiricalResults}
The plots shown in \cref{fig:mainResults} hardly visualize the performance differences between the algorithms with low MSE values.
Thus, we present our findings in \cref{fig:Airport} to \cref{fig:Adult} in higher resolution.

\begin{figure*}[ht]
	\centering
	\subfigure{
		\includegraphics[width=0.485\textwidth]{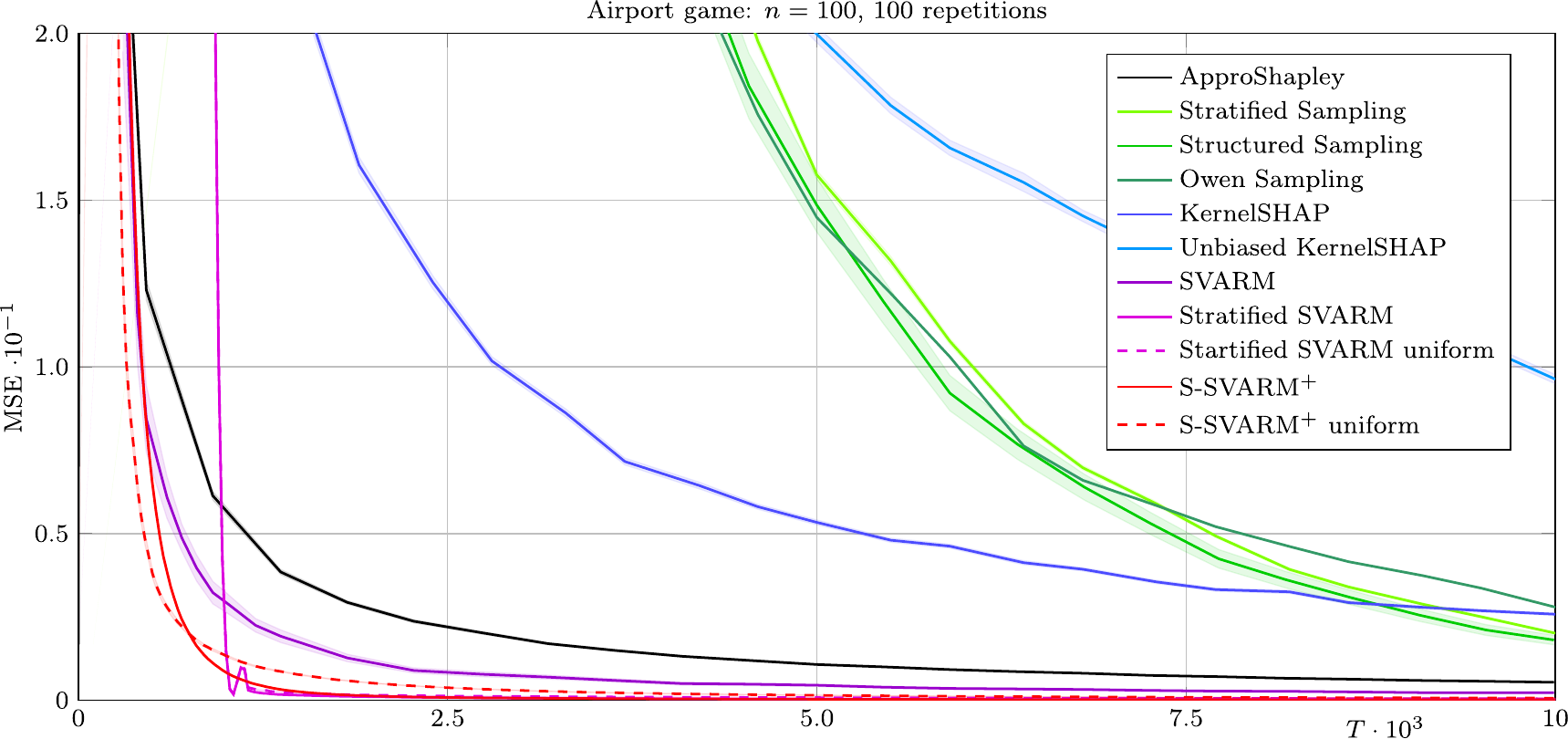}
	}
	\subfigure{
	   \includegraphics[width=0.485\textwidth]{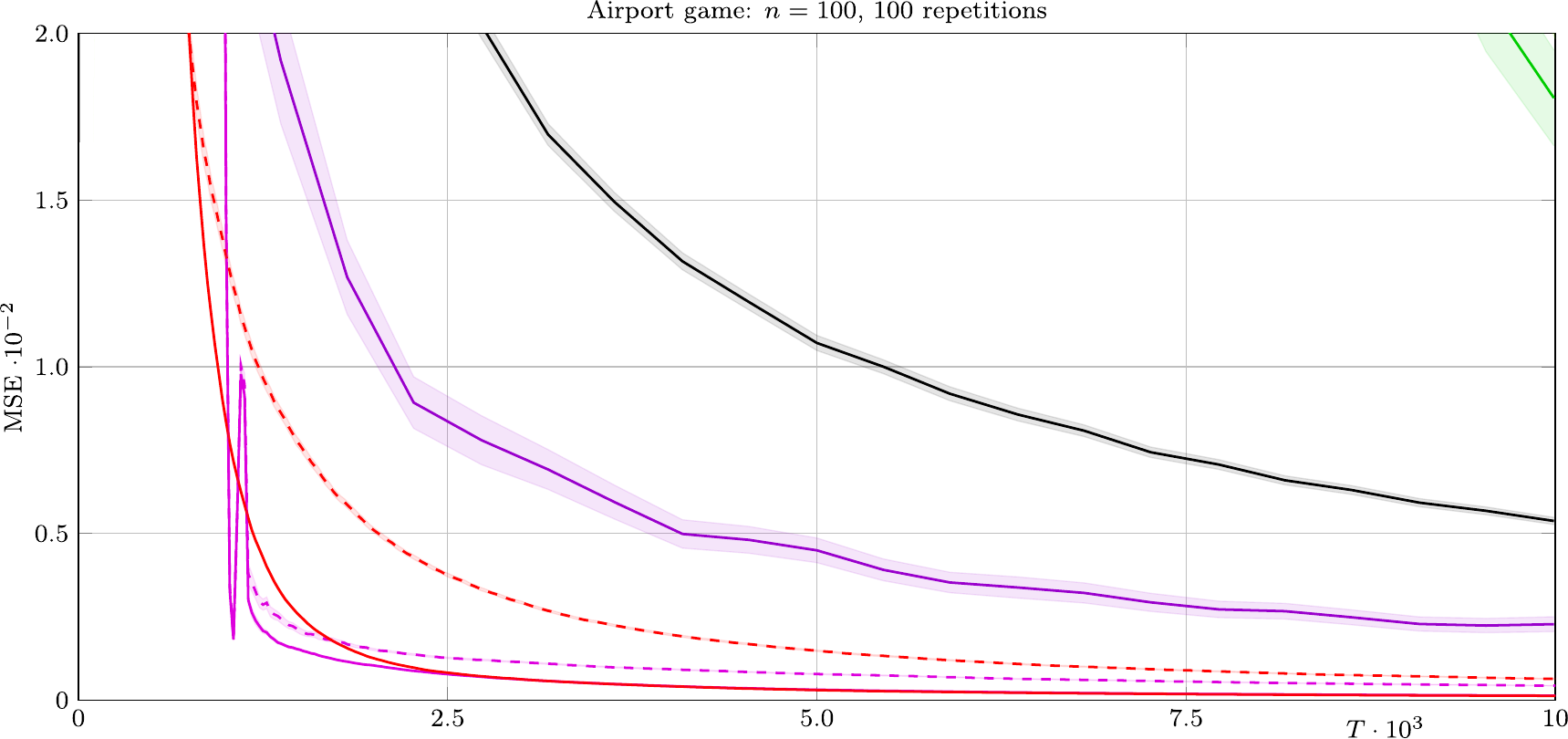}
	}
\caption{Airport game with 100 players: Averaged MSE over 100 repetitions in dependence of fixed budget T, shaded bands showing standard errors.}
\label{fig:Airport}
\end{figure*}

\begin{figure*}[ht]
	\centering
	\subfigure{
		\includegraphics[width=0.485\textwidth]{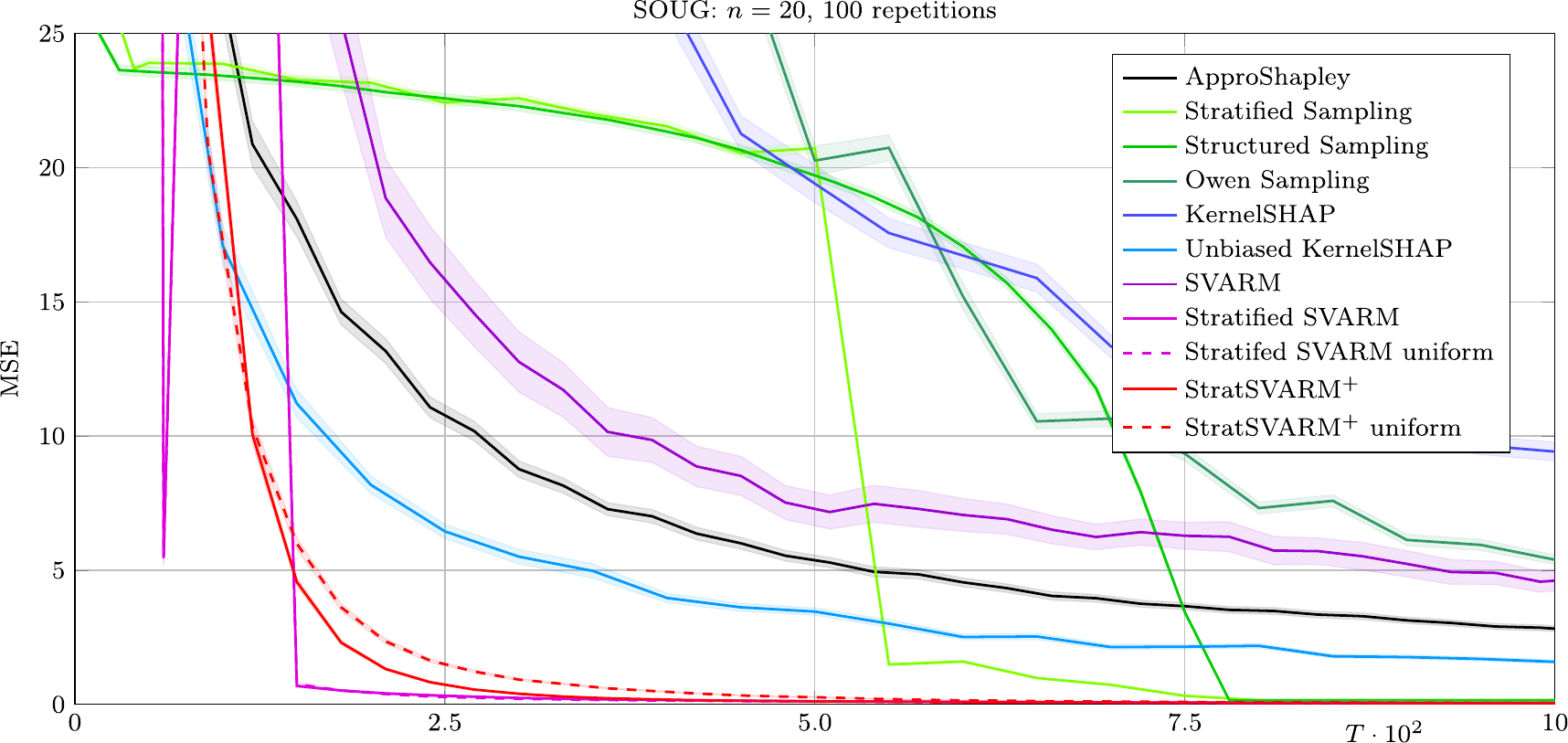}
	}
	\subfigure{
	   \includegraphics[width=0.485\textwidth]{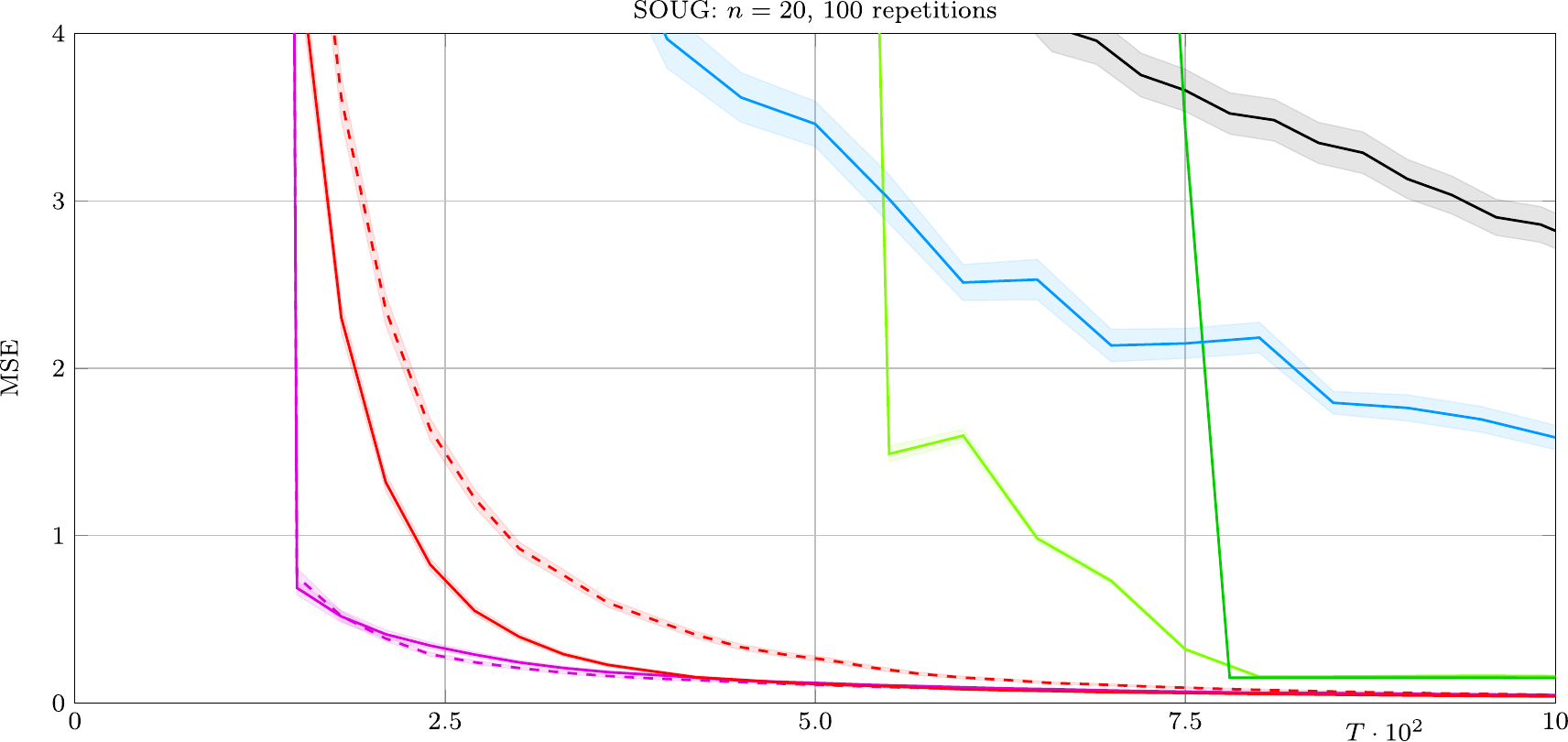}
	}
\caption{SOUG game with 20 players: Averaged MSE over 100 repetitions in dependence of fixed budget T, shaded bands showing standard errors.}
\label{fig:SOUG}
\end{figure*}

\begin{figure*}[ht]
	\centering
	\subfigure{
		\includegraphics[width=0.485\textwidth]{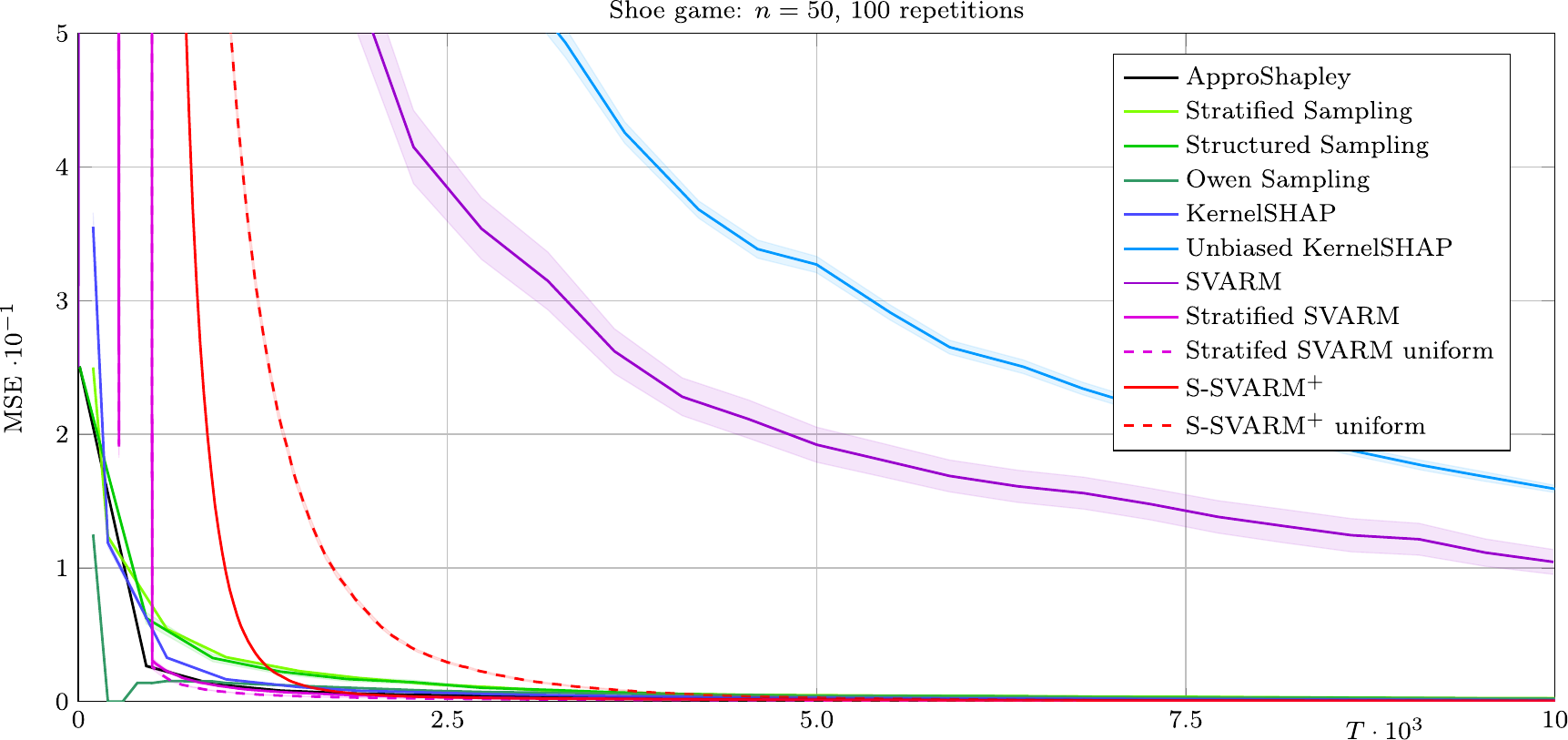}
	}
	\subfigure{
	   \includegraphics[width=0.485\textwidth]{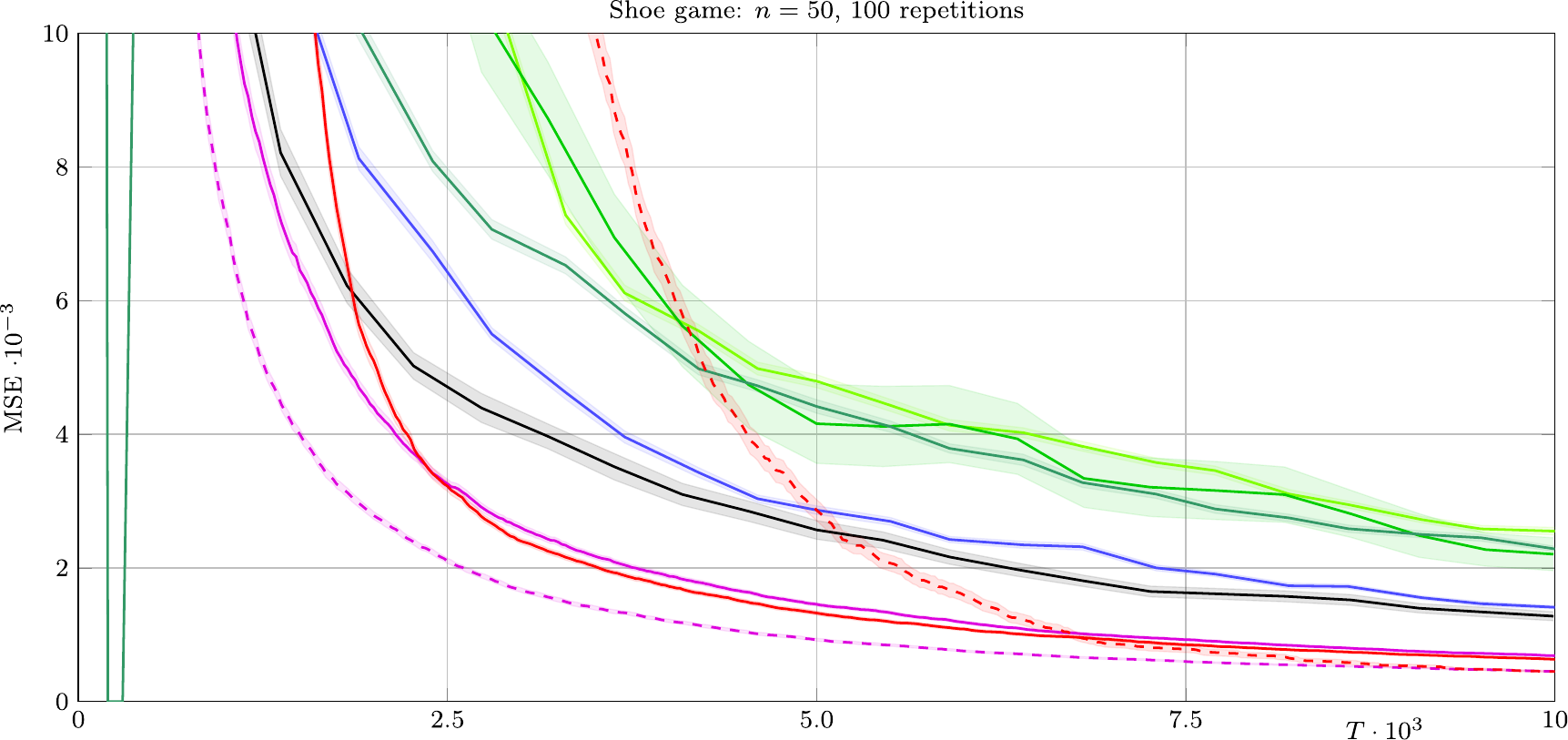}
	}
\caption{Shoe game with 50 players: Averaged MSE over 100 repetitions in dependence of fixed budget T, shaded bands showing standard errors.}
\label{fig:Shoe}
\end{figure*}

\begin{figure*}[ht]
	\centering
	\subfigure{
		\includegraphics[width=0.485\textwidth]{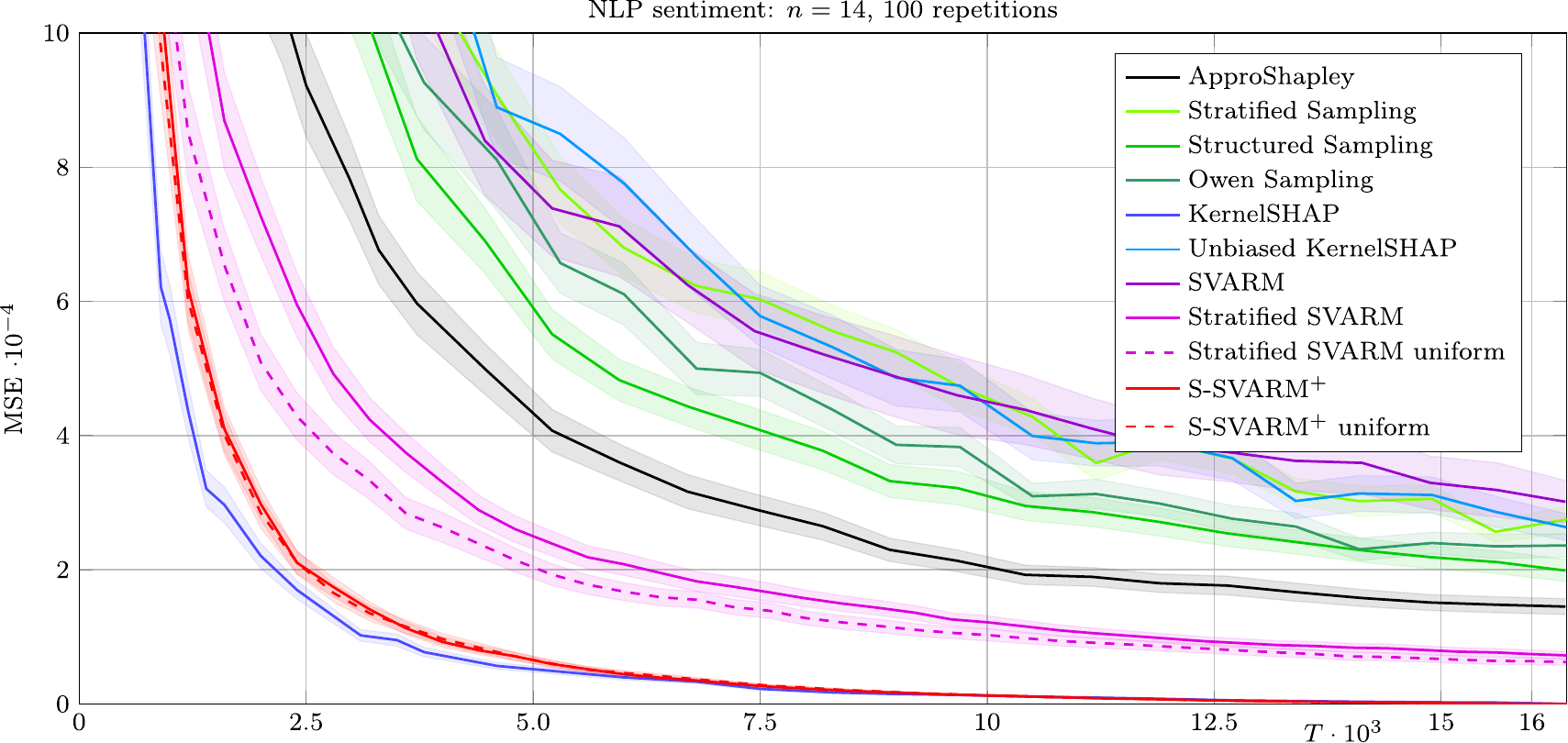}
	}
	\subfigure{
	   \includegraphics[width=0.485\textwidth]{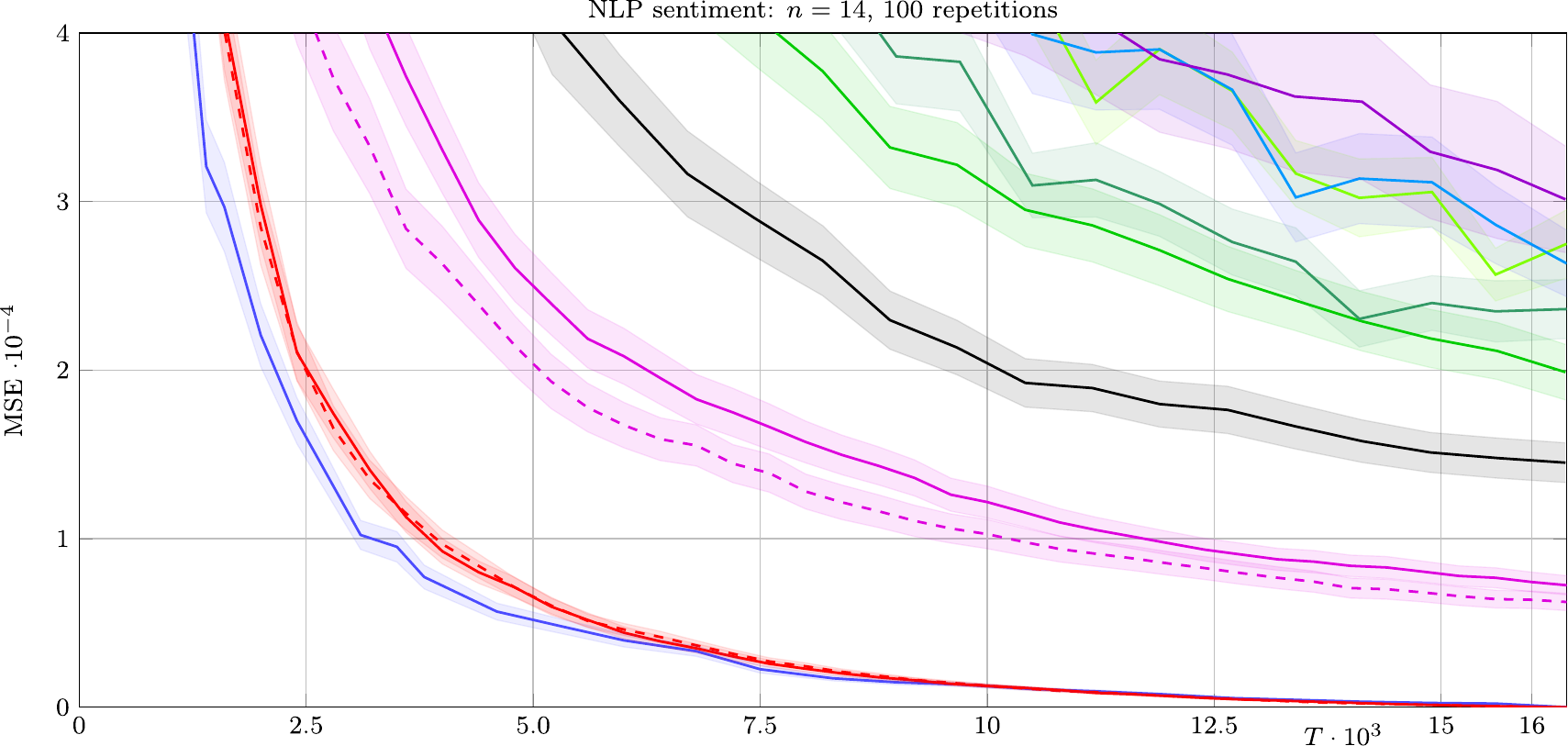}
	}
\caption{NLP game with 14 players: Averaged MSE over 100 repetitions in dependence of fixed budget T, shaded bands showing standard errors.}
\label{fig:NLP}
\end{figure*}

\begin{figure*}[ht]
	\centering
	\subfigure{
		\includegraphics[width=0.485\textwidth]{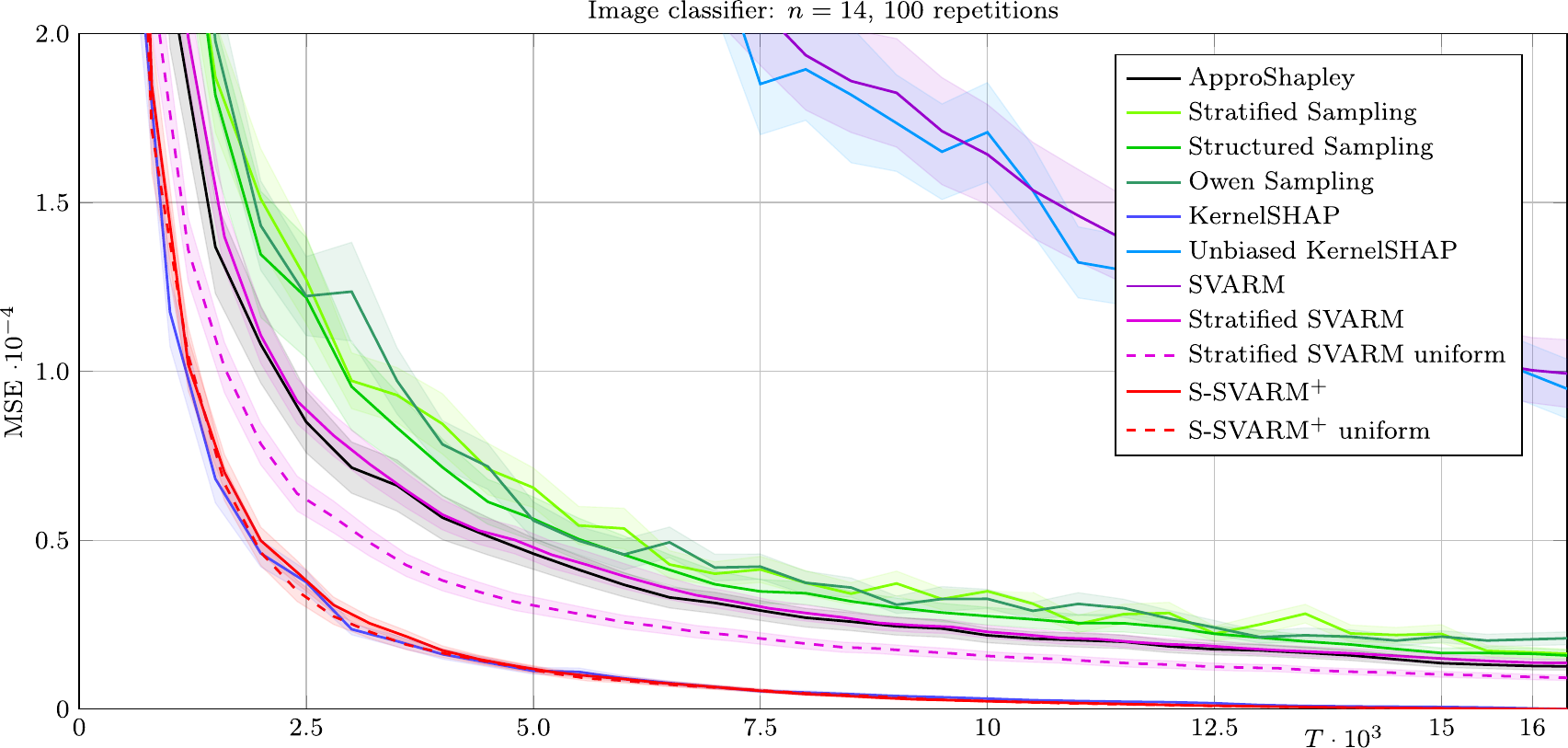}
	}
	\subfigure{
	   \includegraphics[width=0.485\textwidth]{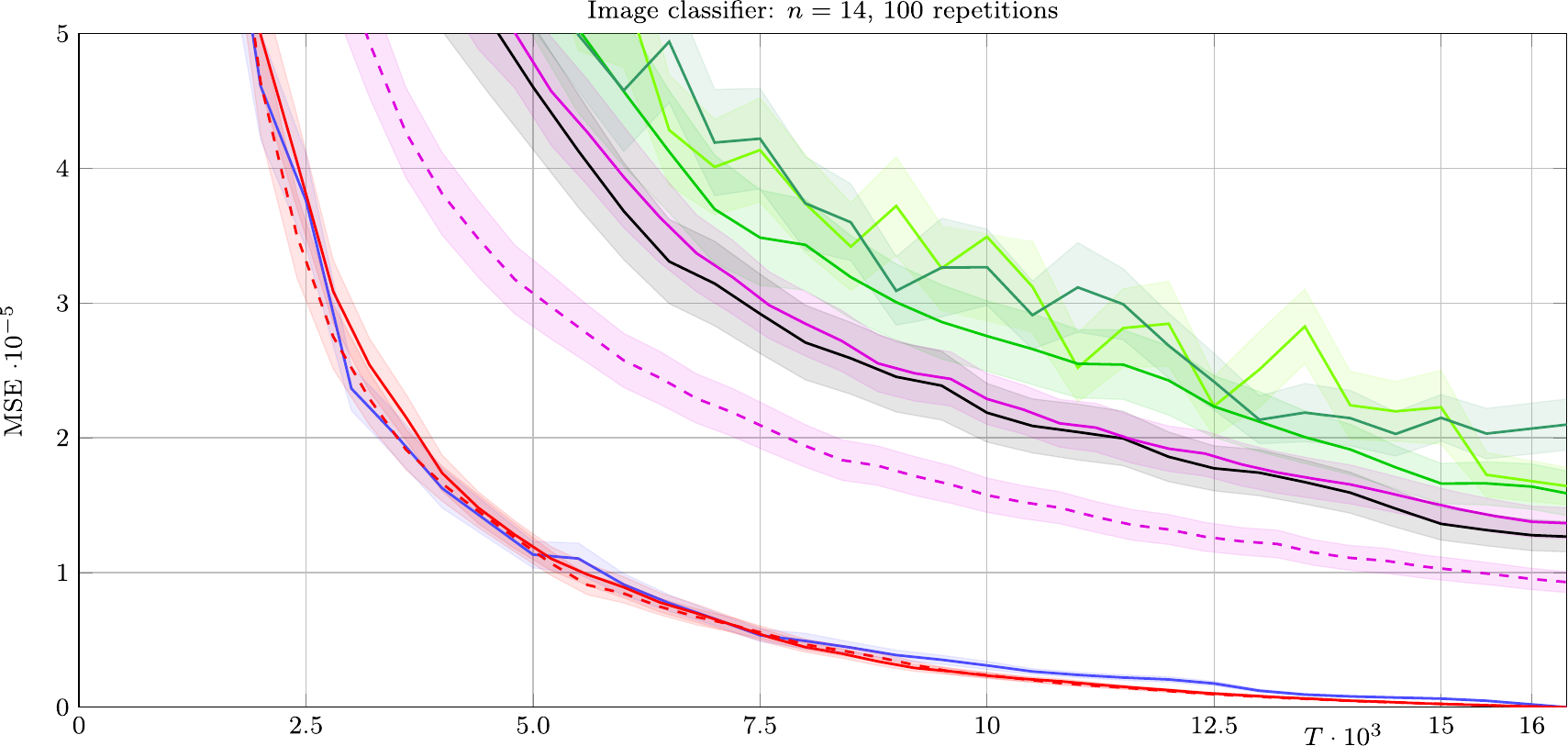}
	}
\caption{Image classifier game with 14 players: Averaged MSE over 100 repetitions in dependence of fixed budget T, shaded bands showing standard errors.}
\label{fig:Image}
\end{figure*}

\begin{figure*}[ht]
	\centering
	\subfigure{
		\includegraphics[width=0.485\textwidth]{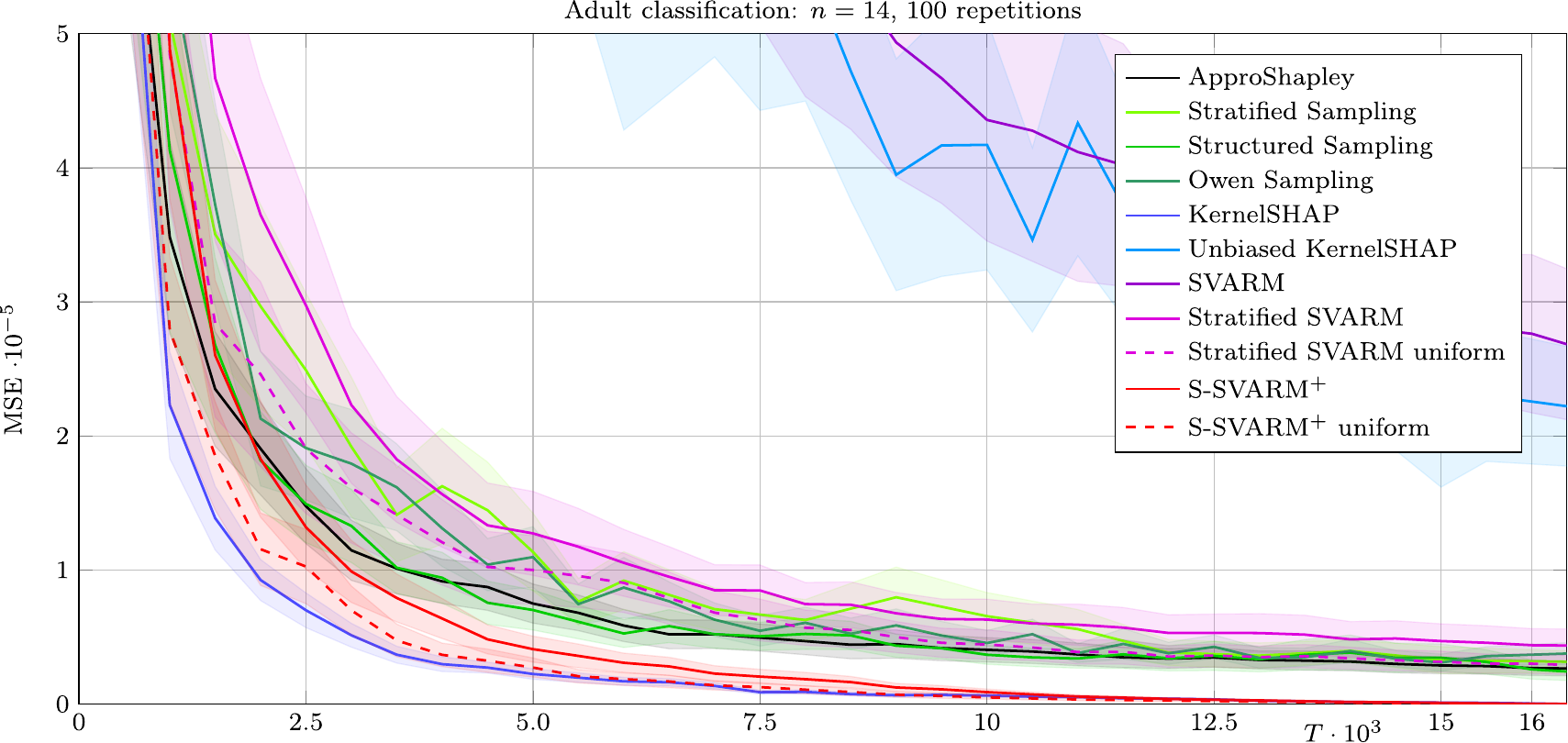}
	}
	\subfigure{
	   \includegraphics[width=0.485\textwidth]{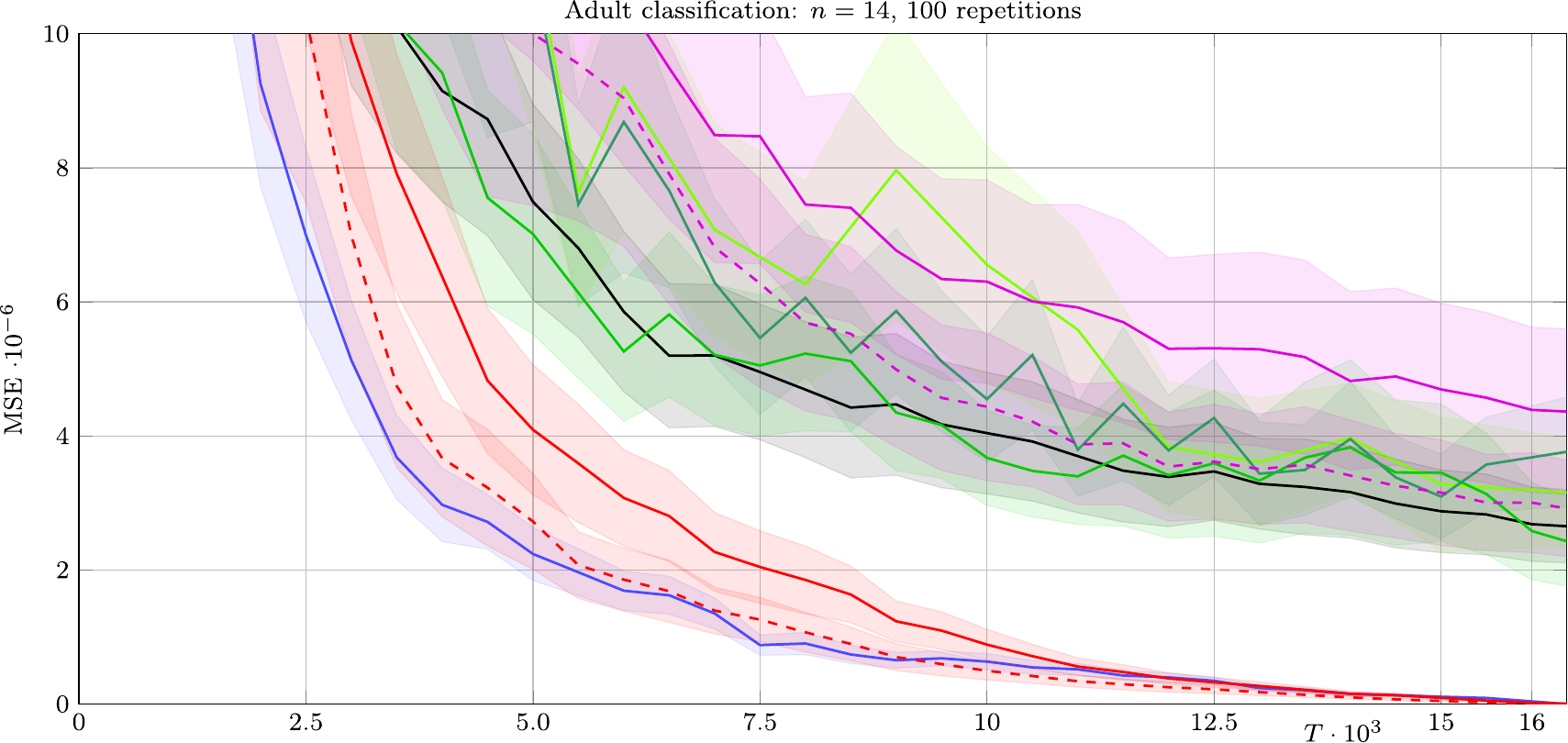}
	}
\caption{Adult classification with 14 players: Averaged MSE over 100 repetitions in dependence of fixed budget T, shaded bands showing standard errors.}
\label{fig:Adult}
\end{figure*}

\end{appendix}

\end{document}